\documentclass{article}
\usepackage[main,final]{neurips_2026}
\usepackage{microtype}
\usepackage{graphicx}
\usepackage{subcaption}
\usepackage{booktabs} 

\usepackage[utf8]{inputenc} 
\usepackage[T1]{fontenc}    
\usepackage{hyperref}       
\usepackage{url}            
\usepackage{booktabs}       
\usepackage{amsfonts}       
\usepackage{nicefrac}       
\usepackage{microtype}      
\usepackage{xcolor}         
\usepackage{tikz}

\usepackage{amsmath}
\usepackage{amssymb}
\usepackage{mathtools}
\usepackage{amsthm}
\usepackage{bm,bbm}
\usepackage{float}
\usepackage{balance}
\usepackage{algorithm2e}
\SetAlCapSkip{7pt}
\usepackage{multirow} 
\usepackage{colortbl}
\usepackage[capitalize,noabbrev]{cleveref}
\usepackage{makecell}
\usepackage{array}

\usepackage{booktabs}
\usepackage{multirow}
\usepackage{colortbl}
\usepackage{graphicx}
\usepackage{wrapfig}

\definecolor{coldlight}{RGB}{222,235,247}
\definecolor{coldmid}{RGB}{158,202,225}
\definecolor{colddark}{RGB}{49,130,189}

\definecolor{warmlight}{RGB}{254,224,210}
\definecolor{warmmid}{RGB}{252,146,114}
\definecolor{warmdark}{RGB}{222,45,38}

\title{Fine-Tuned Language Models Exhibit Improved Information Conveyance}

\author{
  Yuwei Cheng\thanks{Equal contribution.}
  \qquad
  Weiyi Tian\footnotemark[1]
  \qquad
  Haifeng Xu
  \\[0.2em]
  The University of Chicago
  \\[0.1em]
  \texttt{\{yuweicheng,weiyitian,haifengxu\}@uchicago.edu}
}
\begin{document}
\maketitle

\begin{abstract}
Fine-tuning is often believed to reduce uncertainty and diversity in large language models, but existing analyses overlook output length, a key confounder, and therefore fail to capture how uncertainty is distributed across an entire generation rollout. To address this, we propose \emph{Canopy Entropy} ($\mathrm{CE}^\star$), a measure that views language generation from a tree perspective, where ``canopy'' represents the space of all possible rollouts, making $\mathrm{CE}^\star$ naturally quantify the \emph{effective size of the generation space}. $\mathrm{CE}^\star$ jointly captures uncertainty in both the output length $N$ and the generated sequence $Y_{1:N}$ -- indeed, we show that it equals total Shannon entropy $H(N, Y_{1:N}\mid X)$, where $X$ denotes the prompt. This formulation yields interpretable metrics, including a length--entropy correlation term $\rho(N, r_N)$, where $r_N$ is the entropy rate, quantifying \emph{information conveyance efficiency} by indicating whether longer outputs are more or less informative per token. Empirically, across tasks and model families, we find that fine-tuned models predominantly exhibit more positive correlation $\rho(N, r_N)$, even when total entropy decreases. Furthermore, after controlling for model family, task, prompt, and output-length effects, we find that fine-tuning nearly \emph{triples} the correlation strength between entropy rate and semantic diversity, suggesting that aligned models convert {token uncertainty into semantic diversity more efficiently.} 
Overall, these results demonstrate that fine-tuning does not simply reduce uncertainty, but fundamentally reorganizes it into more informative and semantically meaningful generations. Our code is available at \url{https://github.com/WeiyiTian/canopy-entropy}.
\end{abstract}

\def \bp{\mathbb{p}}
\def \y{\mathrm{y}}
\def \p{\mathrm{p}}

\ifx\theorem\undefined
\newtheorem{theorem}{Theorem}
\newtheorem{definition}{Definition}
\newtheorem{assumption}{Assumption}
\newtheorem{corollary}{Corollary}
\newtheorem{remark}{Remark}[section]
\fi
\ifx\example\undefined
\newtheorem{example}[theorem]{Example}
\fi
\ifx\property\undefined
\newtheorem{property}{Property}
\fi
\ifx\lemma\undefined
\newtheorem{lemma}{Lemma}
\fi
\ifx\proposition\undefined
\newtheorem{proposition}{Proposition}
\fi
\providecommand*{\definitionautorefname}{Definition}
\providecommand*{\propositionautorefname}{Proposition}
\providecommand*{\propositionautorefname}{Corollary}
\providecommand*{\propositionautorefname}{Lemma}
\newcommand{\VE}[1]{\textcolor{magenta}{#1}}
\section{Introduction}
Large language models (LLMs) have achieved strong performance across tasks such as natural language understanding \citep{brown2020language, hendrycks2020measuring}, code generation \citep{chen2021evaluating, roziere2023code}, and mathematical reasoning \citep{lightman2024let, lu2025fine}. A key driver of this success is fine-tuning, including instruction tuning \citep{wei2021finetuned, chung2024scaling} and alignment with human preferences \citep{ouyang2022training, rafailov2023direct}. While fine-tuning substantially improves helpfulness and task performance, a growing body of work suggests that it may also reduce output diversity and compress the effective generation space of language models \citep{wang2025optimizing, yang2025llm, lake2025distributional, west2025base}. Prior studies report that base models exhibit greater randomness and creativity, whereas aligned models produce more concentrated and similar responses across samples \citep{west2025base}. Similar conclusions have been drawn using branching factor \citep{yang2025llm}, lexical diversity \citep{lake2025distributional}, and token-level uncertainty measures \citep{agarwal2025unreasonable, wang2026entropy}, reinforcing the view that fine-tuning makes language models more deterministic and less diverse.

However, existing analyses suffer from an important limitation -- they do not control for output length, despite length being a major confounder in diversity measurement \citep{shaib2024standardizing}. Classical metrics such as type-token ratio are biased by text length, often making shorter responses appear more diverse \citep{shaib2024standardizing}. Such metrics that ignore generation length cannot distinguish genuinely diverse outputs from statistical artifacts induced by shorter outputs. This limitation is especially problematic in aligned language models because fine-tuning changes not only what models generate, but also how long they generate. Base models often produce shorter responses, while instruction-tuned models generate longer and more structured outputs \citep{singhal2023long}. As illustrated in \autoref{fig:entropy_rate}, 
base models exhibit declining entropy rates (average uncertainty per generated token, see \autoref{def:BF}), suggesting increasingly predictable and redundant outputs. In contrast, fine-tuned models maintain more stable entropy rates, preserving information diversity over longer generations.

More importantly, prior work largely treats diversity as a static property of outputs, without studying how uncertainty evolves along generation trajectories. Previous work on entropy-rate constancy examines how information density changes across positions within human-written text \citep{genzel-charniak-2002-entropy,verma-etal-2023-revisiting}. However, to the best of our knowledge, no existing work analyzes the relationship between LLM generation length and entropy rate---that is, whether longer generations become more informative or increasingly redundant. Yet this interaction is crucial, since fine-tuning learns not only \emph{what} to say, but also \emph{when} to elaborate and \emph{when} to stop. As a result, output length itself becomes part of the learned generation strategy, making diversity comparisons incomplete without accounting for how uncertainty is allocated across trajectories.

Meanwhile, recent studies suggest a more nuanced picture than simple diversity reduction. Fine-tuning can increase certain forms of diversity, such as semantic diversity in code generation, even while reducing lexical variation \citep{shypula2025does}. These findings suggest that fine-tuning may reorganize uncertainty and generation structure, rather than uniformly suppress diversity, leaving the relationship between uncertainty allocation, semantic diversity, and generation trajectories still poorly understood.

\textbf{Our contribution.}
In this work, we revisit diversity in language models through a principled information-theoretic lens. We propose \emph{Canopy Entropy} ($\mathrm{CE}^\star$), an information-theoretic measure that views language generation from a tree perspective. Just as a biological canopy represents the total spread and reach of a tree's branches, our metric views the ``canopy'' as the set of all potential leaves—or finished rollouts—available to the model. This allows $\mathrm{CE}^\star$ to naturally quantify the effective volume of the generation space by accounting for both how \emph{widely} the model branches and how \emph{deeply} it explores. $\mathrm{CE}^\star$ is introduced to jointly capture  output length and content uncertainty, and we show that it admits an exact characterization as the total Shannon entropy $H(N, Y_{1:N}\mid X)$, where $X$ denotes the prompt, $N$ is the generated length, and $Y_{1:N}$ is the generated sequence. Building on this formulation, we derive interpretable metrics and introduce a length--entropy correlation term that measures how uncertainty is allocated across generation trajectories (see \autoref{sec:method}). {Since $\mathrm{CE}^\star$ exhausts the entire space of rollouts, a key challenge is to estimate it efficiently. Towards that end,}
we develop Monte Carlo estimators and establish their consistency (see \autoref{sec:estimator}).

Empirically, we evaluate multiple model families across sentence completion, mathematical reasoning, coding, and story generation, while explicitly controlling for output length. Our results reveal three consistent findings: \textbf{(i)} fine-tuning generally reduces token-level and trajectory-level uncertainty, although the magnitude of reduction is highly task-dependent; \textbf{(ii)} fine-tuning predominantly shifts {$\rho(N, r_N)$} from negative toward positive, indicating that longer generations become more informative rather than {being redundant as often seen in base models;}
and \textbf{(iii)} entropy rate becomes a substantially stronger predictor of semantic diversity after fine-tuning, suggesting that aligned models utilize uncertainty more efficiently and translate it more effectively into meaningful semantic diversity (see \autoref{sec:results}).  Overall, our results suggest that fine-tuning does not simply reduce uncertainty or shrink the generation space. Rather, it restructures how uncertainty is allocated across generation trajectories, producing outputs that are more organized, semantically meaningful, and information-efficient.


\section{Related work}\label{sec:related_work}
\textbf{Length as a fundamental confounder in diversity measurement.} Recent work \citep{singhal2023long} shows that RLHF and instruction tuning systematically increase LLM output length, since longer responses are often positively correlated with human preference scores. This makes output length a critical confounder in diversity evaluation. Metrics that favor shorter generations may therefore unfairly penalize instruct models simply because they produce longer responses. For example, n-gram diversity metrics \citep{li2016diversity} naturally decrease with sequence length, as longer texts inevitably reuse a finite vocabulary, reducing the ratio of unique n-grams even when the underlying content remains semantically diverse. Embedding-based semantic diversity measures \citep{tevet2021evaluating} suffer from a related issue: longer generation trajectories must be compressed into fixed-dimensional representations, often through mean pooling, which can dilute semantic diversity. Moreover, many existing diversity metrics rely on heuristic sampling procedures and lack a unified probabilistic foundation \citep{west2025base}. Consequently, current approaches cannot rigorously distinguish whether the increased verbosity of instruct models reflects genuine information gain or merely redundant continuation, motivating the need for a principled length-aware and model-intrinsic measure of generation diversity.

\textbf{From local branching to global canopy.} Conceptualizing autoregressive generation as a stochastic, path-dependent branching tree, Branching Factor (BF) \citep{yang2025llm} was introduced to quantify the probability concentration induced by model alignment. However, their framework lacks an intuitive heuristic for interpreting the measure and focuses primarily on the \emph{width} of the generation tree (branching) while ignoring its \emph{depth} (length) and the coupling between the two. $\mathrm{CE}^\star$ addresses this gap by   providing a model-intrinsic characterization of the generation space by jointly modeling content uncertainty and stochastic stopping behavior. By deriving the BF as a normalized version of $\mathrm{CE}^\star$, we provide a principled theoretical foundation for its use. Crucially, our length--entropy correlation demonstrates that while fine-tuning reduces the absolute BF, it restructures the generation tree to sustain information density across longer paths while translating uncertainty into semantic diversity more efficiently. To our knowledge, this is the first work to provide a unified measure of the global LLM generation space that accounts for both the local branching uncertainty and the stochasticity of trajectory length, while also studying the redistribution of uncertainty across generation trajectories.
\section{Methodology}\label{sec:method}
{\textbf{Notations. } The Shannon entropy \citep{shannon1948mathematical} of any discrete random variable (r.v.) $Z$, with probability $q(z)$ for $z \in \mathcal Z$,   is $H(Z) = -\sum_{z \in \mathcal Z } q(z) \log q(z)$. Let  $X$ be another r.v. and $X, Z$ have joint probability $q(z,x)$, then $H(Z|X=x) = -\sum_{z \in \mathcal Z } q(z|x) \log q(z|x)$ is the conditional entropy of $Z$ given realized $x$, whereas $H(Z|X) = \sum_{x} H(Z|x) q(x)$ is the \emph{expected} conditional entropy of $Z$. }
\subsection{Measuring the size of LLM generation space from a tree perspective}
\begin{figure}[t]
    \centering
    \includegraphics[width=\linewidth]{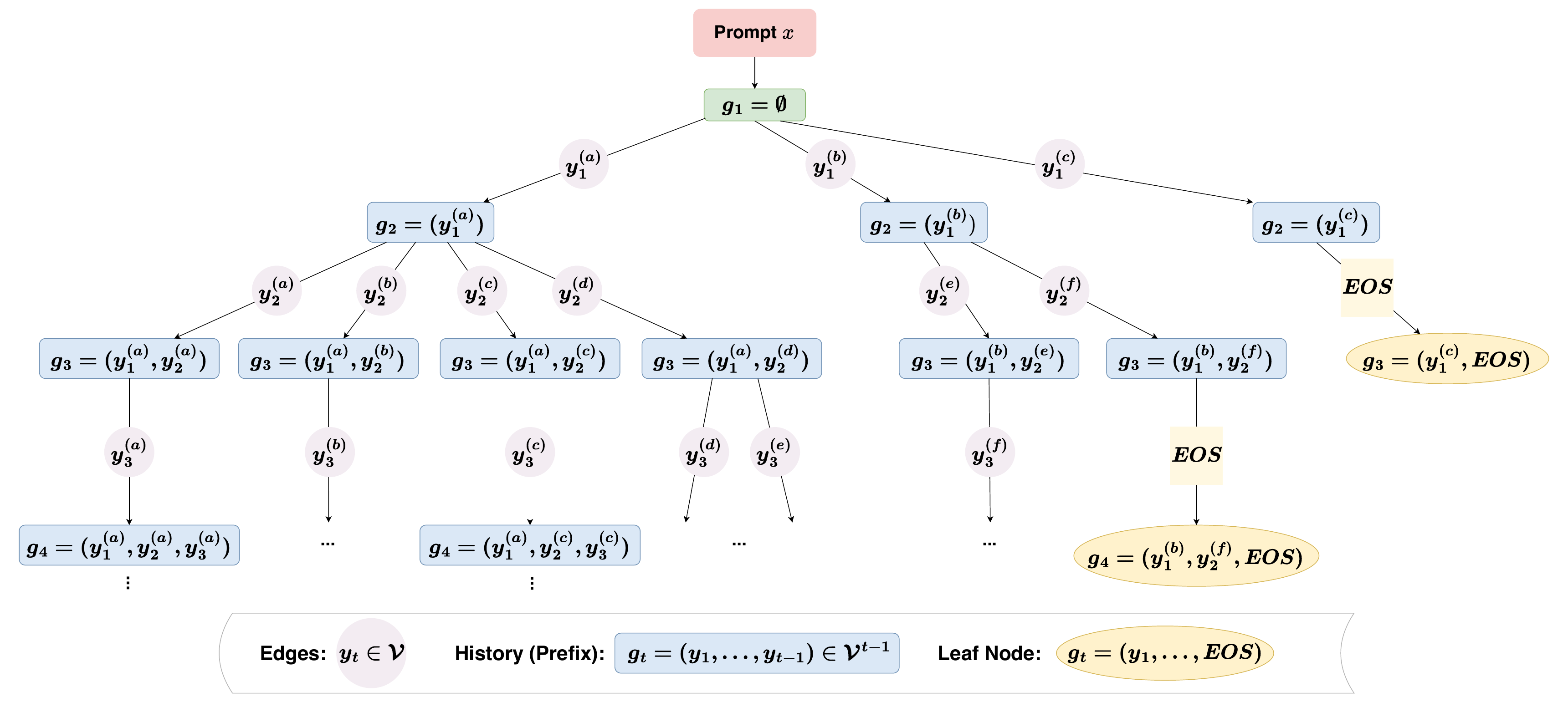}
    \caption{Path-dependent generation tree induced by autoregressive LLMs.}
    \label{fig:generation_tree}
\end{figure}

\textbf{Autoregressive generation as a stochastic tree.}
Autoregressive language models induce a stochastic generation process that can naturally be represented as a rooted, path-dependent tree. Let $X \sim p_X$ denote a prompt sampled from a prompt distribution over a space $\mathcal X$, and let $\mathcal V$ denote the vocabulary. Conditioned on $X=x$ and a decoding policy (e.g., temperature, top-$k$, or nucleus sampling), generation proceeds sequentially by sampling tokens from a conditional distribution. Formally, let
\(
Y_1,Y_2,\ldots,
\)
denote the autoregressive token process generated by the model. For each step $t \ge 1$, define the generation history as
\(
G_t := (Y_1,\ldots,Y_{t-1}),
\)
with $G_1 := \emptyset$. The next-step random variable is denoted by
\(
Z_t \in \mathcal Z := \mathcal V \cup \{\texttt{EOS}\},
\)
where $\texttt{EOS}$ denotes termination. The conditional distribution
\(
q(Z_t \mid G_t, X)
\)
governs the next-step generation process.\footnote{This formulation is equivalent to a two-stage stochastic decision process. Specifically, the model first makes a binary decision between \textsf{STOP} (terminate) and \textsf{CONT} (continue). Conditional on \textsf{CONT}, it then samples the next token from the vocabulary $\mathcal V$. We provide a formal proof of this equivalence in \autoref{app:two_stage_equivalence}.} Under this construction, each node of the generation tree (see \autoref{fig:generation_tree}) corresponds to a history $G_t$. Each outgoing edge corresponds to a possible realization of $Z_t$. A complete generated sequence therefore corresponds to a root-to-leaf trajectory. A key feature of this construction is its path dependence. The next-step distribution $q$ depends on the entire generation history $G_t$, not merely the most recent token. Thus, two trajectories ending with the same token may evolve very differently if they arise from different histories.

\textbf{Width and depth of the generation tree.} The size of the generation tree is determined by two complementary factors. First, at each history $G_t$, the spread of $q(\cdot \mid G_t, X)$ determines the number of plausible continuations, capturing the local branching behavior (\emph{width}) of the tree. Second, generation proceeds for a random number of steps before termination, inducing variability in trajectory depth. We define the stopping time
\(
N := \inf\{t \ge 1 : Z_t = \texttt{EOS}\},
\)
which represents the generation length (\emph{depth}). Under this formulation, the stochastic process jointly models both continuation uncertainty and stopping behavior, together determining the overall size of the generation space.

\textbf{A unified measure of generation space.} To quantify local uncertainty, we use Shannon entropy \citep{shannon1948mathematical}. For a given prompt $X = x$ and a realized generation history $G_t = g_t$, define
\begin{equation}\label{eq:shorthand}
    \widetilde H(g_t \mid x)
:=
H(Z_t \mid G_t=g_t, X=x)
= 
-\sum_{z \in \mathcal Z}
q(z \mid g_t, x)\log q(z \mid g_t, x).
\end{equation}
This quantity measures uncertainty via entropy in the immediate next generation step, including both continuation-token selection and termination. {Note that this quantity summarizes only the uncertainty of the next-step generation conditioned on a realized history $g_t$ and prompt $x$}. {In the following, we slightly overload the notation by also using uppercase $G_t$ and $X$ to denote realized histories and prompts whenever the distinction between random variables and their realizations is clear from context.} To jointly account for local branching diversity and random trajectory length, we introduce the following uncertainty measure.
{
\begin{definition}[Canopy Entropy]\label{def:ce_star}
We define \emph{Canopy Entropy} as 
\(
\mathrm{CE}^{\star}
:=
\mathbb{E}
\left[
\sum_{t=1}^{N(x)}
\widetilde H(g_t \mid x)
\right]
\).
More formally,
\(
\mathrm{CE}^{\star}
=
\mathbb{E}_{x, \{g_t\}_{t=1}^{N(x)} }
\left[
\sum_{t=1}^{N(x)}
H(Z_t \mid G_t = g_t, X = x)
\right]
\),
where the $\mathbb{E}$ is over the randomness of prompt $x$, $N(x)$ (the length), and  $\{g_t\}_{t=1}^{N(x)}$ (all possible partial histories of rollouts). 
\end{definition}
}
That is, the quantity $\mathrm{CE}^{\star}$ aggregates entropy-based uncertainty along generation trajectories and across prompts. A single generated sequence corresponds to a specific traversal from the tree root (prompt $X$) to a leaf (\texttt{EOS} termination). The summation $\sum_{t=1}^{N(X)}$ thus represents the total uncertainty accumulated along that entire path. By taking the expectation over all trajectories, $\mathrm{CE}^\star$ captures the aggregate ``canopy'' of the model's output distribution across all possible rollouts. As such, it provides a unified measure of generation space by accounting for both \emph{local branching width} (the diversity of choices at each node) and \emph{trajectory depth} (the variability in generation length).

\textbf{\boldsymbol{$\mathrm{CE}^\star$} as the joint entropy of length and content.}
Besides being an intuitive measure, the above construction turns out to also admit a clean information-theoretic characterization. In particular, the cumulative uncertainty along a trajectory can be interpreted as the total entropy of the generation process. For reference, we call $H(N, Y_{1:N}|X)$ the \emph{total uncertainty}, $H(N|X)$ the \emph{length uncertainty}, and $H(Y_{1:N} \mid N, X)$ the \emph{content uncertainty}. The following theorem formalizes the connection.

\begin{theorem}[Equivalence of $\mathrm{CE}^\star$ and  total uncertainty]\label{thm:entropy}
$\mathrm{CE}^\star = H(N, Y_{1:N}\mid X).$
\end{theorem}
Theorem~\ref{thm:entropy} specializes classical information-theoretic results on randomly stopped stochastic processes \cite{DBLP:journals/jlp/BiondiLNW14, DBLP:journals/tit/EkrootC91} to autoregressive language generation, showing that the Canopy Entropy coincides with the total uncertainty of the generation process (see \autoref{app:proof_thm1} for a detailed proof). By the entropy chain rule, $H(N, Y_{1:N} \mid X)$ $=$ $H(N \mid X)$ $ + H(Y_{1:N} \mid N, X),$ decomposing total uncertainty into two components. The \emph{length uncertainty} $H(N\mid X)$ captures how unpredictable the stopping time is, while the \emph{content uncertainty} $H(Y_{1:N} \mid N, X)$ captures how diverse the generated sequences are once the length is fixed. This decomposition aligns naturally with the tree perspective: length uncertainty corresponds to variability in the depth of trajectories, while content uncertainty corresponds to the diversity of paths of a given length. We can interpret $\exp(\mathrm{CE}^\star)$ as the effective number of plausible full sequences, analogous to how perplexity represents the effective branching at the token level. 
 
\subsection{{Two normalized variants of Canopy Entropy}}
 
While $\mathrm{CE}^{\star} = H(N, Y_{1:N}\mid X)$ provides a clean characterization of the generation space, it aggregates uncertainty across both sequence lengths and content diversity, making its direct interpretation difficult. To obtain more interpretable metrics, we introduce two normalized variants of $\mathrm{CE}^\star$, respectively  coined \emph{Generation Perplexity} and \emph{Branching Factor} (as introduced in \citep{yang2025llm}). {Intuitively, the former corresponds to the token-level normalization of $\mathrm{CE}^\star$, whereas the latter is trajectory-level normalization. In both definitions, the $\exp$ function converts entropy uncertainty to effective diversity.} 

\begin{definition}[GenPPL]\label{def:GenPPL}
Generation Perplexity is defined as 
\(
\mathrm{GenPPL}
:=
\exp\left(
\frac{\mathrm{CE}^\star}{\mathbb E[N]}
\right)\), which is equivalent to \(
\exp\left(
\frac{
\sum_{X \sim p_X}\,
H(N,Y_{1:N}\mid X=x)
}{
\sum_{X \sim p_X}\,
\mathbb E[N\mid X=x]
}
\right).
\)
\end{definition}

\begin{definition}[BF, \citep{yang2025llm}]\label{def:BF}
Branching Factor is defined as
\(
\mathrm{BF}
:=
\exp\left(
\mathbb E_X
\mathbb E_N
[r_N\mid X=x]
\right),
\)
where
\(
r_N
:=
\frac{
H(Y_{1:N}\mid N=n,X=x)
}{n}
\)
is the entropy rate (i.e., per-token content uncertainty).\footnote{In information theory, the entropy rate quantifies the average amount of information produced per generated token by a stochastic process, and can be interpreted as a measure of information density or non-redundancy.} 
\end{definition}

{Invoking Theorem \ref{thm:entropy}, we observe that the only difference between Canopy Entropy $\mathrm{CE}^\star$ and GenPPL or BF is that -- excluding the ``$\exp$'' function --  $\mathrm{CE}^\star$ does not have the normalization term $1/(\sum_{X \sim p_X}\,
\mathbb E[N\mid X=x])$ or $1/n$. It is worthwhile to think more deeply about the difference between GenPPL and BF, due to their different normalization factors, which will be prominent when the generated sequence length $N$ can differ a lot. In such a case, GenPPL \emph{treats each token equally}, regardless of whether it comes from a short or long sequence, through a ``global'' normalization factor $1/(\sum_{X \sim p_X}\,
\mathbb E[N\mid X=x])$. However, BF as the expectation of $r_N$ \emph{treats each trajectory equally} and consequently will discount the influence of tokens in long sequences much more than those in short sequences. Therefore, at a high level, GenPPL reflects the diversity of a \emph{typical token} in the generation whereas BF reflects the diversity of a \emph{typical trajectory}. For further intuition, we refer readers to \autoref{app:example} for a concrete example illustrating the distinction between GenPPL and BF.}
 
GenPPL and BF together provide a more complete view of generation behavior across tasks. In particular, we observe that structured tasks such as mathematical reasoning or coding are better characterized by token-level uncertainty (GenPPL), whereas open-ended tasks like story generation are better captured by trajectory-level diversity (BF). Our framework provides a principled formulation of GenPPL as a generative analogue of perplexity \citep{jelinek1977perplexity} for variable-length processes, and offers a unified information-theoretic interpretation of BF, initially proposed by Yang et al. \citep{yang2025llm}, showing that both arise naturally as normalized variants of Canopy Entropy. See Appendix~\ref{app:bf_genppl_usage} for guidance on when to use BF or GenPPL, and Appendix~\ref{app:comparison_perplexity} for a comparison with perplexity.

\subsection{The length--entropy correlation: how entropy evolves with generation length}
Aggregate measures such as Canopy Entropy, GenPPL, and BF quantify the overall magnitude of uncertainty. However, they fail to capture how uncertainty evolves across generation rollouts or how effectively uncertainty is translated into meaningful variation as generation length increases. Whether longer generations remain information-rich or become progressively redundant \citep{singhal2023long, shaib2024standardizing} remains poorly understood despite rich previous studies on the diversity of language generation \citep{tevet2021evaluating, li2016diversity, friedman2022vendi}. Understanding this interaction is crucial for characterizing the \emph{information-conveying efficiency} of language models---namely, how much uncertainty is maintained per token as generation progresses.

To study this phenomenon, we introduce the prompt-controlled  correlation $\rho(N, r_N)$ between the length $N$ and entropy rate $r_N$, defined below. 
{
\begin{definition}[The length--entropy correlation]
\label{def:correlation} For each prompt, the length--entropy correlation is defined as 
\begin{equation*} 
\rho(N,r_N) = \mathbb{E}_{X \sim p_X} \bigg\{\texttt{Corr}(N, r_N \mid X = x)\bigg\}, \text{ where } r_N
 =
\frac{
H(Y_{1:N}\mid N=n,X=x)
}{n}   
\end{equation*}
is the entropy rate and \texttt{Corr} is any standard correlation notion from the set \{Pearson correlation, Spearman correlation, Kendall correlation\}.  
\end{definition} 
A positive correlation $\rho(N,r_N)$ indicates that longer sequences remain informative per token, while a negative correlation indicates that longer outputs become increasingly predictable. This quantity characterizes uncertainty allocation along trajectories and is later used in \autoref{sec:results} to compare base and fine-tuned models across tasks and model families.  In our evaluation, we observe broadly consistent trends across the three standard correlation notions as mentioned in the definition above, which illustrates the robustness of the discovered correlations among generation length and entropy rate. Notably, the correlation $\rho(N,r_N)$ is evaluated over the randomness of  rollouts drawn from an exponentially large space; we discuss efficient estimation methods in the section below, with full details deferred to   \autoref{app:other_estimators}.  
}
\section{Efficient and consistent estimation of $\mathrm{CE}^\star$ }\label{sec:estimator}
 
The previous section introduced $\mathrm{CE}^\star = H(N, Y_{1:N} \mid X)$ and its derived metrics such as GenPPL, BF and length--entropy correlation.  However, efficiently computing these metrics in practice is challenging since they are expectations over exponentially many trajectories \citep{hu2023amortizing, yang2025llm}. This section develops computationally tractable approaches to estimate $\mathrm{CE}^\star$ {which are also generalizable to estimating the derived metrics}. In particular, we introduce a Monte Carlo estimator under a maximum length constraint $T_{\max}$, and analyze the consistency of this estimator. 

 
For a random prompt \(X\), let \(N(X)\) be the generation's (random) stopping time. Given a maximum length \(T_{\max}\), define the truncated stopping time
$
N_{\max}(X) := \min\{N(X), T_{\max}\}.
$
The corresponding truncated functional is
$
\mathrm{CE}^{\star}_{\max}
:=
\mathbb{E}
\left[
\sum_{t=1}^{N_{\max}(X)}
\widetilde H(G_t \mid X)
\right]$. Now suppose \(x_1, \dots, x_P\) are sampled i.i.d.\ from \(p_X\), the prompt distribution. For each prompt \(x_p\), we generate \(M\) independent rollouts. For the \(i\)-th rollout, define the truncated cumulative local entropy by
$
S_{\max}^{(p,i)}
:=
\sum_{t=1}^{N_{\max}^{(p,i)}}
\widetilde H(G_t^{(p,i)} \mid x_p),
$
which measures the total branching uncertainty along the generated trajectory up to truncation. We then define the Monte Carlo estimator (see \autoref{alg:tmstar_est} for more detail)
$$
\widehat{\mathrm{CE}}^{\star}_{\max,P,M}
:=
\frac{1}{PM}
\sum_{p=1}^P
\sum_{i=1}^M
S_{\max}^{(p,i)}. 
$$ 
{The estimator's convergence and consistency require   minor technical assumptions. First, let $g(T_{\max}) = \mathrm{CE}^{\star}-\mathrm{CE}^{\star}_{\max} $ denote the truncation bias. Naturally, it always satisfies $g(T_{\max}) \geq 0$. Convergence requires \textbf{diminishing truncation bias} --- i.e., $g(T_{\max}) \to 0$ as $T_{\max} \to \infty$. Second, in our estimation, $T_{\max}$  increases as parameter $P, M$ increases. Convergence requires \textbf{modest truncation length} --- i.e., the length cap $T_{\max}(P, M)$ is picked such that $\nu^2_{\max}(\frac1P+\frac1{PM})
\to 0$ as $T_{\max}=T_{\max}(P,M)\to\infty$, where $\nu_{\max}^2:=\mathbb{E}\left[N_{\max}(X)^2\right]$ is the second raw moment of the truncated stopping time. Both conditions are easily satisfied when the tail of $N$ is not too heavy, which we observe in our empirical evaluation (see \autoref{fig:sequence_length_over}). The following theorem characterizes the convergence of $\widehat{\mathrm{CE}}^{\star}_{\max,P,M}$ under these technical conditions.
\begin{theorem}[Consistency and non-asymptotic error bounds of the Monte Carlo estimator]\label{thm:unified_tmstar_prompt}
Assume diminishing truncation bias   
\(g(T_{\max})\) and modest truncation length.  
Then $\widehat{\mathrm{CE}}^{\star}_{\max,P,M}$ converges to
$\mathrm{CE}^{\star}$ in probability when $P,M\to\infty$. Moreover,  $\left|
\widehat{\mathrm{CE}}^{\star}_{\max,P,M}
-
\mathrm{CE}^{\star}
\right|
=
O_{\mathbb P}\left(
\nu_{\max}
\sqrt{
\frac1P+\frac1{PM}
}
+
g(T_{\max})
\right)$, where $\nu_{\max}^2:=\mathbb{E}\left[N_{\max}(X)^2\right].$ 
\end{theorem}
}

\autoref{thm:unified_tmstar_prompt} shows that the estimator $\widehat{\mathrm{CE}}^\star_{\max, P, M}$ is consistent, with estimation error naturally decomposing into a Monte Carlo error term and a truncation bias term. The Monte Carlo error is of order
\(
O_\mathbb{P}\left(\frac{1}{\sqrt{P}} + \frac{1}{\sqrt{PM}}\right),
\)
arising from sampling $P$ prompts and $M$ rollouts per prompt. 

The second component is a truncation bias $g(T_{\max})$, which captures the bias introduced by truncating trajectories at length $T_{\max}$ instead of allowing full generations. Empirically, we observe that the generation length distribution of LLMs typically exhibits exponential decay across tasks and model architectures (see \autoref{fig:sequence_length_over}). This implies that long trajectories are increasingly rare, so the truncation bias $g(T_{\max})$ decays rapidly with $T_{\max}$. We defer the proof of \autoref{thm:unified_tmstar_prompt} to Appendix~\ref{app:proof_tmstar_prompt}. 

Notably, the same Monte Carlo framework extends to the estimation of GenPPL, BF, and prompt-controlled correlation. Their consistency follows by the continuous mapping theorem \citep{mann1943stochastic} together with mild integrability and truncation-bias conditions, which are empirically supported by the observed low truncation rates and rapidly decaying active-rollout distributions (see Appendix~\ref{app:other_estimators} for detailed proofs and discussions).

\begin{remark}[Connecting  truncation bias to the distribution tail of \(N\)] \label{rem:g_tail_prompt}
The truncation bias can be expressed as 

\begin{equation}\label{eq:truncation_bias}
g(T_{\max}) =\mathrm{CE}^{\star}-\mathrm{CE}^{\star}_{\max} = 
\mathbb{E}\left[
\sum_{t=T_{\max}+1}^{N(X)}
\widetilde H(G_t\mid X)
\mathbf 1\{N(X)>T_{\max}\}
\right].
\end{equation}
Let $H_{\max}$ denote the entropy upper bound, i.e.,  $\widetilde H(G_t\mid X) \leq H_{\max}$ for any $G_t$. Then we have
$
g(T_{\max})
$$\le
 H_{\max}
\mathbb{E}_{X\sim p_X}
\left[
\mathbb{E}\left[
(N(X)-T_{\max})_+
\middle| X
\right]
\right] 
$ 
by the tail-sum formula  
$
\mathbb{E}\left[
(N(x)-T_{\max})_+
\middle| X=x
\right]
= 
\sum_{t=T_{\max}+1}^{\infty}
\mathbb{P}(N(x)\ge t\mid X=x).
$
Therefore,
$H_{\max}
\mathbb{E}_{X\sim p_X}
\left[
\sum_{t=T_{\max}+1}^{\infty}
\mathbb{P}(N(X)\ge t\mid X)
\right],
$ with $H_{\max}:=\log |\mathcal Z|$, is a natural upper bound of $g(T_{\max})$, showing that the truncation bias is governed by the prompt-averaged tail behavior of the stopping time. We refer readers to \autoref{app:g_examples_prompt} for examples of prompt-averaged bias decay rates.
\end{remark}

\section{Empirical studies}\label{sec:results}
In this section, we present our experimental results, organized into two complementary parts. First, we examine the global uncertainty structure of the LLM generation space using the metrics introduced in \autoref{sec:method}, including $\mathrm{CE}^\star$, GenPPL, BF, and length--entropy correlation. Second, building on this analysis, we investigate how uncertainty is translated into semantic diversity through a regression framework that links entropy rate to semantic diversity. We evaluate three model families—LLaMA-3.1-8B \citep{grattafiori2024llama}, Qwen3-8B \citep{yang2025qwen3}, and Gemma-3-12B \citep{gemmateam2025gemma3technicalreport}—testing both their pre-trained (base) and instruction-tuned (instruct) versions. We evaluate these models on four tasks spanning a spectrum from highly structured settings, including mathematical reasoning \citep{cobbe2021gsm8k}, coding \citep{zhuo2024bigcodebench}, and sentence completion \cite{zellers2019hellaswag}, to more open-ended creative story generation \citep{ismayilzada2025creativepreferenceoptimization}. For each task, we sample $P=100$ prompts, and for each sampled prompt $X_p$, we generate $M=100$ independent rollouts using stochastic decoding with temperature equal to $1$ and top-$k=100$ (renormalized next-token probabilities),\footnote{Specifically, the one step local entropy $\tilde{H}(G_t|X)$ is computed as $-\sum_{z \in \mathcal{Z}_{100}} q(z|G_t, X) \log q(z|G_t, X)$, where $\mathcal{Z}_{100}$ is the set of the top-100 tokens (including \texttt{EOS} if present) and $q$ is the renormalized distribution over the sampled subset.} where each rollout produces a sequence $Y_{1:N}$ with stopping time $N$ up to a maximum of 4000 new tokens, i.e. $T_{\max} = 4000$ (see \autoref{app:stopping_criteria} for detailed description). To ensure consistent evaluation across model families, we disable the thinking mode in Qwen3-8B Instruct.

\subsection{Global effects of fine-tuning on uncertainty}
 
Across all model families and tasks, we observe three consistent patterns.

\textbf{Fine-tuning reduces both token-level and trajectory-level normalized uncertainty.}
Instruction fine-tuning consistently reduces both token-level (measured by GenPPL) and trajectory-level (measured by BF) normalized uncertainty, as evidenced by systematic declines in GenPPL and BF across nearly all models and tasks (see \autoref{tab:global_uncertainty}). Nevertheless, the magnitude of this reduction varies substantially across settings. For example, Qwen3-8B on mathematical reasoning exhibits only a mild decrease of roughly $5\%$, whereas many other model--task combinations show much larger reductions, often in the range of $60\%$--$80\%$. These results suggest that fine-tuning contracts the effective generation space at both the token and trajectory levels, producing outputs that are more concentrated and deterministic. At the same time, the heterogeneous magnitude of the reduction suggests that alignment does not impose a uniform compression of uncertainty. Instead, the extent of contraction depends on the underlying task structure and model behavior, a phenomenon we discuss in the following paragraph. Overall, our findings are consistent with prior observations that alignment reduces diversity and increases probability concentration \citep{wang2025optimizing, yang2025llm, lake2025distributional, west2025base}, with reductions in branching factor broadly comparable in magnitude to those reported by Yang et al. \citep{yang2025llm}.

\begin{table*}[t]
\centering
\small
\setlength{\tabcolsep}{2.2pt}
\renewcommand{\arraystretch}{1.14}
\caption{Task-dependent changes in generation-space uncertainty after instruction tuning across different model families. We report three complementary uncertainty measures: GenPPL, BF, and $\mathrm{CE^{\star}}$ (see \autoref{app:other_estimators}).
For each task and model family, we compare the base and instruct variants and report the relative percentage change after instruction tuning. 
Standard errors (SE) are estimated via prompt-cluster bootstrapping (2K iterations, resampled with replacement) as in \autoref{app:bootstrap}.
Blue cells indicate reductions in uncertainty after instruction tuning, while red cells indicate increases. Darker color intensity corresponds to larger magnitude changes. }
\label{tab:global_uncertainty}
\resizebox{\textwidth}{!}{%
\begin{tabular}{llccccccccc}
\toprule
\multirow{2}{*}{Task} & \multirow{2}{*}{Type}
& \multicolumn{3}{c}{\textbf{GenPPL}}
& \multicolumn{3}{c}{\textbf{BF}}
& \multicolumn{3}{c}{$\mathbf{CE}^{\star}$} \\
\cmidrule(lr){3-5}\cmidrule(lr){6-8}\cmidrule(lr){9-11}
&
& Qwen3-8B & Llama-3.1-8B & Gemma-3-12B
& Qwen3-8B & Llama-3.1-8B & Gemma-3-12B
& Qwen3-8B & Llama-3.1-8B & Gemma-3-12B \\
\midrule

\multirow{3}{*}{\textbf{Math}}
& Base
& 1.23 $\pm$ 0.0052 & 3.57 $\pm$ 0.041 & 4.56 $\pm$ 0.052
& 1.23 $\pm$ 0.0064 & 5.31 $\pm$ 0.12 & 4.98 $\pm$ 0.055
& 55.7 $\pm$ 2.0 & 458 $\pm$ 18 & 487 $\pm$ 11 \\

& Instruct
& 1.17 $\pm$ 0.0067 & 1.76 $\pm$ 0.076 & 1.16 $\pm$ 0.0091
& 1.16 $\pm$ 0.0057 & 1.51 $\pm$ 0.025 & 1.15 $\pm$ 0.0052
& 40.5 $\pm$ 2.5 & 108 $\pm$ 12 & 56.7 $\pm$ 4.7 \\

& \% Change
& \cellcolor{colddark!4!coldlight} -4.44 $\pm$ 0.47\%
& \cellcolor{colddark!51!coldlight} -50.8 $\pm$ 2.3\%
& \cellcolor{colddark!74!coldlight} -74.5 $\pm$ 0.33\%
& \cellcolor{colddark!6!coldlight} -5.76 $\pm$ 0.38\%
& \cellcolor{colddark!72!coldlight} -71.5 $\pm$ 0.71\%
& \cellcolor{colddark!77!coldlight} -76.9 $\pm$ 0.24\%
& \cellcolor{colddark!27!coldlight} -27.3 $\pm$ 3.3\%
& \cellcolor{colddark!76!coldlight} -76.3 $\pm$ 2.8\%
& \cellcolor{colddark!88!coldlight} -88.4 $\pm$ 1.0\% \\

\midrule

\multirow{3}{*}{\textbf{Coding}}
& Base
& 1.68 $\pm$ 0.017 & 3.16 $\pm$ 0.055 & 2.60 $\pm$ 0.041
& 2.08 $\pm$ 0.034 & 4.96 $\pm$ 0.057 & 3.24 $\pm$ 0.038
& 207 $\pm$ 4.8 & 326 $\pm$ 13 & 492 $\pm$ 25 \\

& Instruct
& 1.23 $\pm$ 0.0053 & 1.27 $\pm$ 0.0075 & 1.04 $\pm$ 0.0023
& 1.21 $\pm$ 0.0061 & 1.27 $\pm$ 0.0070 & 1.04 $\pm$ 0.0018
& 81.2 $\pm$ 4.0 & 126 $\pm$ 3.9 & 16.2 $\pm$ 1.0 \\

& \% Change
& \cellcolor{colddark!27!coldlight} -27.1 $\pm$ 0.79\%
& \cellcolor{colddark!60!coldlight} -59.6 $\pm$ 0.71\%
& \cellcolor{colddark!60!coldlight} -59.9 $\pm$ 0.64\%
& \cellcolor{colddark!42!coldlight} -41.9 $\pm$ 0.99\%
& \cellcolor{colddark!74!coldlight} -74.4 $\pm$ 0.33\%
& \cellcolor{colddark!68!coldlight} -67.8 $\pm$ 0.38\%
& \cellcolor{colddark!61!coldlight} -60.7 $\pm$ 1.6\%
& \cellcolor{colddark!61!coldlight} -61.4 $\pm$ 1.7\%
& \cellcolor{colddark!97!coldlight} -96.7 $\pm$ 0.28\% \\

\midrule

\multirow{3}{*}{\textbf{Sentence Completion}}
& Base
& 4.77 $\pm$ 0.16 & 6.73 $\pm$ 0.12 & 8.81 $\pm$ 0.11
& 4.20 $\pm$ 0.11 & 7.80 $\pm$ 0.15 & 9.49 $\pm$ 0.11
& 826 $\pm$ 55 & 969 $\pm$ 40 & 1084 $\pm$ 36 \\

& Instruct
& 1.68 $\pm$ 0.025 & 2.32 $\pm$ 0.060 & 1.78 $\pm$ 0.019
& 1.55 $\pm$ 0.018 & 2.53 $\pm$ 0.042 & 1.71 $\pm$ 0.024
& 112 $\pm$ 11 & 188 $\pm$ 9.4 & 513 $\pm$ 29 \\

& \% Change
& \cellcolor{colddark!65!coldlight} -64.9 $\pm$ 1.1\%
& \cellcolor{colddark!66!coldlight} -65.6 $\pm$ 1.1\%
& \cellcolor{colddark!80!coldlight} -79.8 $\pm$ 0.30\%
& \cellcolor{colddark!63!coldlight} -63.0 $\pm$ 0.88\%
& \cellcolor{colddark!68!coldlight} -67.6 $\pm$ 0.81\%
& \cellcolor{colddark!82!coldlight} -82.0 $\pm$ 0.30\%
& \cellcolor{colddark!87!coldlight} -86.5 $\pm$ 1.0\%
& \cellcolor{colddark!81!coldlight} -80.6 $\pm$ 1.2\%
& \cellcolor{colddark!53!coldlight} -52.6 $\pm$ 2.9\% \\

\midrule

\multirow{3}{*}{\textbf{Story Generation}}
& Base
& 3.57 $\pm$ 0.11 & 7.40 $\pm$ 0.071 & 7.98 $\pm$ 0.062
& 4.03 $\pm$ 0.092 & 8.24 $\pm$ 0.073 & 8.68 $\pm$ 0.070
& 445 $\pm$ 20 & 1320 $\pm$ 41 & 692 $\pm$ 23 \\

& Instruct
& 2.28 $\pm$ 0.037 & 2.93 $\pm$ 0.083 & 2.22 $\pm$ 0.023
& 2.11 $\pm$ 0.039 & 3.07 $\pm$ 0.077 & 2.16 $\pm$ 0.027
& 454 $\pm$ 28 & 471 $\pm$ 19 & 892 $\pm$ 39 \\

& \% Change
& \cellcolor{colddark!36!coldlight} -36.0 $\pm$ 1.3\%
& \cellcolor{colddark!60!coldlight} -60.4 $\pm$ 1.2\%
& \cellcolor{colddark!72!coldlight} -72.2 $\pm$ 0.28\%
& \cellcolor{colddark!48!coldlight} -47.6 $\pm$ 1.3\%
& \cellcolor{colddark!63!coldlight} -62.7 $\pm$ 1.1\%
& \cellcolor{colddark!75!coldlight} -75.1 $\pm$ 0.30\%
& \cellcolor{warmdark!2!warmlight} 1.92 $\pm$ 3.5\%
& \cellcolor{colddark!64!coldlight} -64.3 $\pm$ 1.3\%
& \cellcolor{warmdark!29!warmlight} 29.0 $\pm$ 8.0\% \\

\bottomrule
\end{tabular}%
}
\end{table*}

\textbf{The effect of fine-tuning on canopy entropy is \emph{task-dependent} and could exhibit opposite trends.} 
In structured domains such as math, coding, and sentence completion, $\mathrm{CE}^\star$ decreases significantly after fine-tuning. In contrast, story generation exhibits a different pattern: Qwen3-8B shows a slight increase of roughly $2\%$, while Gemma-3-12B shows a much larger increase of about $29\%$ (see \autoref{tab:global_uncertainty}). This suggests that fine-tuning does not uniformly reduce canopy entropy, but instead behaves differently in constrained versus open-ended generation settings.

For deterministic tasks such as math and coding, the valid generation space is naturally constrained, so fine-tuning acts as a pruning force that collapses the model’s generation space into a small set of correct trajectories. Consequently, both stopping-time variance and local token-level entropy are significantly reduced. Meanwhile, fine-tuned models generally exhibit a rightward shift in sequence length distributions across tasks (see \autoref{fig:sequence_length_over}), reflecting a tendency toward longer and more elaborative generations. In story generation, this contrast is further accentuated by the fact that base models frequently fall into ``completion traps'', treating prompts as short text continuation tasks rather than directly answering the underlying questions (see \autoref{app:completion_traps}).

These differences in generation behavior help explain why the effect of instruction tuning on $\mathrm{CE}^\star$ is task-dependent. Although fine-tuning reduces local branching diversity, it can also increase stopping-time uncertainty and sustain generation over much deeper trajectories. As shown in \autoref{fig:entropy_rate} and \autoref{fig:sequence_length_over}, while base models have higher local entropy rates (\textit{width}), they may lack sufficient trajectory length (\textit{depth}) to accumulate a large Canopy Entropy ($\mathrm{CE}^\star$). Since $\mathrm{CE}^\star$ acts as a measure of the total \textit{volume} of the generation tree by accumulating uncertainty across the entire rollout -- including stopping behavior, the increase in trajectory depth can outweigh the reduction in local branching.

These findings suggest that fine-tuning changes not only the magnitude of uncertainty, but also how it is allocated across generation trajectories, particularly through its interaction with generation length and stopping behavior.

\textbf{Fine-tuning predominantly increases length--entropy correlation.} To better understand how uncertainty is distributed across trajectories, we analyze the relationship between output length $N$ and entropy rate $r_N$. 
Empirically, fine-tuning consistently shifts the correlation (see \autoref{def:correlation}) between $N$ and $r_N$ from negative toward positive values in structured tasks, while generally making the correlation less negative in open-ended tasks such as story generation. \autoref{tab:er_len_r_change} shows the concrete results under Pearson correlation, with positive shifts in 11 of the 12 model--task comparisons. Similar results based on Spearman correlation and Kendall's $\tau$ exhibit the same trends in \autoref{tab:er_len_rho_tau_change}.

This reveals a fundamental change in generation behavior. In base models, longer outputs are often associated with lower entropy rates, meaning later tokens become increasingly redundant and contribute little new information. In contrast, fine-tuned models generally exhibit much higher length--entropy correlation: longer generations tend to sustain higher entropy rates and remain informative throughout the trajectory (see \autoref{fig:entropy_rate}).

The distinction is especially important in the presence of length confounding. As shown in \autoref{fig:sequence_length_over}, fine-tuned models generally produce longer outputs than their base counterparts, consistent with observations in prior work \citep{singhal2023long}. Under conventional diversity metrics that do not explicitly account for length \citep{tevet2021evaluating, li2016diversity}, these longer generations can artificially appear less diverse due to length-dependent normalization effects, potentially biasing diversity comparisons \citep{friedman2022vendi, shaib2024standardizing}. The positive length--entropy correlation observed here demonstrates that increased length in fine-tuned models reflects sustained information content rather than simple redundancy. Under five-shot prompting, positive shifts persist in the majority of model--task comparisons (see \autoref{app:fewshot}).

Overall, these results suggest that fine-tuning does not merely reduce the quantity of uncertainty, but also improves its organization and efficiency. Rather than uniformly suppressing diversity, fine-tuning restructures uncertainty to better align with task requirements -- constraining the generation space while preserving informative variation across longer trajectories. Our findings on generation space uncertainty and length–entropy coupling remain robust under top-\(k=5000\) entropy estimation (see \autoref{app:topk}).

\begin{figure}[t]
    \centering
    \includegraphics[width=\linewidth]{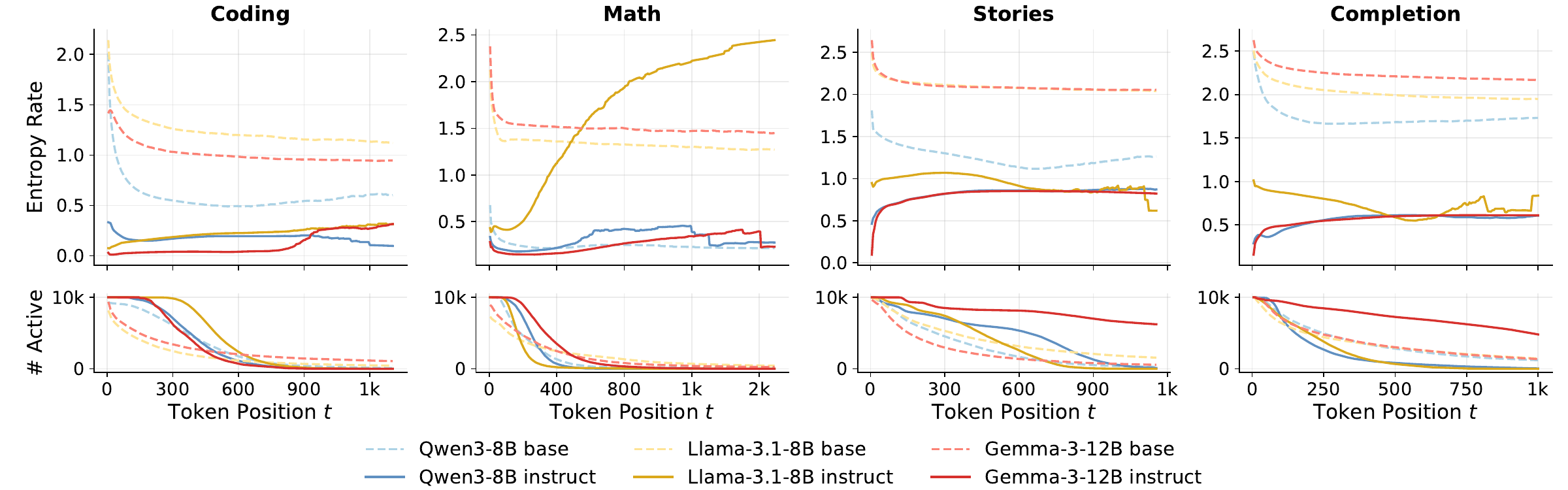}
    \caption{Running entropy rate vs. token position. We plot mean running entropy rate $\bar R_{\le t} = \frac{1}{t}\sum_{s=1}^{t} H(Z_s \mid G_{s}, X)$ averaged across active rollouts at position $t$, with the corresponding active rollout count beneath. A rollout is considered active at position $t$ if its generated sequence has length $N \geq t$ (i.e., it has not terminated before position $t$). Dashed curves denote base models and solid curves denote fine-tuned models.}
    \label{fig:entropy_rate}
\end{figure}

\begin{figure}[t]
    \centering
    \includegraphics[width=\linewidth]{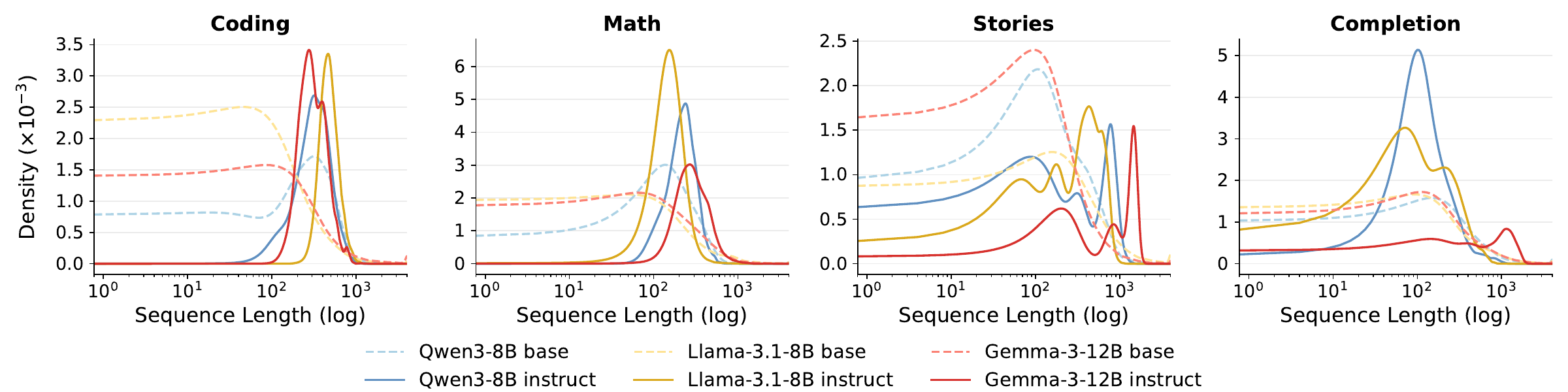}
    \caption{Kernel density estimates of generated sequence lengths across task domains with sequence length shown on a logarithmic scale.}
    \label{fig:sequence_length_over}
\end{figure}


\begin{table*}[t]
\centering
\small
\setlength{\tabcolsep}{1.9pt}
\renewcommand{\arraystretch}{1.10}

\caption{
Base--instruct comparison of the length--entropy Pearson correlation, measured by $\rho(N, r_N)$ (see \autoref{def:correlation}), for the Qwen3-8B, Llama-3.1-8B, and Gemma-3-12B model families. Each entry reports the estimate $\pm$ standard error, where the absolute change is computed as the instruct value minus the corresponding base value.
}
\label{tab:er_len_r_change}
\resizebox{\textwidth}{!}{%
\begin{tabular}{llcccccccccccc}
\toprule
\multirow{2}{*}{Metric} &
\multirow{2}{*}{Model}
& \multicolumn{3}{c}{\textbf{MATH}}
& \multicolumn{3}{c}{\textbf{CODING}}
& \multicolumn{3}{c}{\textbf{SENTENCE COMPLETION}}
& \multicolumn{3}{c}{\textbf{STORY GENERATION}} \\
\cmidrule(lr){3-5}
\cmidrule(lr){6-8}
\cmidrule(lr){9-11}
\cmidrule(lr){12-14}
&
& Base & Instruct & Abs. Change
& Base & Instruct & Abs. Change
& Base & Instruct & Abs. Change
& Base & Instruct & Abs. Change \\
\midrule

\multirow{3}{*}{\textbf{$\rho(N, r_N)$}}

& Qwen
& \cellcolor{colddark!2!coldlight} -0.018 $\pm$ 0.016
& \cellcolor{warmdark!9!warmlight} 0.087 $\pm$ 0.037
& \cellcolor{warmdark!11!warmlight} 0.11 $\pm$ 0.042

& \cellcolor{colddark!41!coldlight} -0.41 $\pm$ 0.014
& \cellcolor{warmdark!61!warmlight} 0.61 $\pm$ 0.034
& \cellcolor{warmdark!100!warmlight} 1.0 $\pm$ 0.037

& \cellcolor{warmdark!6!warmlight} 0.062 $\pm$ 0.016
& \cellcolor{warmdark!52!warmlight} 0.52 $\pm$ 0.023
& \cellcolor{warmdark!46!warmlight} 0.46 $\pm$ 0.027

& \cellcolor{colddark!27!coldlight} -0.27 $\pm$ 0.017
& \cellcolor{warmdark!16!warmlight} 0.16 $\pm$ 0.021
& \cellcolor{warmdark!43!warmlight} 0.43 $\pm$ 0.027 \\

& Llama
& \cellcolor{colddark!27!coldlight} -0.27 $\pm$ 0.016
& \cellcolor{warmdark!38!warmlight} 0.38 $\pm$ 0.041
& \cellcolor{warmdark!65!warmlight} 0.65 $\pm$ 0.045

& \cellcolor{colddark!38!coldlight} -0.38 $\pm$ 0.011
& \cellcolor{warmdark!26!warmlight} 0.26 $\pm$ 0.018
& \cellcolor{warmdark!64!warmlight} 0.64 $\pm$ 0.019

& \cellcolor{colddark!26!coldlight} -0.26 $\pm$ 0.015
& \cellcolor{colddark!29!coldlight} -0.29 $\pm$ 0.021
& \cellcolor{colddark!3!coldlight} -0.028 $\pm$ 0.027

& \cellcolor{colddark!24!coldlight} -0.24 $\pm$ 0.012
& \cellcolor{colddark!13!coldlight} -0.13 $\pm$ 0.016
& \cellcolor{warmdark!11!warmlight} 0.11 $\pm$ 0.022 \\

& Gemma
& \cellcolor{colddark!13!coldlight} -0.13 $\pm$ 0.014
& \cellcolor{warmdark!23!warmlight} 0.23 $\pm$ 0.039
& \cellcolor{warmdark!36!warmlight} 0.36 $\pm$ 0.041

& \cellcolor{colddark!27!coldlight} -0.27 $\pm$ 0.012
& \cellcolor{warmdark!40!warmlight} 0.40 $\pm$ 0.036
& \cellcolor{warmdark!68!warmlight} 0.68 $\pm$ 0.038

& \cellcolor{colddark!23!coldlight} -0.23 $\pm$ 0.015
& \cellcolor{warmdark!15!warmlight} 0.15 $\pm$ 0.038
& \cellcolor{warmdark!39!warmlight} 0.39 $\pm$ 0.040

& \cellcolor{colddark!17!coldlight} -0.17 $\pm$ 0.012
& \cellcolor{colddark!1!coldlight} -0.015 $\pm$ 0.026
& \cellcolor{warmdark!16!warmlight} 0.16 $\pm$ 0.031 \\
\bottomrule
\end{tabular}%
}
\end{table*}

\subsection{Fine-tuning more effectively converts token uncertainty into semantic diversity}
 
After establishing how fine-tuning reshapes entropy allocation, we next investigate how these changes in entropy rate translate into semantic diversity \citep{guo2025benchmarking}. Specifically, we study how generation uncertainty is converted into semantic diversity across sampled trajectories.

For each task--prompt--model tuple $(t,p,m)$, we measure semantic diversity by computing the average pairwise cosine distance between generated outputs:
\(
D_{tpm}
=
\frac{2}{M(M-1)}
\sum_{i<j}
d(e_{t,p,m,i}, e_{t,p,m,j})
\) \citep{tevet2021evaluating},
where $e_{t,p,m,i}$ denotes the embedding of the $i$-th Monte Carlo rollout, prepended with the retrieval prefix (\texttt{"search\_document:"}) to ensure alignment with the encoder's training objective for ModernBERT embedding model (modernbert-embed-large) \citep{modernbert}. Here $d(\cdot,\cdot)$ denotes cosine distance. Since cosine distance is bounded in $(0,1)$, we employ the Beta mixed-effects regression, which is a standard method to accommodate  bounded responses while allowing flexible modeling of dispersion and hierarchical variation \citep{figueroa2013mixed}.\footnote{In contrast, Gaussian linear models are inappropriate for bounded responses and exhibit clear heteroscedasticity and residual misspecification in diagnostic analyses.}

Formally, we model
\(
D_{tpm}
\sim
\operatorname{Beta}
\!\left(
\mu_{tpm}\phi_{tpm},
(1-\mu_{tpm})\phi_{tpm}
\right),
\)
where \(\mu_{tpm}\in(0,1)\) denotes the conditional mean and \(\phi_{tpm}>0\) is the precision parameter. The conditional mean model is specified as
\(\operatorname{logit}(\mu_{tpm})=
\alpha+\beta_1 R_{tpm}
+\beta_2 \mathbf{1}\{m=\mathrm{FT}\}+
f_0\left(\mathrm{invN}_{tpm}\right)
+
\sum_k
\theta_k \mathbf{1}\{t=k\}
+
\beta_3
(
R_{tpm}
\cdot
\mathbf{1}\{m=\mathrm{FT}\}
)
+
f_1\left(\mathrm{invN}_{tpm}\right)
\mathbf{1}\{m=\mathrm{FT}\}
+
u_p
+
v_m
\),
where \(R_{tpm}\) is the entropy rate, and \(\mathrm{invN}_{tpm}\) is the inverse sequence length, both standardized to zero mean and unit variance, so the entropy-rate coefficients are reported per standard deviation. \(f_0(\cdot)\), \(f_1(\cdot)\) are natural spline functions corresponding to nonlinear length effects and their interaction with instruction tuning. {For consistency with the measurement of semantic diversity, the reported sequence length and entropy rate are likewise computed by averaging over the same set of $M$ rollouts for each task--prompt--model combination. The coefficients \(\theta_k\) represent fixed task effects.} To account for hierarchical dependence, we include random intercepts
\(
u_p \sim \mathcal N(0,\sigma_p^2),
v_m \sim \mathcal N(0,\sigma_m^2),
\)
for prompt identity and model family, respectively. To capture heteroscedasticity, we also model the precision parameter $\phi_{tpm}$ as a function of task, model variant, entropy rate, sequence length, and model family. Residual diagnostics (see \autoref{fig:qq_plot}) indicate good calibration of the mean structure (Kolmogorov-Smirnov Test \citep{massey1951kolmogorov}: $p=0.47$) and no significant dispersion issues ($p=0.11$) or outliers ($p=0.084$).

\textbf{Main findings.} First, entropy rate exhibits a strong and highly significant positive association with semantic diversity. The coefficient on entropy rate is large and significant ($\beta_1 = 0.688$, $p < 2\times 10^{-16}$), indicating that higher per-token uncertainty leads to substantially greater semantic diversity across generated outputs. More importantly, the interaction between entropy rate and fine-tuning is strongly positive ($\beta_3 = 1.323$, $p < 2\times 10^{-16}$). Under the logit link, this implies a substantial amplification of the slope:
\(
\partial \operatorname{logit}(\mu)/\partial R
= \beta_1 = 
0.688 \; \text{(base)},
\;
\beta_1 + \beta_3 = 0.688 + 1.323 = 2.011\; \text{(fine-tuned)}.
\)
Thus, fine-tuning nearly triples the sensitivity of semantic diversity to entropy rate. This indicates that uncertainty in fine-tuned models is significantly more \emph{semantically productive}: small increases in token-level uncertainty translate into much larger increases in trajectory-level diversity. These findings remain robust across various alternative embedding models, semantic diversity metrics, and regression specifications (see \autoref{app:semantic_robustness}).

At the same time, the main effect of fine-tuning is negative ($\beta_2 = -0.325$, $p < 2\times10^{-4}$), indicating that at fixed entropy rate and length, fine-tuned models produce less baseline semantic diversity. This is consistent with alignment compressing the generation space and enforcing more structured outputs. However, the strong positive interaction shows that the remaining uncertainty is used more efficiently. We observe a nonlinear and generally negative relationship between semantic diversity and sequence length, making length control essential in diversity analysis (see \autoref{tab:beta_conditional_model}). Full regression results are reported in \autoref{app:beta_regression}, along with additional model specifications and regression details.


\section{Conclusion and limitations}\label{sec:conclusion}
 
In this work, we revisit diversity in large language models through a trajectory-level information-theoretic perspective. By introducing Canopy Entropy and length--entropy coupling, we show that fine-tuning does not simply reduce uncertainty or shrink the generation space. Instead, aligned models reorganize uncertainty more efficiently across generation trajectories, producing longer generations that remain more semantically informative. Our results highlight the importance of modeling how uncertainty evolves throughout the rollout process, beyond aggregate diversity metrics alone.

Our work also has several limitations. Entropy-based measures characterize uncertainty allocation, but do not directly capture factuality or human preference quality. Our framework is descriptive rather than causal: although we observe systematic changes in length--entropy coupling after fine-tuning, the underlying   mechanisms causing this effect remain an important direction for future work.

\newpage

\bibliography{ref}
\bibliographystyle{plain}

\appendix
\newpage
\section{Missing proofs and algorithms}\label{sec:proof}

\subsection{Equivalence to a two-stage stochastic decision process}\label{app:two_stage_equivalence}
\begin{theorem}[Equivalence between autoregressive generation and a two-stage decision process]
\label{thm:two_stage_equivalence}
Fix a prompt \(X=x\) and a generation history \(G_t = g_t=(y_1,\ldots,y_{t-1})\). Suppose the autoregressive model defines a probability distribution over the extended vocabulary $\mathcal \mathcal Z := \mathcal V \cup \{\texttt{EOS}\},$ denoted by $q(z \mid g_t, x), z \in \mathcal \mathcal Z.$ Then this one-step generation rule is equivalent to the following two-stage stochastic decision process: $A_t \in \{\textsf{STOP},\textsf{CONT}\},$ where
$\mathbb P(A_t=\textsf{STOP}\mid g_t,x)=q(\texttt{EOS}\mid g_t,x),$ and $\mathbb P(A_t=\textsf{CONT}\mid g_t,x)=
1-q(\texttt{EOS}\mid g_t,x).$ Conditional on \(A_t=\textsf{CONT}\), the next token \(Y_t\in\mathcal V\) is sampled from $\mathbb P(Y_t=v\mid g_t,x,A_t=\textsf{CONT})=\frac{q(v\mid g_t,x)}
{1-q(\texttt{EOS}\mid g_t,x)},v\in\mathcal V,$ whenever \(1-q(\texttt{EOS}\mid g_t,x)>0\). This two-stage process induces the same distribution over generated sequences as the original autoregressive model.
\end{theorem}
\begin{proof}
Fix a prompt \(x\) and history \(g_t\). Under the original autoregressive model, the next outcome is sampled directly from $q(z\mid g_t,x), z\in \mathcal V\cup\{\texttt{EOS}\}.$
In particular, generation terminates at step \(t\) with probability $q(\texttt{EOS}\mid g_t,x).$ Now consider the two-stage process. The probability of termination is $\mathbb P(A_t=\textsf{STOP}\mid g_t,x)=q(\texttt{EOS}\mid g_t,x),$ which matches the probability assigned to \(\texttt{EOS}\) by the original model. Next, for any token \(v\in\mathcal V\), the probability that the two-stage process generates \(v\) is
\[
\begin{aligned}
\mathbb P(Y_t=v\mid g_t,x)
&=
\mathbb P(A_t=\textsf{CONT}\mid g_t,x)
\,
\mathbb P(Y_t=v\mid g_t,x,A_t=\textsf{CONT}) \\
&=
\left(1-q(\texttt{EOS}\mid g_t,x)\right)
\frac{q(v\mid g_t,x)}
{1-q(\texttt{EOS}\mid g_t,x)} \\
&=
q(v\mid g_t,x).
\end{aligned}
\]
Thus, the two-stage process assigns exactly the same probability to every possible next outcome:
\[
\mathbb P(\textsf{STOP}\mid g_t,x)
=
q(\texttt{EOS}\mid g_t,x),
\;
\mathbb P(Y_t=v\mid g_t,x)
=
q(v\mid g_t,x),
\quad v\in\mathcal V.
\]
Therefore, the one-step transition distribution is identical under the two formulations. Since this equality holds for every prompt \(x\), every generation history \(g_t\), and every generation step \(t\), the induced probability of any finite sequence \(y_{1:N}\) terminating at length \(N\) is the same under both processes. Thus, the two-stage stochastic decision process is equivalent to the original autoregressive generation model.
\end{proof}

\subsection{Proof of \autoref{thm:entropy}}\label{app:proof_thm1}
\begin{proof}
Let \(X \sim p_X\) denote the prompt distribution. For each realization \(X=x\), consider the sequence of random variables \((Z_t)_{t \ge 1}\) with \(Z_t \in \mathcal Z := \mathcal V \cup \{\texttt{EOS}\}\), where selecting \(\texttt{EOS}\) indicates termination. Define the stopping time $N := \min\{t : Z_t = \texttt{EOS}\},$
and let \(Y_t = Z_t\) for \(t < N\). Adopt the convention that \(Z_t = \texttt{EOS}\) for all \(t > N\), which induces a one-to-one correspondence between the infinite sequence \((Z_1, Z_2, \ldots)\) and the pair \((N, Y_{1:N})\), conditional on \(X\). Therefore,
\[
H(N, Y_{1:N} \mid X)
=
H(Z_1, Z_2, \ldots \mid X).
\]

By the chain rule of entropy,
\(
H(Z_1, Z_2, \ldots \mid X)
=
\sum_{t=1}^{\infty}
H(Z_t \mid Z_{1:t-1}, X).
\)
Since the history \(G_t\) is completely determined by \((Z_{1:t-1}, X)\), we have
\[
H(Z_t \mid Z_{1:t-1}, X)
=
H(Z_t \mid G_t, X)
=
\mathbb{E}_{G_t, X}\big[\widetilde H(G_t \mid X)\big].
\]
Therefore,
\(
H(Z_1, Z_2, \ldots \mid X)
=
\sum_{t=1}^{\infty}
\mathbb{E}_{G_t, X}\big[\widetilde H(G_t \mid X)\big].
\)

Since \(\widetilde H(G_t \mid X) \ge 0\), we can apply Tonelli's theorem \citep{folland1999real} to exchange the sum and expectation.
\[
\sum_{t=1}^{\infty}
\mathbb{E}_{G_t, X}\big[\widetilde H(G_t \mid X)\big]
=
\mathbb{E}_{G_t, X}\left[
\sum_{t=1}^{\infty}
\widetilde H(G_t \mid X)
\right].
\]

Moreover, for all \(t > N\), \(Z_t = \texttt{EOS}\) deterministically, so \(\widetilde H(G_t \mid X) = 0\). Therefore,
\(
\sum_{t=1}^{\infty}
\widetilde H(G_t \mid X)
=
\sum_{t=1}^{N(X)}
\widetilde H(G_t \mid X)
\), almost surely. Combining the above, we have
\(
H(Z_1, Z_2, \ldots \mid X)
=
\mathbb{E}_{G_t, X}\left[
\sum_{t=1}^{N(X)}
\widetilde H(G_t \mid X)
\right].
\) Finally, since \((Z_1, Z_2, \ldots)\) is in one-to-one correspondence with \((N, Y_{1:N})\) given \(X\), we obtain
\(
H(N, Y_{1:N} \mid X)
=
\mathbb{E}_{G_t, X}\left[
\sum_{t=1}^{N(X)}
\widetilde H(G_t \mid X)
\right]
=
\mathrm{CE}^\star.
\)
\end{proof}
\subsection{Algorithms}

\begin{algorithm}[H]
\caption{Estimating $\mathrm{CE}^{\star}_{\max}$ over a prompt distribution}
\label{alg:tmstar_est}
\KwIn{Prompt distribution $p_X$; number of prompts $P$; rollout count per prompt $M$; maximum length $T_{\max}$; sampling policy $\pi$ (e.g.\ temperature/top-$p$/top-$k$).}
\For{$p=1$ \KwTo $P$}{
    Sample prompt $x_p \sim p_X$;\\

    \For{$i=1$ \KwTo $M$}{
        Initialize history $G^{(p,i)}_1 \leftarrow x_p$; set $S_{\max}^{(p,i)} \leftarrow 0$;\\

        \For{$t=1$ \KwTo $T_{\max}$}{
            Compute model distribution
            \(
            q_t^{(p,i)}(\cdot \mid G_t^{(p,i)}, x_p)
            \)
            under sampling policy $\pi$, including \texttt{EOS} as an outcome;\\

            Compute one-step local entropy
            \(
            \tilde H(G_t^{(p,i)} \mid X_p)
            \leftarrow
            -\sum_{z\in\mathcal Z}
            q_t^{(p,i)}(z\mid G_t^{(p,i)},x_p)
            \log q_t^{(p,i)}(z\mid G_t^{(p,i)},x_p);
            \)

            Update
            \(
            S_{\max}^{(p,i)}
            \leftarrow
            S_{\max}^{(p,i)}+\tilde H(G_t^{(p,i)} \mid X_p);
            \)

            Sample next outcome
            \(
            z_t^{(p,i)}
            \sim
            q_t^{(p,i)}(\cdot \mid G_t^{(p,i)},x_p);
            \)

            If $z_t^{(p,i)}=\texttt{EOS}$ then set $N_{\max}^{(p,i)}\leftarrow t$ and break;\\
            Else update history
                \(
                G_{t+1}^{(p,i)}
                \leftarrow
                G_t^{(p,i)}\oplus z_t^{(p,i)}.
                \)
            
        }

        If no \texttt{EOS} is sampled by $T_{\max}$: set $N_{\max}^{(p,i)}\leftarrow T_{\max}$.
        
    }
}

\Return{
\(
\widehat{\mathrm{CE}}^{\star}_{\max,P,M}
=
\frac{1}{PM}
\sum_{p=1}^{P}
\sum_{i=1}^{M}
S_{\max}^{(p,i)},
\;
\widehat{\mathbb E}[N_{\max}]
=
\frac{1}{PM}
\sum_{p=1}^{P}
\sum_{i=1}^{M}
N_{\max}^{(p,i)}.
\)
}
\end{algorithm}
\subsection{Examples of prompt-averaged bias decay rates}\label{app:g_examples_prompt}
\begin{corollary}
\label{cor:g_examples_prompt}
By Remark~\ref{rem:g_tail_prompt}, \(g(T_{\max})\) admits the following rates.

\textbf{Uniform exponential tail.}
If there exist constants \(c,\lambda>0\) such that, for all prompts \(x \in \mathcal X\) and
all \(t\ge 1\),
$
\mathbb{P}(N(x)\ge t\mid X=x)\le c e^{-\lambda t},
$
then
$
g(T_{\max})
\le
c\log |\mathcal Z|
\sum_{t=T_{\max}+1}^{\infty}e^{-\lambda t}
=
c\log |\mathcal Z|
\frac{e^{-\lambda(T_{\max}+1)}}{1-e^{-\lambda}}
=
O(e^{-\lambda T_{\max}}).
$

\textbf{Uniform polynomial tail.}
If there exist constants \(c>0\) and \(\alpha>1\) such that, for all prompts
\(x \in \mathcal X\) and all \(t\ge 1\),
$
\mathbb{P}(N(x)\ge t\mid X=x)\le c t^{-\alpha},
$
then
$
g(T_{\max})
\le
c\log |\mathcal Z|
\sum_{t=T_{\max}+1}^{\infty}t^{-\alpha}
\le
\frac{c\log |\mathcal Z|}{\alpha-1}
T_{\max}^{1-\alpha}
=
O(T_{\max}^{1-\alpha}).
$
\end{corollary}

\subsection{Proof of \autoref{thm:unified_tmstar_prompt} and MSE analysis}\label{app:proof_tmstar_prompt}

\begin{proof}
Given any fixed prompt $x \in \mathcal X$,
denote $\mu_{\max}(x)
:=
\mathbb E\!\left[S_{\max}\mid X=x\right],$ and $S_{\max}
=\sum_{t=1}^{N_{\max}(x)}
\widetilde H(G_t\mid x).$ Then $\mathrm{CE}^{\star}_{\max}
=
\mathbb E_{X \sim p_X}[\mu_{\max}(X)].$ Since \(|\mathcal Z|<\infty\), for all $g_t$, we have $0\le \widetilde H(g_t\mid x)\le H_{\max}:=\log|\mathcal Z|.$ Therefore
\(
0\le S_{\max}\le H_{\max}N_{\max}(x).
\)
Therefore,
\(
\mathbb E[S_{\max}^2]
\le
H_{\max}^2\mathbb E[N_{\max}(X)^2]
=
H_{\max}^2\nu_{\max}^2,
\) where $\nu_{\max}^2:= E[N_{\max}(X)^2]$.

Now decompose the variance of $\widehat{\mathrm{CE}}^{\star}_{\max,P,M}$. Conditional on
\(X_1,\dots,X_P\), the \(M\) rollouts for each prompt are independent, by law of total variance, we have
\[
\operatorname{Var}\!\left(
\widehat{\mathrm{CE}}^{\star}_{\max,P,M}
\right)
=
\operatorname{Var}\!\left(
\frac1P\sum_{p=1}^P \mu_{\max}(X_p)
\right)
+
\mathbb E\!\left[
\operatorname{Var}\!\left(
\widehat{\mathrm{CE}}^{\star}_{\max,P,M}
\mid X_1,\dots,X_P
\right)
\right].
\]
By Jensen's inequality, $\mu_{\max}^2(X) \le H_{\max}^2\,\mathbb{E}\!\left[N_{\max}(X)^2 \mid X\right]$, and taking expectations over $X$ gives $\mathbb{E}[\mu_{\max}^2(X)] \le H_{\max}^2 \nu_{\max}^2$. Therefore
\[
\operatorname{Var}\!\left(
\frac1P\sum_{p=1}^P \mu_{\max}(X_p)
\right)
=
\frac{1}{P}\operatorname{Var}(\mu_{\max}(X))
\le
\frac{1}{P}\mathbb E[\mu_{\max}(X)^2]
\le
\frac{H_{\max}^2\nu_{\max}^2}{P}.
\]
From
\(
\mathbb E\left[
\operatorname{Var}\!\left(
\widehat{\mathrm{CE}}^{\star}_{\max,P,M}
\mid X_1,\dots,X_P
\right)
\right]
\le
\frac{H_{\max}^2\nu_{\max}^2}{PM},
\)
we have
\[
\operatorname{Var}\left(
\widehat{\mathrm{CE}}^{\star}_{\max,P,M}
\right)
\le
H_{\max}^2\nu_{\max}^2
\left(
\frac1P+\frac1{PM}
\right).
\]

Since
\(
\mathbb E\left[
\widehat{\mathrm{CE}}^{\star}_{\max,P,M}
\right]
=
\mathrm{CE}^{\star}_{\max},
\)
Chebyshev's inequality gives that, for any \(\delta\in(0,1)\), with probability at least
\(1-\delta\),
\(
\left|
\widehat{\mathrm{CE}}^{\star}_{\max,P,M}
-
\mathrm{CE}^{\star}_{\max}
\right|
\le
H_{\max}\nu_{\max}
\sqrt{
\frac{1}{\delta}
\left(
\frac1P+\frac1{PM}
\right)
}.
\)
By the truncation bias assumption,
\(
0\le
\mathrm{CE}^{\star}
-
\mathrm{CE}^{\star}_{\max}
\le
g(T_{\max}).
\)
Thus,
\[
\left|
\widehat{\mathrm{CE}}^{\star}_{\max,P,M}
-
\mathrm{CE}^{\star}
\right|
\le
H_{\max}\nu_{\max}
\sqrt{
\frac{1}{\delta}
\left(
\frac1P+\frac1{PM}
\right)
}
+
g(T_{\max}).
\]
Equivalently,
\(
\left|
\widehat{\mathrm{CE}}^{\star}_{\max,P,M}
-
\mathrm{CE}^{\star}
\right|
=
O_{\mathbb P}\left(
\nu_{\max}
\sqrt{
\frac1P+\frac1{PM}
}
+
g(T_{\max})
\right).
\) Thus, if \(T_{\max}=T_{\max}(P,M)\to\infty\),
\(
g(T_{\max})\to 0,
\)
and if
\(
\nu_{\max}(P,M)
\sqrt{
\frac1P+\frac1{PM}
}
\to 0,
\)
then we have
\(
\widehat{\mathrm{CE}}^{\star}_{\max,P,M}
\xrightarrow{\mathbb P}
\mathrm{CE}^{\star}.
\)
\end{proof}

\begin{theorem}[MSE of the prompt-averaged truncated CE estimator]\label{thm:prompt_tmstar_refined_variance}
\[
\mathbb E\left[
\left(
\widehat{\mathrm{CE}}^{\star}_{\max,P,M}
-
\mathrm{CE}^{\star}
\right)^2
\right]
\le
\frac{H_{\max}^2}{P}
\operatorname{Var}_{X \sim p_X}\!\left(
\mathbb E[N_{\max}\mid X]
\right)
+
\frac{H_{\max}^2}{PM}
\mathbb E[N_{\max}^2]
+
g^2(T_{\max}).
\]
\end{theorem}

In \autoref{thm:prompt_tmstar_refined_variance}, we further characterize the mean squared error (MSE) of $\widehat{\mathrm{CE}}^\star_{\max, P, M}$ as the sum of the squared truncation bias and the variance of the estimator. The variance decreases with both the number of prompts $P$ and the total number of rollouts $PM$, highlighting two distinct sources of statistical efficiency: diversity across prompts and repeated sampling within each prompt.

The variance admits a hierarchical decomposition with two distinct sources of randomness: a $1/P$ term due to variability across prompts, and a $1/PM$ term arising from rollout-level stochasticity within each prompt. Importantly, the prompt-level variance cannot be reduced by increasing the number of rollouts per prompt, reflecting the intrinsic hierarchical structure of the estimator. Moreover, the Monte Carlo error is scaled by the effective trajectory length (captured by $T_{\max}$ or, more precisely, $N_{\max}$), since longer generations accumulate more uncertainty along the trajectory.

These results reveal a fundamental trade-off. Increasing $T_{\max}$ reduces the truncation bias by capturing more of the generation process, but also increases computational cost and variance through longer trajectories. 
\begin{proof}
Define the per-prompt rollout average
\(
\bar S_{\max}^{(p)}
:=
\frac{1}{M}
\sum_{i=1}^M
S_{\max}^{(p,i)}
\).
Then
\(
\widehat{\mathrm{CE}}^{\star}_{\max,P,M}
=
\frac{1}{P}
\sum_{p=1}^P
\bar S_{\max}^{(p)}.
\)
Conditional on \(X_p\),
\(
\mathbb E[\bar S_{\max}^{(p)}\mid X_p]
=
\mu_{\max}(X_p),
\)
and, since the \(M\) rollouts are conditionally independent,
\(
\operatorname{Var}(\bar S_{\max}^{(p)}\mid X_p)
=
\frac{1}{M}\sigma_{\max}^2(X_p).
\)
By the law of total variance, we have
\[
\operatorname{Var}(\bar S_{\max}^{(p)})
=
\operatorname{Var}_{X_p}
\left(
\mathbb E[\bar S_{\max}^{(p)}\mid X_p]
\right)
+
\mathbb E_{X_p}
\left[
\operatorname{Var}(\bar S_{\max}^{(p)}\mid X_p)
\right].
\]
Substituting the conditional mean and variance gives
\(
\operatorname{Var}(\bar S_{\max}^{(p)})
=
\operatorname{Var}_{X \sim p_X}(\mu_{\max}(X))
+
\frac{1}{M}
\mathbb E_{X \sim p_X}[\sigma_{\max}^2(X)].
\)
Since the prompts \(X_1,\dots,X_P\) are i.i.d., the variables
\(\bar S_{\max}^{(1)},\dots,\bar S_{\max}^{(P)}\) are i.i.d. Therefore,
\(
\operatorname{Var}
\left(
\widehat{\mathrm{CE}}^{\star}_{\max,P,M}
\right)
=
\frac{1}{P}
\operatorname{Var}(\bar S_{\max}^{(p)})
=
\frac{1}{P}
\operatorname{Var}_{X \sim p_X}(\mu_{\max}(X))
+
\frac{1}{PM}
\mathbb E_{X \sim p_X}[\sigma_{\max}^2(X)].
\)

For the MSE, write
\(
\widehat{\mathrm{CE}}^{\star}_{\max,P,M}
-
\mathrm{CE}^{\star}
=
\left(
\widehat{\mathrm{CE}}^{\star}_{\max,P,M}
-
\mathrm{CE}^{\star}_{\max}
\right)
-
g(T_{\max}).
\)
Taking squares and expectations yields
\(
\mathbb E\left[
\left(
\widehat{\mathrm{CE}}^{\star}_{\max,P,M}
-
\mathrm{CE}^{\star}
\right)^2
\right]
=
\mathbb E\left[
\left(
\widehat{\mathrm{CE}}^{\star}_{\max,P,M}
-
\mathrm{CE}^{\star}_{\max}
\right)^2
\right]
+
(g(T_{\max}))^2,
\)
because
\(
\mathbb E[
\widehat{\mathrm{CE}}^{\star}_{\max,P,M}
-
\mathrm{CE}^{\star}_{\max}
]
=0.
\)

Since pathwise, for any given history $g_t$ and prompt $x$, we have 
\(
0\le
S_{\max}(x)
=
\sum_{t=1}^{N_{\max}(x)}
\widetilde H(g_t\mid x)
\le
H_{\max}N_{\max}(x).
\)
Thus, conditional on \(X=x\),
\(
\operatorname{Var}(S_{\max}(x)\mid X=x)
\le
\mathbb E[S_{\max}^2(x)\mid X=x]
\le
H_{\max}^2
\mathbb E[N_{\max}^2(x)\mid X=x].
\)
Averaging over prompts gives
\(
\mathbb E_X[\sigma_{\max}^2(X)]
\le
H_{\max}^2
\mathbb E_X\mathbb E[N_{\max}^2(X)\mid X]
=
H_{\max}^2\mathbb E[N_{\max}^2].
\)

Similarly,
\(
\mu_{\max}(X)
=
\mathbb E[S_{\max}(X)\mid X]
\le
H_{\max}
\mathbb E[N_{\max}\mid X].
\)
Since both quantities are nonnegative, this gives the prompt-level bound
\(
0\le
\mu_{\max}(X)
\le
H_{\max}
\mathbb E[N_{\max}\mid X].
\)
Therefore,
\(
\operatorname{Var}_X(\mu_{\max}(X))
\le
H_{\max}^2
\operatorname{Var}_X\!\left(
\mathbb E[N_{\max}\mid X]
\right).
\)
Combining these two refined bounds with the exact variance decomposition gives
\[
\operatorname{Var}
\left(
\widehat{\mathrm{CE}}^{\star}_{\max,P,M}
\right)
\le
\frac{H_{\max}^2}{P}
\operatorname{Var}_{X \sim p_X}\!\left(
\mathbb E[N_{\max}\mid X]
\right)
+
\frac{H_{\max}^2}{PM}
\mathbb E[N_{\max}^2].
\]
Adding the squared truncation bias \( g^2(T_{\max})\) gives the stated refined
MSE upper bound.
\end{proof}

\subsection{Estimation and consistency of GenPPL, BF, and length--entropy correlation}\label{app:other_estimators}
\begin{lemma}[Consistency of the truncated length estimator]\label{lem:nmax_consistency}
Let
\(
\widehat N_{\max,P,M}
:=
\frac{1}{PM}
\sum_{p=1}^P\sum_{i=1}^M
N_{\max}^{(p,i)}
\), where
\(
N_{\max}^{(p,i)}
:=
\min\{N^{(p,i)},T_{\max}\}.
\)
Assume \(\mathbb E[N]<\infty\). If \(T_{\max}=T_{\max}(P,M)\to\infty\) and
\(
\frac{\sqrt{\mathbb E[N_{\max}^2]}}{\sqrt{P}}\to 0,
\)
then
\(
\widehat N_{\max,P,M}
\xrightarrow{\mathbb P}
\mathbb E[N].
\)
\end{lemma}

\begin{proof}
Let
\(
\eta_{\max}(x):=\mathbb E[N_{\max}\mid X=x].
\)
By the law of total variance,
\(
\operatorname{Var}(\widehat N_{\max,P,M})
=
\frac{1}{P}\operatorname{Var}_{X \sim p_X}(\eta_{\max}(X))
+
\frac{1}{PM}
\mathbb E_{X \sim p_X}\left[
\operatorname{Var}(N_{\max}\mid X)
\right].
\)
Using
\(
\operatorname{Var}(\eta_{\max}(X))
\le
\mathbb E[\eta_{\max}(X)^2]
\le
\mathbb E[N_{\max}^2],
\)
and
\(
\mathbb E[\operatorname{Var}(N_{\max}\mid X)]
\le
\mathbb E[N_{\max}^2],
\)
we obtain
\(
\operatorname{Var}(\widehat N_{\max,P,M})
\le
\mathbb E[N_{\max}^2]
\left(
\frac{1}{P}
+
\frac{1}{PM}
\right).
\)
Therefore, under the stated condition,
\(
\operatorname{Var}(\widehat N_{\max,P,M})\to 0.
\)

In addition,
\(
\mathbb E[\widehat N_{\max,P,M}]
=
\mathbb E[N_{\max}].
\)
Since \(N_{\max}=\min\{N,T_{\max}\}\uparrow N\) as \(T_{\max}\to\infty\), the monotone convergence theorem \citep{rudin1976principles} gives
\(
\mathbb E[N_{\max}]
\to
\mathbb E[N].
\)
Hence,
\[
\widehat N_{\max,P,M}
-
\mathbb E[N]
=
\left(
\widehat N_{\max,P,M}
-
\mathbb E[N_{\max}]
\right)
+
\left(
\mathbb E[N_{\max}]
-
\mathbb E[N]
\right).
\]
The first term converges to \(0\) in probability by Chebyshev's inequality, and the second term converges to \(0\) deterministically. Therefore,
\(
\widehat N_{\max,P,M}
\xrightarrow{\mathbb P}
\mathbb E[N].
\)
\end{proof}
\begin{corollary}[Consistency of GenPPL]
\label{cor:consistency_genppl_bf}
Let
\(
\widehat{\mathrm{CE}}^\star_{\max,P,M}
=
\frac{1}{PM}\sum_{p=1}^P\sum_{i=1}^M S_{\max}^{(p,i)}
\)
be the Monte Carlo estimator of \(\mathrm{CE}^\star\), and let
\(
\widehat N_{\max,P,M}
:=
\frac{1}{PM}\sum_{p=1}^P\sum_{i=1}^M N_{\max}^{(p,i)}
\)
be the corresponding estimator of \(\mathbb E[N]\). Recall that 
\(
\widehat{\mathrm{CE}}^\star_{\max,P,M}
\xrightarrow{\mathbb P}
\mathrm{CE}^\star
\) (by \autoref{thm:unified_tmstar_prompt}) and 
\(
\widehat N_{\max,P,M}
\xrightarrow{\mathbb P}
\mathbb E[N]
\) (by Lemma~\autoref{lem:nmax_consistency}). Define
\(
\widehat{\mathrm{GenPPL}}_{\max,P,M}
:=
\exp\left(
\frac{
\widehat{\mathrm{CE}}^\star_{\max,P,M}
}{
\widehat N_{\max,P,M}
}
\right),
\;
\mathrm{GenPPL}
:=
\exp\left(
\frac{\mathrm{CE}^\star}{\mathbb E[N]}
\right).
\)
Then
\(
\widehat{\mathrm{GenPPL}}_{\max,P,M}
\xrightarrow{\mathbb P}
\mathrm{GenPPL}.
\)
\end{corollary}

\begin{proof}
Since we have
\(
\widehat{\mathrm{CE}}^\star_{\max,P,M}
\xrightarrow{\mathbb P}
\mathrm{CE}^\star, \
\widehat N_{\max,P,M}
\xrightarrow{\mathbb P}
\mathbb E[N]
\), and \(\mathbb E[N]>0\), the function
\(
f(a,b)=\exp(a/b)
\)
is continuous at \((\mathrm{CE}^\star,\mathbb E[N])\). Therefore, by the continuous mapping theorem \citep{mann1943stochastic},
\(
\exp\left(
\frac{
\widehat{\mathrm{CE}}^\star_{\max,P,M}
}{
\widehat N_{\max,P,M}
}
\right)
\xrightarrow{\mathbb P}
\exp\left(
\frac{\mathrm{CE}^\star}{\mathbb E[N]}
\right).
\)
Hence,
\(
\widehat{\mathrm{GenPPL}}_{\max,P,M}
\xrightarrow{\mathbb P}
\mathrm{GenPPL}
\), which completes the proof.
\end{proof}

\begin{corollary}[Consistency of the BF estimator]\label{cor:bf_consistency}
Let
\(
\widehat L_{\max,P,M}
:=
\frac{1}{PM}
\sum_{p=1}^P\sum_{i=1}^M
\frac{S_{\max}^{(p,i)}}{N_{\max}^{(p,i)}},
\;
\widehat{\mathrm{BF}}_{\max,P,M}
:=
\exp\!\left(\widehat L_{\max,P,M}\right),
\)
where
\(
S_{\max}^{(p,i)}
=
\sum_{t=1}^{N_{\max}^{(p,i)}}
\widetilde H(G_t^{(p,i)}\mid x_p),
\;
N_{\max}^{(p,i)}
=
\min\{N^{(p,i)},T_{\max}\}.
\)
Assume that $N\ge 1$ almost surely and that the normalized truncated entropy variables
\(
r_{\max}
:=
\frac{S_{\max}}{N_{\max}}
\)
are integrable. Then, for fixed $T_{\max}$, as $P,M\to\infty$,
\(
\widehat L_{\max,P,M}
\xrightarrow{\mathbb P}
L_{\max}
:=
\mathbb E\!\left[
\frac{S_{\max}}{N_{\max}}
\right],
\)
and therefore, by the continuous mapping theorem,
\(
\widehat{\mathrm{BF}}_{\max,P,M}
\xrightarrow{\mathbb P}
\mathrm{BF}_{\max}
:=
\exp(L_{\max}).
\)
Moreover, if
\(
\frac{S_{\max}}{N_{\max}}
\xrightarrow[T_{\max}\to\infty]{a.s.}
\frac{S}{N}
\)
and the sequence $\{S_{\max}/N_{\max}\}_{T_{\max}\ge 1}$ is uniformly integrable, then
\(
L_{\max}\to
L
:=
\mathbb E\!\left[
\frac{1}{N}
\sum_{t=1}^{N}
\widetilde H(G_t\mid X)
\right],
\)
and consequently
\(
\mathrm{BF}_{\max}
\to
\mathrm{BF}
:=
\exp(L).
\)
Thus, under these conditions,
\(
\widehat{\mathrm{BF}}_{\max,P,M}
\xrightarrow{\mathbb P}
\mathrm{BF}
\)
as $P,M\to\infty$ and $T_{\max}\to\infty$.
\end{corollary}

\begin{proof}
For fixed $T_{\max}$, the variables
\(
\hat{r}_{\max}^{(p,i)}
=
\frac{S_{\max}^{(p,i)}}{N_{\max}^{(p,i)}}
\)
are integrable rollout-level observations. Since $N\ge 1$ almost surely, the denominator is bounded away from zero. Applying the law of large numbers to the empirical average over prompts and rollouts gives
\(
\widehat L_{\max,P,M}
=
\frac{1}{PM}
\sum_{p=1}^P\sum_{i=1}^M
\hat r_{\max}^{(p,i)}
\xrightarrow{\mathbb P}
\mathbb E[r_{\max}]
=
L_{\max}.
\)
Since the exponential map is continuous,
\(
\widehat{\mathrm{BF}}_{\max,P,M}
=
\exp(\widehat L_{\max,P,M})
\xrightarrow{\mathbb P}
\exp(L_{\max})
=
\mathrm{BF}_{\max}.
\)
Finally, if $S_{\max}/N_{\max}\to S/N$ almost surely and the normalized entropy ratios are uniformly integrable, then convergence of expectations follows, yielding $L_{\max}\to L$. Another application of continuity of the exponential function gives $\mathrm{BF}_{\max}\to\mathrm{BF}$.
\end{proof}

\paragraph{Empirical support for the truncation condition.}
The consistency result above requires the truncation bias induced by $T_{\max}$ to vanish, or at least to be negligible in finite-sample estimation. While this is a population-level condition and cannot be proven from data alone, our diagnostics provide empirical support for it in our experimental setting. With $T_{\max}=4000$, the observed truncation rates are uniformly small across tasks, model families, and variants: most are below $1\%$, and the largest observed rate is $4.32\%$ (see \autoref{app:stopping_criteria}). In addition, the empirical length distributions concentrate well below the truncation threshold (see \autoref{fig:sequence_length_over}), and the number of active rollouts decays rapidly with token position (see \autoref{fig:entropy_rate}). Together with the bounded observed entropy-rate trajectories (see \autoref{fig:entropy_rate}), these results suggest that the normalized entropy ratios are well behaved and that the finite-$T_{\max}$ truncation bias is small in practice.

\paragraph{Estimating entropy rate.}
For each rollout $(p,i)$, Algorithm~\ref{alg:tmstar_est} returns the truncated stopping time
\(
N_{\max}^{(p,i)}=\min\{N^{(p,i)},T_{\max}\}
\)
and the accumulated truncated canopy entropy
\(
S_{\max}^{(p,i)}
=
\sum_{t=1}^{N_{\max}^{(p,i)}}
\widetilde H(G_t^{(p,i)}\mid x_p).
\)
We estimate the rollout-level entropy rate by
\(
\widehat r_{\max}^{(p,i)}
:=
\frac{S_{\max}^{(p,i)}}{N_{\max}^{(p,i)}}.
\)
Thus, for each prompt $x_p$, the prompt-specific length--entropy relationship is estimated from the rollout-level pairs
\(
\left\{
\left(
N_{\max}^{(p,i)},
\widehat r_{\max}^{(p,i)}
\right)
\right\}_{i=1}^M .
\)

\paragraph{Length--entropy correlation.} \autoref{def:correlation} introduced the length--entropy correlation
$\rho(N,r_N)$ at a conceptual level. We now provide its statistically rigorous
formulation together with the corresponding aggregation procedure.
\begin{definition}[length--entropy correlation]\label{def:correlation_appendix}
For a fixed prompt \(X=x\), define the prompt-conditional Pearson correlation
$\rho_X:=\rho(N,r_N\mid X=x)=\frac{\operatorname{Cov}(N,r_N\mid X=x)}{\sqrt{\operatorname{Var}(N\mid X=x)\operatorname{Var}(r_N\mid X=x)}}.$ Since the prompt-conditional correlation $\rho_X$ is bounded and nonlinear, directly averaging correlations across prompts may introduce bias. Therefore, we aggregate the prompt-conditional Pearson correlations on the Fisher-$z$\footnote{This transformation maps correlation coefficients from $[-1,1]$ to $\mathbb{R}$ and stabilizes their variance, making them approximately normally distributed with variance $1/(n_p - 3)$. $n_p$ is the number of rollouts associated with each prompt $p$. Directly averaging correlations would lead to biased estimates due to the nonlinear and bounded nature of correlation coefficients. See \autoref{alg:pc_corr_est} for details.
} 
scale to obtain a prompt-averaged correlation measure. Then the (prompt-averaged) length--entropy correlation 
$\rho(N, r_N):=\tanh\left(\frac{\mathbb{E}_{X\sim p_X}\left[w_X\operatorname{atanh}(\rho_X)\right]}{\mathbb{E}_{X\sim p_X}[w_X]}\right),$ where \(w_X\ge 0\) is a prompt-specific reliability weight.
\end{definition}

\paragraph{Rank-based length--entropy correlations.}
In addition to Pearson correlation, which measures linear association, we also report two rank-based measures. Spearman correlation \citep{spearman1904proof} is Pearson correlation applied to the ranked variables and therefore captures monotone dependence. Kendall's $\tau_b$ \citep{kendall1938new} measures ordinal association by comparing concordant and discordant rollout pairs, with tie adjustment. Since Kendall's $\tau_b$ is defined through pairwise comparisons, we aggregate it using pair-count weights rather than Fisher-$z$ averaging. {We estimate and report rank-based length--entropy correlations to assess the robustness of our findings. As shown in \autoref{tab:er_len_rho_tau_change}, fine-tuning increases both Spearman's $\rho$ and Kendall's $\tau$ in all twelve model--task settings, including Llama-3.1-8B on sentence completion, where the Pearson change is slightly negative but not statistically significant. The predominantly positive base--instruct shift across three different correlation measures suggests that the observed trend is not an artifact of any particular notion of correlation.}

\begin{table*}[t]
\centering
\small
\setlength{\tabcolsep}{3.0pt}
\renewcommand{\arraystretch}{1.12}
\caption{
Base--instruct comparison of length--entropy correlation measured by Spearman's $\rho$ and Kendall's $\tau$. Each entry reports estimate $\pm$ standard error, and absolute change is computed as instruct minus base.
All numbers are rounded to two significant digits. Warm colors indicate increases, while cold colors indicate decreases; darker colors correspond to larger magnitude changes.
}
\label{tab:er_len_rho_tau_change}
\resizebox{\textwidth}{!}{%
\begin{tabular}{llcccccc}
\toprule
\multirow{2}{*}{Task} & \multirow{2}{*}{Type}
& \multicolumn{3}{c}{Spearman's $\rho$ correlation}
& \multicolumn{3}{c}{Kendall's $\tau$ correlation} \\
\cmidrule(lr){3-5}
\cmidrule(lr){6-8}
& 
& Qwen3-8B & Llama-3.1-8B & Gemma-3-12B
& Qwen3-8B & Llama-3.1-8B & Gemma-3-12B \\
\midrule

\multirow{3}{*}{MATH}
& Base
& -0.19 $\pm$ 0.023
& -0.56 $\pm$ 0.031
& -0.27 $\pm$ 0.023
& -0.12 $\pm$ 0.015
& -0.36 $\pm$ 0.024
& -0.17 $\pm$ 0.015 \\

& Instruct
& 0.039 $\pm$ 0.032
& 0.15 $\pm$ 0.031
& 0.12 $\pm$ 0.032
& 0.026 $\pm$ 0.020
& 0.094 $\pm$ 0.021
& 0.070 $\pm$ 0.020 \\

& Abs. Change
& \cellcolor{warmdark!23!warmlight} 0.23 $\pm$ 0.037
& \cellcolor{warmdark!70!warmlight} 0.70 $\pm$ 0.048
& \cellcolor{warmdark!39!warmlight} 0.39 $\pm$ 0.043
& \cellcolor{warmdark!15!warmlight} 0.15 $\pm$ 0.024
& \cellcolor{warmdark!45!warmlight} 0.45 $\pm$ 0.034
& \cellcolor{warmdark!24!warmlight} 0.24 $\pm$ 0.027 \\

\midrule

\multirow{3}{*}{CODING}
& Base
& -0.28 $\pm$ 0.018
& -0.66 $\pm$ 0.014
& -0.40 $\pm$ 0.011
& -0.18 $\pm$ 0.013
& -0.47 $\pm$ 0.012
& -0.26 $\pm$ 0.0079 \\

& Instruct
& 0.49 $\pm$ 0.030
& 0.25 $\pm$ 0.016
& 0.38 $\pm$ 0.032
& 0.32 $\pm$ 0.021
& 0.17 $\pm$ 0.011
& 0.25 $\pm$ 0.022 \\

& Abs. Change
& \cellcolor{warmdark!76!warmlight} 0.76 $\pm$ 0.035
& \cellcolor{warmdark!91!warmlight} 0.91 $\pm$ 0.019
& \cellcolor{warmdark!78!warmlight} 0.78 $\pm$ 0.032
& \cellcolor{warmdark!50!warmlight} 0.50 $\pm$ 0.024
& \cellcolor{warmdark!64!warmlight} 0.64 $\pm$ 0.015
& \cellcolor{warmdark!51!warmlight} 0.51 $\pm$ 0.022 \\

\midrule

\multirow{3}{*}{SENTENCE COMPLETION}
& Base
& -0.0049 $\pm$ 0.025
& -0.39 $\pm$ 0.019
& -0.32 $\pm$ 0.016
& 0.0026 $\pm$ 0.017
& -0.26 $\pm$ 0.013
& -0.22 $\pm$ 0.011 \\

& Instruct
& 0.52 $\pm$ 0.024
& -0.30 $\pm$ 0.024
& 0.14 $\pm$ 0.034
& 0.36 $\pm$ 0.018
& -0.20 $\pm$ 0.016
& 0.092 $\pm$ 0.021 \\

& Abs. Change
& \cellcolor{warmdark!53!warmlight} 0.53 $\pm$ 0.033
& \cellcolor{warmdark!9!warmlight} 0.091 $\pm$ 0.029
& \cellcolor{warmdark!46!warmlight} 0.46 $\pm$ 0.038
& \cellcolor{warmdark!35!warmlight} 0.35 $\pm$ 0.023
& \cellcolor{warmdark!6!warmlight} 0.063 $\pm$ 0.020
& \cellcolor{warmdark!31!warmlight} 0.31 $\pm$ 0.024 \\

\midrule

\multirow{3}{*}{STORY}
& Base
& -0.36 $\pm$ 0.015
& -0.32 $\pm$ 0.015
& -0.30 $\pm$ 0.014
& -0.25 $\pm$ 0.011
& -0.21 $\pm$ 0.011
& -0.20 $\pm$ 0.0099 \\

& Instruct
& 0.16 $\pm$ 0.021
& -0.13 $\pm$ 0.016
& -0.018 $\pm$ 0.025
& 0.10 $\pm$ 0.013
& -0.089 $\pm$ 0.010
& -0.013 $\pm$ 0.016 \\

& Abs. Change
& \cellcolor{warmdark!52!warmlight} 0.52 $\pm$ 0.025
& \cellcolor{warmdark!18!warmlight} 0.18 $\pm$ 0.025
& \cellcolor{warmdark!28!warmlight} 0.28 $\pm$ 0.030
& \cellcolor{warmdark!35!warmlight} 0.35 $\pm$ 0.017
& \cellcolor{warmdark!12!warmlight} 0.12 $\pm$ 0.017
& \cellcolor{warmdark!18!warmlight} 0.18 $\pm$ 0.020 \\

\bottomrule
\end{tabular}%
}
\end{table*}

\begin{algorithm}[ht]
\caption{Estimating prompt-averaged length--entropy correlations}
\label{alg:pc_corr_est}
\KwIn{
Prompt distribution $p_X$; number of prompts $P$; rollout counts $\{n_p\}_{p=1}^P$; maximum length $T_{\max}$; sampling policy $\pi$; clipping constant $\epsilon>0$.
}

\For{$p=1$ \KwTo $P$}{
    Sample prompt $x_p\sim p_X$;\\

    \For{$i=1$ \KwTo $n_p$}{
        Generate one rollout using Algorithm~\ref{alg:tmstar_est};\\
        Record
        \(
        N_{\max}^{(p,i)}
        =
        \min\{N^{(p,i)},T_{\max}\},
        \;
        S_{\max}^{(p,i)}
        =
        \sum_{t=1}^{N_{\max}^{(p,i)}}
        \widetilde H(G_t^{(p,i)}\mid x_p).
        \)
        Compute the rollout-level entropy rate
        \(
        r_{\max}^{(p,i)}
        =
        \frac{S_{\max}^{(p,i)}}{N_{\max}^{(p,i)}}.
        \)
    }

    Compute the prompt-specific Pearson correlation
    \(
    \widehat\rho^{\mathrm{P}}_p
    =
    \operatorname{Corr}
    \left(
    \{N_{\max}^{(p,i)}\}_{i=1}^{n_p},
    \{r_{\max}^{(p,i)}\}_{i=1}^{n_p}
    \right).
    \)

    Compute the prompt-specific Spearman correlation
    \(
    \widehat\rho^{\mathrm{S}}_p
    =
    \operatorname{Corr}
    \left(
    \{\operatorname{rank}(N_{\max}^{(p,i)})\}_{i=1}^{n_p},
    \{\operatorname{rank}(r_{\max}^{(p,i)})\}_{i=1}^{n_p}
    \right),
    \)
    using mid-ranks for ties.\\

    Compute the prompt-specific Kendall's $\tau_b$
    \(
    \widehat\tau^{\mathrm{K}}_p
    =
    \operatorname{KendallTauB}
    \left(
    \{N_{\max}^{(p,i)}\}_{i=1}^{n_p},
    \{r_{\max}^{(p,i)}\}_{i=1}^{n_p}
    \right).
    \)

    Compute the prompt-specific covariance
    \(
    \widehat{\operatorname{Cov}}_p
    =
    \widehat{\operatorname{Cov}}
    \left(
    \{N_{\max}^{(p,i)}\}_{i=1}^{n_p},
    \{r_{\max}^{(p,i)}\}_{i=1}^{n_p}
    \right).
    \)
}

Aggregate Pearson and Spearman using Fisher-$z$ weighted averages:
\(
w^{\mathrm{F}}_p
=
\max\{n_p-3,0\},
\)
\(
\widehat\rho^{\mathrm{P}}_{\mathrm{pc}}
=
\tanh\left(
\frac{
\sum_{p=1}^P
w^{\mathrm{F}}_p
\operatorname{atanh}
\left(
\operatorname{clip}
(\widehat\rho^{\mathrm{P}}_p,-1+\epsilon,1-\epsilon)
\right)
}{
\sum_{p=1}^P w^{\mathrm{F}}_p
}
\right), 
\) and 
\(
\widehat\rho^{\mathrm{S}}_{\mathrm{pc}}
=
\tanh\left(
\frac{
\sum_{p=1}^P
w^{\mathrm{F}}_p
\operatorname{atanh}
\left(
\operatorname{clip}
(\widehat\rho^{\mathrm{S}}_p,-1+\epsilon,1-\epsilon)
\right)
}{
\sum_{p=1}^P w^{\mathrm{F}}_p
}
\right).
\)

Aggregate Kendall's $\tau_b$ using pair-count weights:\(
w^{\mathrm{K}}_p
=
\max\left\{
\frac{n_p(n_p-1)}{2},
0
\right\},
\)\(
\widehat\tau^{\mathrm{K}}_{\mathrm{pc}}
=
\frac{
\sum_{p=1}^P
w^{\mathrm{K}}_p
\widehat\tau^{\mathrm{K}}_p
}{
\sum_{p=1}^P
w^{\mathrm{K}}_p
}.
\)

Aggregate covariance using degrees-of-freedom weights: \(
w^{\mathrm{Cov}}_p
=
\max\{n_p-1,0\},
\) and \(
\widehat{\operatorname{Cov}}_{\mathrm{pc}}
=
\frac{
\sum_{p=1}^P
w^{\mathrm{Cov}}_p
\widehat{\operatorname{Cov}}_p
}{
\sum_{p=1}^P
w^{\mathrm{Cov}}_p
}.
\)

\Return{
\(
\widehat\rho^{\mathrm{P}}_{\mathrm{pc}},
\;
\widehat\rho^{\mathrm{S}}_{\mathrm{pc}},
\;
\widehat\tau^{\mathrm{K}}_{\mathrm{pc}},
\;
\widehat{\operatorname{Cov}}_{\mathrm{pc}}.
\)
}
\end{algorithm}

\begin{proposition}[Consistency of prompt-averaged length--entropy correlation estimators]
\label{prop:pc_corr_consistency}
For each prompt $X=x$, define
\(
r_{\max}(x)
:=
\frac{S_{\max}(x)}{N_{\max}(x)},
\;
S_{\max}(x)
=
\sum_{t=1}^{N_{\max}(x)}
\widetilde H(G_t\mid x),
\;
N_{\max}(x)
=
\min\{N(x),T_{\max}\}
\). Note that in this case, $S_{\max}(x)$ and $N_{\max}(x)$ are random variables which depend on the generated history $G_t$.
Let
\(
\rho^{\mathrm P}_{\max}(x)
:=
\operatorname{Corr}(N_{\max}(x),r_{\max}(x)\mid X=x)
\)
denote the Pearson correlation, let
\(
\rho^{\mathrm S}_{\max}(x)
:=
\operatorname{Corr}
\left(
F_{N_{\max}\mid x}(N_{\max}),
F_{r_{\max}\mid x}(r_{\max})
\mid X=x
\right)
\)
denote the population Spearman rank correlation, and let
\(
\tau^{\mathrm K}_{\max}(x)
\)
denote Kendall's $\tau_b$ between $N_{\max}$ and $r_{\max}$ conditional on $X=x$.
For sampled prompts $x_1,\ldots,x_P$, suppose that for each prompt $x_p$ we observe
$n_p$ conditionally i.i.d.\ rollouts and compute the prompt-specific estimates
\(
\widehat \rho^{\mathrm P}_p,\;
\widehat \rho^{\mathrm S}_p,\;
\widehat \tau^{\mathrm K}_p,\;
\widehat{\operatorname{Cov}}_p
\)
as in Algorithm~\ref{alg:pc_corr_est}. 
Assume that $N_{\max}\ge 1$ almost surely. For Pearson and covariance, assume that the conditional fourth moments of $N_{\max}$ and $r_{\max}$ are uniformly bounded and that
\(
0<c\le
\operatorname{Var}(N_{\max}\mid X=x),
\;
0<c\le
\operatorname{Var}(r_{\max}\mid X=x)
\)
uniformly over $x$. For Spearman and Kendall, assume the corresponding rank-correlation functionals are continuous at the conditional law of $(N_{\max},r_{\max})\mid X=x$, uniformly over $x$; in particular, ties are either absent or handled by the mid-rank and $\tau_b$ conventions used in Algorithm~\ref{alg:pc_corr_est}. Finally, assume
\(
|\rho^{\mathrm P}_{\max}(x)|\le 1-\epsilon_0,
\;
|\rho^{\mathrm S}_{\max}(x)|\le 1-\epsilon_0
\)
for some $\epsilon_0>0$, and that $\min_p n_p\to\infty$ as $P\to\infty$.
Then, as $P\to\infty$ and $\min_p n_p\to\infty$,
\(
\widehat\rho^{\mathrm P}_{\mathrm{pc},\max}
\xrightarrow{\mathbb P}
\rho^{\mathrm P}_{\mathrm{pc},\max},
\;
\widehat\rho^{\mathrm S}_{\mathrm{pc},\max}
\xrightarrow{\mathbb P}
\rho^{\mathrm S}_{\mathrm{pc},\max},
\)
where
\(
\rho^{\mathrm P}_{\mathrm{pc},\max}
=
\tanh\left(
\frac{
\mathbb E_X[
w^{\mathrm F}(X)\operatorname{atanh}\{\rho^{\mathrm P}_{\max}(X)\}]
}{
\mathbb E_X[w^{\mathrm F}(X)]
}
\right),
\)
\(
\rho^{\mathrm S}_{\mathrm{pc},\max}
=
\tanh\left(
\frac{
\mathbb E_X[
w^{\mathrm F}(X)\operatorname{atanh}\{\rho^{\mathrm S}_{\max}(X)\}]
}{
\mathbb E_X[w^{\mathrm F}(X)]
}
\right).
\)
Also,
\(
\widehat\tau^{\mathrm K}_{\mathrm{pc},\max}
\xrightarrow{\mathbb P}
\tau^{\mathrm K}_{\mathrm{pc},\max}
:=
\frac{
\mathbb E_X[w^{\mathrm K}(X)\tau^{\mathrm K}_{\max}(X)]
}{
\mathbb E_X[w^{\mathrm K}(X)]
},
\)
and
\(
\widehat{\operatorname{Cov}}_{\mathrm{pc},\max}
\xrightarrow{\mathbb P}
\operatorname{Cov}_{\mathrm{pc},\max}
:=
\frac{
\mathbb E_X[w^{\mathrm{Cov}}(X)
\operatorname{Cov}(N_{\max},r_{\max}\mid X)]
}{
\mathbb E_X[w^{\mathrm{Cov}}(X)]
}.
\)
If additionally $T_{\max}\to\infty$ and the corresponding truncated targets converge to their untruncated counterparts, then the same estimators are consistent for the corresponding prompt-averaged untruncated Pearson, Spearman, Kendall, and covariance targets.
\end{proposition}

\begin{proof}
We prove the result in two steps: first within each prompt, and then across prompts.
For a fixed prompt $x_p$, the rollout-level pairs
\(
\left(
N_{\max}^{(p,i)},
r_{\max}^{(p,i)}
\right),
\; i=1,\ldots,n_p,
\)
are conditionally i.i.d. By Algorithm~\ref{alg:pc_corr_est}, both coordinates are observed from the same rollout, with
\(
r_{\max}^{(p,i)}
=
\frac{S_{\max}^{(p,i)}}{N_{\max}^{(p,i)}}.
\)
The moment and nondegeneracy assumptions imply consistency of the sample covariance and variances. Hence,
\(
\widehat\rho^{\mathrm P}_p
\xrightarrow{\mathbb P}
\rho^{\mathrm P}_{\max}(x_p),
\;
\widehat{\operatorname{Cov}}_p
\xrightarrow{\mathbb P}
\operatorname{Cov}(N_{\max},r_{\max}\mid x_p).
\)
Similarly, by consistency of empirical ranks and the assumed continuity of the rank-correlation functionals,
\(
\widehat\rho^{\mathrm S}_p
\xrightarrow{\mathbb P}
\rho^{\mathrm S}_{\max}(x_p),
\;
\widehat\tau^{\mathrm K}_p
\xrightarrow{\mathbb P}
\tau^{\mathrm K}_{\max}(x_p).
\)
Since $|\rho^{\mathrm P}_{\max}(x)|$ and $|\rho^{\mathrm S}_{\max}(x)|$ are bounded away from one, and the clipping constant is chosen smaller than this margin, the Fisher transform is continuous in a neighborhood of the limiting values. Therefore,
\(
\operatorname{atanh}
\left(
\operatorname{clip}
(\widehat\rho^{\mathrm P}_p,-1+\epsilon,1-\epsilon)
\right)
\xrightarrow{\mathbb P}
\operatorname{atanh}\{\rho^{\mathrm P}_{\max}(x_p)\},
\)
and
\(
\operatorname{atanh}
\left(
\operatorname{clip}
(\widehat\rho^{\mathrm S}_p,-1+\epsilon,1-\epsilon)
\right)
\xrightarrow{\mathbb P}
\operatorname{atanh}\{\rho^{\mathrm S}_{\max}(x_p)\}.
\)
Now consider the aggregation across prompts. For Pearson, define
\(
Z^{\mathrm P}_p
=
\operatorname{atanh}\{\rho^{\mathrm P}_{\max}(x_p)\}.
\)
By the previous step and the uniform assumptions,
\(
\frac{
\sum_{p=1}^P w^{\mathrm F}_p
\left[
\operatorname{atanh}
\left(
\operatorname{clip}
(\widehat\rho^{\mathrm P}_p,-1+\epsilon,1-\epsilon)
\right)
-
Z^{\mathrm P}_p
\right]
}{
\sum_{p=1}^P w^{\mathrm F}_p
}
\xrightarrow{\mathbb P}
0.
\)
Since prompts are sampled i.i.d., the weighted law of large numbers gives
\(
\frac{\sum_{p=1}^P w^{\mathrm F}_p Z^{\mathrm P}_p}
{\sum_{p=1}^P w^{\mathrm F}_p}
\xrightarrow{\mathbb P}
\frac{
\mathbb E_X[
w^{\mathrm F}(X)\operatorname{atanh}\{\rho^{\mathrm P}_{\max}(X)\}]
}{
\mathbb E_X[w^{\mathrm F}(X)]
}.
\)
Applying the continuous mapping theorem with $\tanh(\cdot)$ yields
\(
\widehat\rho^{\mathrm P}_{\mathrm{pc},\max}
\xrightarrow{\mathbb P}
\rho^{\mathrm P}_{\mathrm{pc},\max}.
\)
The Spearman result is identical after replacing $\rho^{\mathrm P}_{\max}$ by $\rho^{\mathrm S}_{\max}$.
For Kendall's $\tau_b$, Algorithm~\ref{alg:pc_corr_est} aggregates directly on the $\tau$ scale using pair-count weights
\(
w^{\mathrm K}_p=\frac{n_p(n_p-1)}{2}.
\)
Thus, by the prompt-level consistency of $\widehat\tau^{\mathrm K}_p$ and the weighted law of large numbers,
\(
\widehat\tau^{\mathrm K}_{\mathrm{pc},\max}
=
\frac{\sum_{p=1}^P w^{\mathrm K}_p\widehat\tau^{\mathrm K}_p}
{\sum_{p=1}^P w^{\mathrm K}_p}
\xrightarrow{\mathbb P}
\frac{
\mathbb E_X[w^{\mathrm K}(X)\tau^{\mathrm K}_{\max}(X)]
}{
\mathbb E_X[w^{\mathrm K}(X)]
}.
\)
The covariance result follows in the same way, using weights
$w^{\mathrm{Cov}}_p=(n_p-1)_+$ and the prompt-level consistency of
$\widehat{\operatorname{Cov}}_p$.
Finally, if $T_{\max}\to\infty$ and the truncated conditional covariance, variance, and rank-association functionals converge to their untruncated counterparts, then the truncated prompt-averaged targets converge to the corresponding untruncated targets. Combining this approximation step with the fixed-$T_{\max}$ consistency above completes the proof.
\end{proof}

\begin{table}[!htbp]
\centering
\small
\caption{Empirical Truncation Rates. We report the proportion of Monte Carlo rollouts that reached the maximum generation budget ($T_{\max} = 4000$ tokens) without sampling an $\texttt{EOS}$ token. Rates are calculated over $P=100$ prompts with $M=100$ rollouts per prompt.}
\label{tab:truncation_rates}
\begin{tabular}{lcccccccc}
\toprule

\multirow{2}{*}{Model}
& \multicolumn{2}{c}{MATH}
& \multicolumn{2}{c}{STORY}
& \multicolumn{2}{c}{CODING}
& \multicolumn{2}{c}{COMPLETION} \\

\cmidrule(lr){2-3}
\cmidrule(lr){4-5}
\cmidrule(lr){6-7}
\cmidrule(lr){8-9}

& Base & Instruct
& Base & Instruct
& Base & Instruct
& Base & Instruct \\

\midrule

Qwen3-8B
& 0.0073 & 0.0000
& 0.0005 & 0.0000
& 0.0008 & 0.0004
& 0.0313 & 0.0000 \\

Llama-3.1-8B
& 0.0105 & 0.0012
& 0.0325 & 0.0000
& 0.0095 & 0.0000
& 0.0174 & 0.0000 \\

Gemma-3-12B
& 0.0036 & 0.0000
& 0.0066 & 0.0000
& 0.0432 & 0.0000
& 0.0119 & 0.0000 \\

\bottomrule
\end{tabular}%
\end{table}

\begin{figure}[!htbp]
    \centering
    \includegraphics[width=1\linewidth]{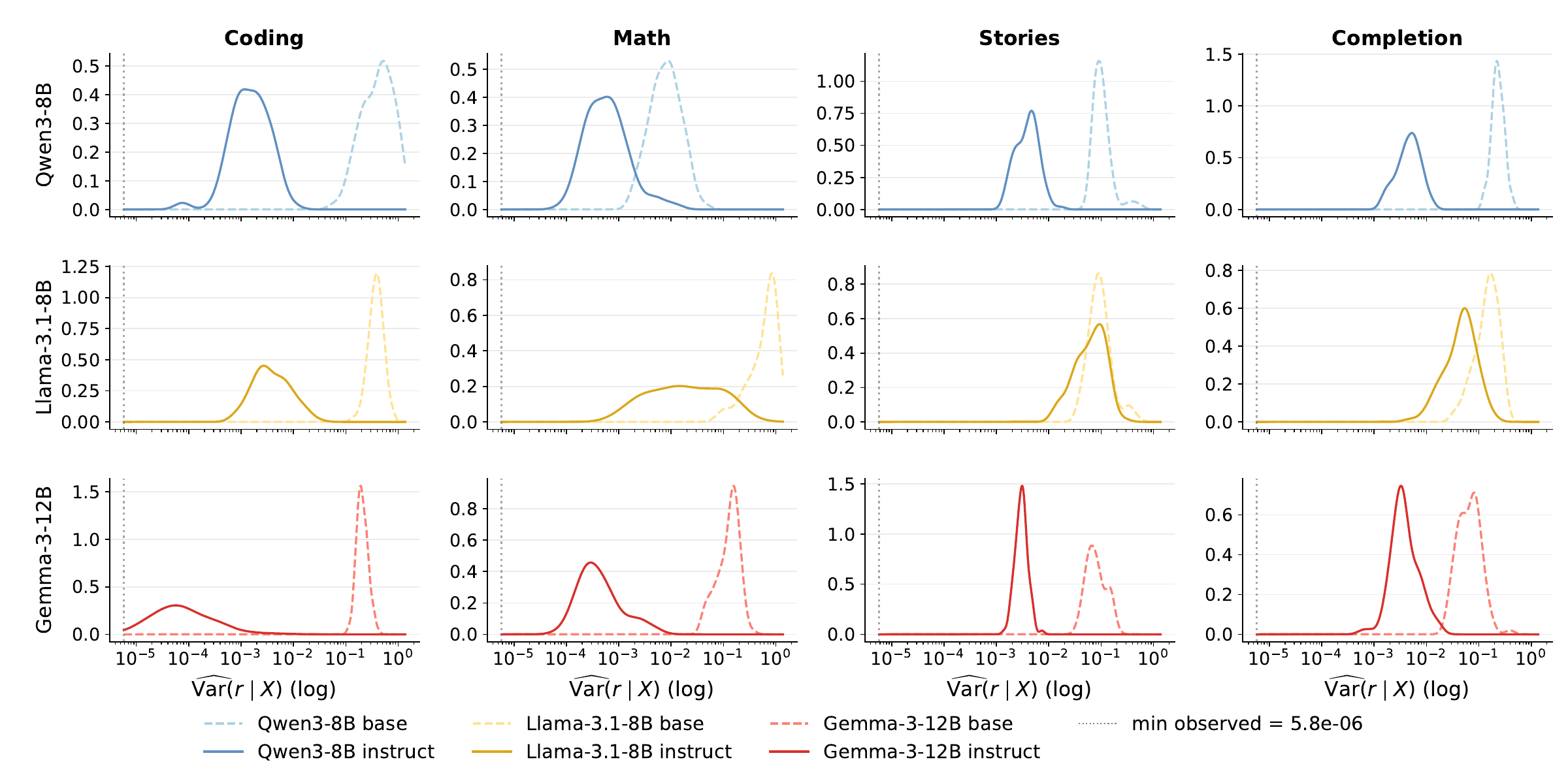}
    \caption{Gaussian KDEs of $\log \widehat{\mathrm{Var}}(r\mid x_p)$ over $P{=}100$ prompts. The dotted vertical marks $\min_p \widehat{\mathrm{Var}}(r\mid x_p)\!\approx\!6\!\times\!10^{-6}$. All densities sit well to the right of zero, providing empirical support for the bounded-away-from-zero assumption.}
    \label{fig:entropy_rate_var}
\vspace{1em}
    \includegraphics[width=1\linewidth]{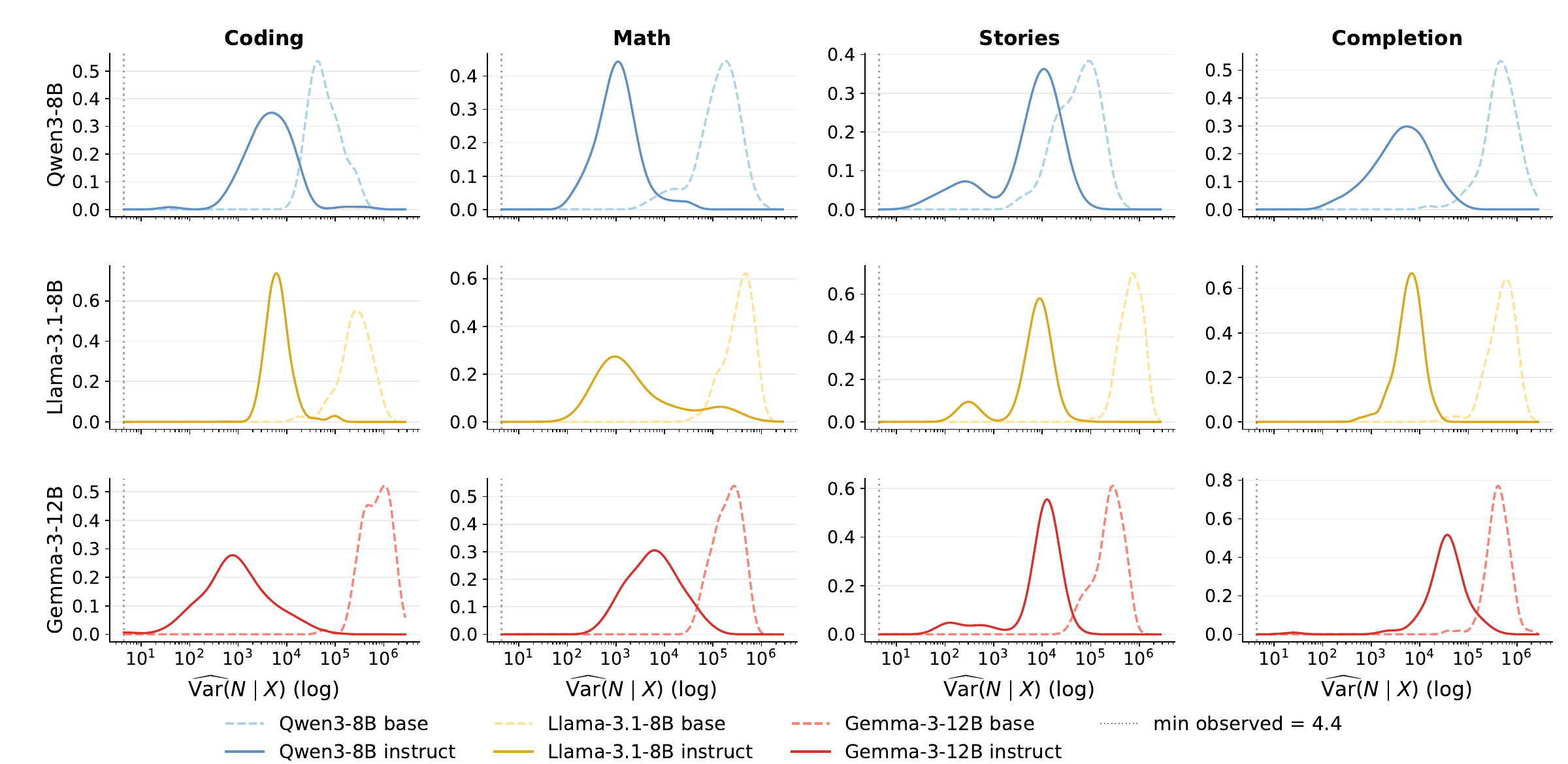}
    \caption{Gaussian KDEs of $\log \widehat{\mathrm{Var}}(N\mid x_p)$ over $P{=}100$ prompts. The dotted vertical line marks
    $\min_p \widehat{\mathrm{Var}}(N\mid x_p)\!\approx\!4.4$, the smallest per-prompt variance observed across all model-dataset combinations. All densities concentrate at $\widehat{\mathrm{Var}}(N\mid x_p)\!\gg\!0$, empirically supporting the bounded-away-from-zero assumption. Instruct variants consistently shift to smaller values, indicating that fine-tuning produces tighter length distributions.}
    \label{fig:length_var}
\end{figure}

\paragraph{Empirical support for the consistency assumptions.}
The assumptions underlying \autoref{prop:pc_corr_consistency} are supported by several empirical diagnostics observed in our experiments. First, the truncation rates under $T_{\max}=4000$ are uniformly small across model families, tasks, and variants, with most rates below $1\%$ and the maximum observed rate equal to $4.32\%$ (see \autoref{tab:truncation_rates}). This suggests that the truncated quantities $(N_{\max},r_{\max})$ closely approximate their untruncated counterparts $(N,r_N)$ in practice. Second, the empirical sequence-length distributions exhibit rapidly decaying tails, and the number of active rollouts decreases quickly as generation proceeds, providing evidence that the stopping-time distribution is sufficiently light-tailed and that truncation bias is negligible for the chosen $T_{\max}$ (see \autoref{fig:sequence_length_over}). 

Third, the rollout-level entropy rates remain bounded and stable throughout generation trajectories, supporting the bounded-moment and continuity assumptions required for consistency of the covariance and correlation estimators (see \autoref{fig:entropy_rate}). Finally, substantial variability is consistently observed in both generation length and entropy rate across prompts and rollouts, indicating that the conditional variances of $N_{\max}$ and $r_{\max}$ are bounded away from zero with high probability (see \autoref{fig:entropy_rate_var} and \autoref{fig:length_var}). Together, these empirical observations suggest that the regularity conditions required for the consistency of the prompt-averaged Pearson, Spearman, Kendall, and covariance estimators are well approximated in our experimental regime.

\section{Discussions}

\subsection{How GenPPL and BF capture different notions of diversity: an essay analogy case study}\label{app:example}
We illustrate the difference between GenPPL and BF through an essay--word analogy, where trajectories correspond to essays and tokens correspond to individual words. For a fixed prompt $x$, suppose a model generates two types of essays. With probability $1/2$, it produces a short essay of length $N=2$, where each word is selected uniformly from $100$ plausible words. With probability $1/2$, it produces a long essay of length $N=10$, where each word is selected uniformly from only $2$ plausible words. For the short essays,
\(
H(Y_{1:2}\mid N=2,x)
=
2\log 100,
r_2=\log 100.
\) 
For the long essays,
\(
H(Y_{1:10}\mid N=10,x)
=
10\log 2,
r_{10}=\log 2.
\)

BF treats each \emph{trajectory} equally:
\(
\log \mathrm{BF}
=
\mathbb E_N[r_N]
=
\sum_n P(N=n)\frac{H(Y_{1:n}\mid N=n,x)}{n}
=
\frac12 \log 100
+
\frac12 \log 2,
\)
which yields
\(
\mathrm{BF}\approx 14.
\)
This means that a typical generated trajectory behaves as if each token has roughly $14$ plausible continuations. Because BF weights trajectories equally, short but highly diverse generations contribute substantially even if they contain relatively few tokens overall.
In contrast, GenPPL treats each \emph{token} equally. \(\log \mathrm{GenPPL}
=
\frac{H(N, Y_{1:N}\mid x)}{\mathbb{E}[N]}
=
\frac{H(N\mid x)+\sum_n P(N=n)\,H(Y_{1:n}\mid N=n,x)}
     {\sum_n P(N=n)\,n}
=
\frac{6\log 2 + \log 100}{6}.\) Since most generated tokens come from the long low-diversity essays, we obtain
\(
\mathrm{GenPPL}\approx 4.
\) 
Thus, a token on average has only about $4$ plausible continuations.

The distinction arises from the averaging mechanism. BF averages uncertainty at the trajectory level, while GenPPL averages at the token level and therefore assigns greater weight to longer generations. Consequently, BF reflects the diversity of a \emph{typical trajectory}, whereas GenPPL reflects the uncertainty of a \emph{typical token}. These two measures therefore provide complementary views of generation space that cannot be captured by a single metric alone. 

\subsection{When to use BF vs.\ GenPPL across different tasks.}\label{app:bf_genppl_usage}
We refer readers to \autoref{tab:genppl_vs_bf} for a concise practical comparison of when to use BF versus GenPPL across different tasks.

\begin{table}[h]
\centering
\caption{When to use BF vs.\ GenPPL across different tasks.}
\label{tab:genppl_vs_bf}

\renewcommand{\arraystretch}{1.3}

\begin{tabular}{
    p{2.3cm}
    p{2.5cm}
    p{1.6cm}
    p{5.7cm}
}
\toprule
\textbf{Task} & \textbf{Preferred Metric} & \textbf{Level} & \textbf{Reason} \\
\midrule

MMLU (QA)
& GenPPL
& Token
& Requires selecting a single correct answer. Performance depends on calibrated token probabilities. Diversity across full trajectories is unnecessary. \\

\makecell[tl]{Story\\Generation}
& BF
& Trajectory
& Outputs are evaluated as complete units. Diversity across stories reflects creativity and coverage of different narrative possibilities. \\

Math
& \makecell[tl]{GenPPL \\[3pt] BF}
& \makecell[tl]{Token \\[3pt] Trajectory}
& GenPPL is primary because correctness depends on precise step-by-step reasoning. BF is secondary and useful for analyzing diversity across alternative solution strategies. \\

Coding
& \makecell[tl]{GenPPL \\[3pt] BF}
& \makecell[tl]{Token \\[3pt] Trajectory}
& GenPPL is primary because code correctness requires token-level precision. BF is secondary and useful for comparing multiple valid implementations. \\

\makecell[tl]{Sentence\\Completion}
& GenPPL
& Token
& This is a next-token prediction task, so GenPPL aligns directly with likelihood and perplexity-based evaluation. \\

\bottomrule
\end{tabular}
\renewcommand{\arraystretch}{1}
\end{table}

\subsection{Comparison with perplexity.}\label{app:comparison_perplexity}
Perplexity \citep{jelinek1977perplexity} is a widely used evaluation metric in the LLM literature. To avoid confusion, we briefly contrast it with GenPPL and BF. Perplexity is defined on a fixed dataset $(y_1, \ldots, y_T)$, where $y_t \in \mathcal V$ is the observed token and $y_{<t}$ denotes its context. Given the model distribution $q(\cdot \mid y_{<t})$, it is defined as $\mathrm{PPL} = \exp\left(\frac{1}{T}\sum_{t=1}^T -\log q(y_t \mid y_{<t})\right),$
where $-\log q(y_t \mid y_{<t})$ measures how surprising the true token is under the model. Thus, perplexity evaluates predictive accuracy, with lower values indicating better fit to the data. In contrast, GenPPL and BF are defined with respect to the model’s generative distribution and measure generation capacity. Namely, the size and diversity of the space of possible outputs rather than accuracy on the observed data.

\section{Experiment details}\label{sec:experiment_details}
In this section, we provide additional experimental and implementation details underlying our empirical analysis. In particular, we discuss the stopping behavior of autoregressive generation and the resulting truncation bias, as well as the statistical modeling procedures used throughout the paper. We further present detailed outputs from the Beta mixed-effects regression analysis, including coefficient estimates, interaction effects, and dispersion modeling results, together with additional discussions on uncertainty allocation, semantic diversity, and length--entropy dynamics across model families and tasks.

\subsection{Stopping criteria and empirical truncation bias}\label{app:stopping_criteria}
To ensure the generative process remains consistent with the theoretical framework, we include the $\texttt{EOS}$ token in the sampling vocabulary $\mathcal V \cup \{\texttt{EOS}\}$. Crucially, during $\text{top}$-$k$ sampling, a rollout only terminates if $\texttt{EOS}$ is present within the top $k$ most probable tokens and is subsequently sampled.

While our theoretical framework allows for infinite generation length, practical estimation requires a finite truncation point $T_{\max}$. To ensure that our estimator $\widehat{\mathrm{CE}}^\star_{\max,P,M}$ remains a consistent proxy for the true Canopy Entropy $\mathrm{CE}^\star$, we set a generous token budget of $T_{\max} = 4000$. As shown in Table \ref{tab:truncation_rates}, the empirical truncation rate remains low across all model families and tasks. Notably, fine-tuned models exhibit a 0.00\% truncation rate in almost all settings, confirming that aligned models naturally terminate well within the provided budget. While truncation rates are higher for base models, they remain sufficiently low ($\leq 0.0432$). Overall, these results demonstrate that the truncation bias term $g(T_{\max})$ defined in \autoref{thm:unified_tmstar_prompt} is sufficiently small, ensuring that the observed differences in generation space are not artifacts of the maximum length constraint.

\subsection{Completion traps}\label{app:completion_traps}
Instead of answering the instruction, base models sometimes treat the prompt as an unfinished text prefix and continue it with a short completion fragment, or terminate immediately when the text appears complete. Here we present some empirical examples where base models fall into such ``completion traps''. Some examples are listed below.

Prompt:
\begin{quote}
Come up with an original and creative solution for the following real-world problem: Clara, a junior pre-med student, is working part-time and taking a 15 hour credit load at school \textit{[...]} Clara is not sure how to solve her problem.
\end{quote}

Qwen3-8B Base:
\begin{quote}
\ttfamily
"In order to entice donors, you would like to take the following approach:"
\end{quote}

Llama-3.1-8B Base:
\begin{quote}
\ttfamily
"How can she decide what steps she should take to solve her problem?"
\end{quote}

Gemma-3-12B Base:
\begin{quote}
\ttfamily
"There will be partial credit for good ideas that might just work (even if they're not perfect). Your solution will be judged on its creativity and how well it works."
\end{quote}

Coding prompts were especially vulnerable to premature termination caused by structural delimiters. Prompts ending with a closing code block delimiter (\verb|```|) frequently produced empty responses from base models, as they likely interpret the closing backticks as an end-of-document signal from the training distribution. To ensure valid evaluations, we manually stripped these trailing delimiters to force the model to generate the intended logic rather than treating the task as already finalized.

\subsection{Chat-style replies for sentence completion}\label{app:chat_replies}
Conversely, instruct models sometimes treat a sentence-completion prompt as a conversational query. Since the unfinished sentence is passed as a user message through the chat template, the models often comment on the scene, offer to elaborate, or ask what the user wants, rather than continuing the text.

\subsection{Bootstrap procedure}\label{app:bootstrap}
All standard errors and confidence intervals reported in \autoref{tab:global_uncertainty}, \autoref{tab:er_len_r_change}, and \autoref{tab:er_len_rho_tau_change} are estimated using a prompt-level cluster bootstrap with $B=2000$ replicates. In each replicate, we sample $P=100$ prompt indices uniformly with replacement from the original set of prompts. All $M=100$ rollouts associated with a sampled prompt are resampled jointly as a single cluster, preserving the within-prompt rollout structure.

Point estimates are always computed on the original, non-resampled data; bootstrap replicates are used only for uncertainty quantification. Reported standard errors correspond to the sample standard deviation of the bootstrap replicates with degrees of freedom = $n-1$.

\subsection{Beta mixed-effects model for semantic diversity}\label{app:beta_regression}
The analysis is based on 2400 observations, corresponding to \(4\) tasks, \(3\) model families, and \(2\) model variants (base and instruct), with \(100\) prompts evaluated for each configuration.

\begin{table}[h]
\centering
\caption{Full table of conditional mean model estimates for the Beta mixed-effects regression.}
\label{tab:beta_conditional_model}
\begin{tabular}{lrrr}
\toprule
Term & Estimate & Std. Error & $p$-value \\
\midrule
Intercept & -0.900 & 0.071 & $< 2\times 10^{-16}$ \\
\textbf{\boldmath $R$ ($\beta_1$)} & \textbf{0.688} & 0.033 & $< 2\times 10^{-16}$ \\
\textbf{\boldmath Instruct ($\beta_2$)} & \textbf{-0.325} & 0.087 & $1.87 \times 10^{-4}$ \\
$\mathrm{ns}(\mathrm{invN},4)_1$ & 0.169 & 0.040 & $2.04 \times 10^{-5}$ \\
$\mathrm{ns}(\mathrm{invN},4)_2$ & 0.448 & 0.236 & 0.058 \\
$\mathrm{ns}(\mathrm{invN},4)_3$ & 0.368 & 1.499 & 0.806 \\
$\mathrm{ns}(\mathrm{invN},4)_4$ & 0.011 & 3.156 & 0.997 \\
Coding & -0.035 & 0.036 & 0.329 \\
Math & 0.180 & 0.034 & $1.53 \times 10^{-7}$ \\
Stories & -0.252 & 0.028 & $< 2\times 10^{-16}$ \\
\textbf{\boldmath $R \times$ Instruct ($\beta_3$)} & \textbf{1.323} & 0.054 & $< 2\times 10^{-16}$ \\
Instruct $\times \mathrm{ns}(\mathrm{invN},4)_1$ & -0.099 & 0.077 & 0.199 \\
Instruct $\times \mathrm{ns}(\mathrm{invN},4)_2$ & 1.305 & 0.333 & $8.76 \times 10^{-5}$ \\
Instruct $\times \mathrm{ns}(\mathrm{invN},4)_3$ & 1.237 & 1.514 & 0.414 \\
Instruct $\times \mathrm{ns}(\mathrm{invN},4)_4$ & 1.059 & 3.159 & 0.737 \\
\bottomrule
\end{tabular}
\end{table}

\textbf{Conditional mean modeling.} Formally, for each task--prompt--model tuple $(t,p,m)$, we model
\(
D_{tpm}
\sim
\operatorname{Beta}
\!\left(
\mu_{tpm}\phi_{tpm},
(1-\mu_{tpm})\phi_{tpm}
\right),
\)
where \(\mu_{tpm}\in(0,1)\) denotes the conditional mean and \(\phi_{tpm}>0\) is the precision parameter. The conditional mean model is specified as
\(\operatorname{logit}(\mu_{tpm})=
\alpha+\beta_1 R_{tpm}
+\beta_2 \mathbf{1}\{m=\mathrm{FT}\}+
f_0\left(\mathrm{invN}_{tpm}\right)
+
\sum_k
\theta_k \mathbf{1}\{t=k\}
+
\beta_3
(
R_{tpm}
\cdot
\mathbf{1}\{m=\mathrm{FT}\}
)
+
f_1\left(\mathrm{invN}_{tpm}\right)
\mathbf{1}\{m=\mathrm{FT}\}
+
u_p
+
v_m
\)
where \(R_{tpm}\) is the entropy rate, \(\mathrm{invN}_{tpm}\) is the inverse sequence length, and \(f_0(\cdot)\), \(f_1(\cdot)\) are natural spline functions corresponding to nonlinear length effects and their interaction with instruction tuning. The coefficients \(\theta_k\) represent fixed task effects. To account for hierarchical dependence, we include random intercepts
\(
u_p \sim \mathcal N(0,\sigma_p^2),
v_m \sim \mathcal N(0,\sigma_m^2),
\)
for prompt identity and model family, respectively.

In addition to the finding presented in \autoref{sec:results}, we also observe strong nonlinear dependence on inverse sequence length. Several spline terms and their interactions with fine-tuning are significant, confirming that the relationship between diversity and length is highly nonlinear and differs between base and fine-tuned models. This reinforces that simple linear controls for length are insufficient.

Task-level effects are also pronounced. Relative to completion, mathematical reasoning exhibits significantly higher semantic diversity ($+0.180$, $p < 10^{-6}$), while story generation shows substantially lower embedding-based diversity ($-0.252$, $p < 2\times 10^{-16}$). This suggests that open-ended generation can remain semantically concentrated, despite its greater freedom of expression (see \autoref{tab:beta_conditional_model}).

Finally, random effects indicate that prompt-level variability remains larger than model-level variability, highlighting the strong influence of prompt semantics on generation diversity (see \autoref{tab:random_effects}).

{\setlength{\textfloatsep}{2pt plus 1pt minus 1pt}
\begin{table}[h]
    \centering
    \caption{Estimated random intercept variances for prompt identity and model family in the Beta mixed-effects regression model.}
    \label{tab:random_effects}
    \begin{tabular}{llrr}
    \toprule
    Random effect & Group & Variance & Std. Dev. \\
    \midrule
    Intercept & prompt\_uid & 0.0127 & 0.112 \\
    Intercept & model\_name & 0.00714 & 0.0845 \\
    \bottomrule
\end{tabular}
\end{table}}

\textbf{Dispersion modeling.} The dispersion model specifies how the precision parameter $\phi_{tpm}$ varies across model variants, tasks, sequence length, entropy rate, and model family. Specifically, we use the log-link specification
\(
\log(\phi_{tpm})
=
\gamma_0
+
\gamma_1 \mathbf{1}\{m=\mathrm{FT}\}
+
\sum_k \delta_k \mathbf{1}\{t=k\}
+
h\!\left(\mathrm{invN}_{tpm}\right)
+
\gamma_2 R_{tpm}
+
\sum_j \eta_j \mathbf{1}\{\mathrm{fam}(m)=j\}
+
\sum_k \zeta_k
\mathbf{1}\{m=\mathrm{FT}\}
\mathbf{1}\{t=k\}
+
\gamma_3
\left(
R_{tpm}
\cdot
\mathbf{1}\{m=\mathrm{FT}\}
\right),
\)
where $h(\cdot)$ is a natural spline function capturing nonlinear length-dependent dispersion effects. The task indicators $\delta_k$ allow the residual precision to vary across tasks, while the model-family indicators $\eta_j$ account for systematic differences in dispersion across model families. The interaction terms between model variant and task allow fine-tuning to affect dispersion differently across tasks, and the interaction between entropy rate and model variant tests whether the relationship between entropy rate and precision changes after instruction tuning.

Empirically, the dispersion model reveals substantial heterogeneity in residual precision. Instruction-tuned models have significantly higher baseline precision than base models ($\hat{\gamma}_1=1.315$, $p<0.001$), indicating lower conditional variability after accounting for the mean structure. Precision also differs strongly by task, with coding, math, and story generation all showing significantly higher precision relative to the reference task. The spline terms for inverse length indicate that dispersion varies nonlinearly with output length, while the positive main effect of entropy rate ($\hat{\gamma}_2=1.016$, $p<0.001$) suggests that higher entropy rate is associated with greater precision among base models. However, the negative interaction between instruction tuning and entropy rate ($\hat{\gamma}_3=-1.878$, $p<0.001$) shows that this relationship is substantially weakened or reversed for fine-tuned models. The significant model-variant-by-task interactions further indicate that fine-tuning changes residual variability in a task-dependent manner (see \autoref{tab:beta_dispersion_model}).

\begin{table}[h]
\centering
\caption{Dispersion model estimates for the Beta mixed-effects regression. The dispersion submodel uses a log link for the precision parameter $\phi_{tpm}$.}
\label{tab:beta_dispersion_model}
\begin{tabular}{lrrr}
\toprule
Term & Estimate & Std. Error & $p$-value \\
\midrule
Intercept & 2.658 & 0.226 & $< 2\times 10^{-16}$ \\
Instruct & 1.315 & 0.255 & $2.48 \times 10^{-7}$ \\
Coding & 2.602 & 0.213 & $< 2\times 10^{-16}$ \\
Math & 2.147 & 0.209 & $< 2\times 10^{-16}$ \\
Stories & 1.677 & 0.148 & $< 2\times 10^{-16}$ \\
$\mathrm{ns}(\mathrm{invN},3)_1$ & -3.464 & 0.562 & $7.20 \times 10^{-10}$ \\
$\mathrm{ns}(\mathrm{invN},3)_2$ & -2.097 & 0.904 & 0.020 \\
$\mathrm{ns}(\mathrm{invN},3)_3$ & 2.192 & 1.916 & 0.253 \\
$R$ & 1.016 & 0.126 & $5.97 \times 10^{-16}$ \\
Llama-3.1-8B & 0.909 & 0.113 & $9.16 \times 10^{-16}$ \\
Qwen3-8B & 0.622 & 0.114 & $5.42 \times 10^{-8}$ \\
Instruct $\times$ Coding & -2.119 & 0.288 & $1.88 \times 10^{-13}$ \\
Instruct $\times$ Math & -1.589 & 0.284 & $2.14 \times 10^{-8}$ \\
Instruct $\times$ Stories & -2.790 & 0.208 & $< 2\times 10^{-16}$ \\
Instruct $\times R$ & -1.878 & 0.227 & $< 2\times 10^{-16}$ \\
\bottomrule
\end{tabular}
\end{table}

\textbf{Residual diagnosis.} The figure (see \autoref{fig:qq_plot}) presents a DHARMa residual diagnostic QQ plot \citep{hartig2016dharma} for the fitted Beta mixed-effects regression model, computed from 10,000 simulated datasets, comparing the empirical residual distribution against the expected uniform distribution. The residuals closely follow the diagonal reference line, indicating good overall model calibration and no substantial systematic deviation from the assumed distribution. The associated diagnostic tests further support model adequacy: the Kolmogorov--Smirnov test ($p=0.473$) suggests no significant deviation from uniformity, the dispersion test ($p=0.110$) indicates no significant over- or under-dispersion, and the outlier test ($p=0.084$) shows no evidence of excessive outliers. We note, however, that while the KS and dispersion results remain stable across random seeds, the outlier test is noticeably more sensitive to the seed used for simulation. Overall, these diagnostics suggest that the fitted model provides an adequate characterization of the observed data and that the regression results are statistically reliable.

\begin{figure}[h]
    \centering
    \includegraphics[width=0.7\linewidth]{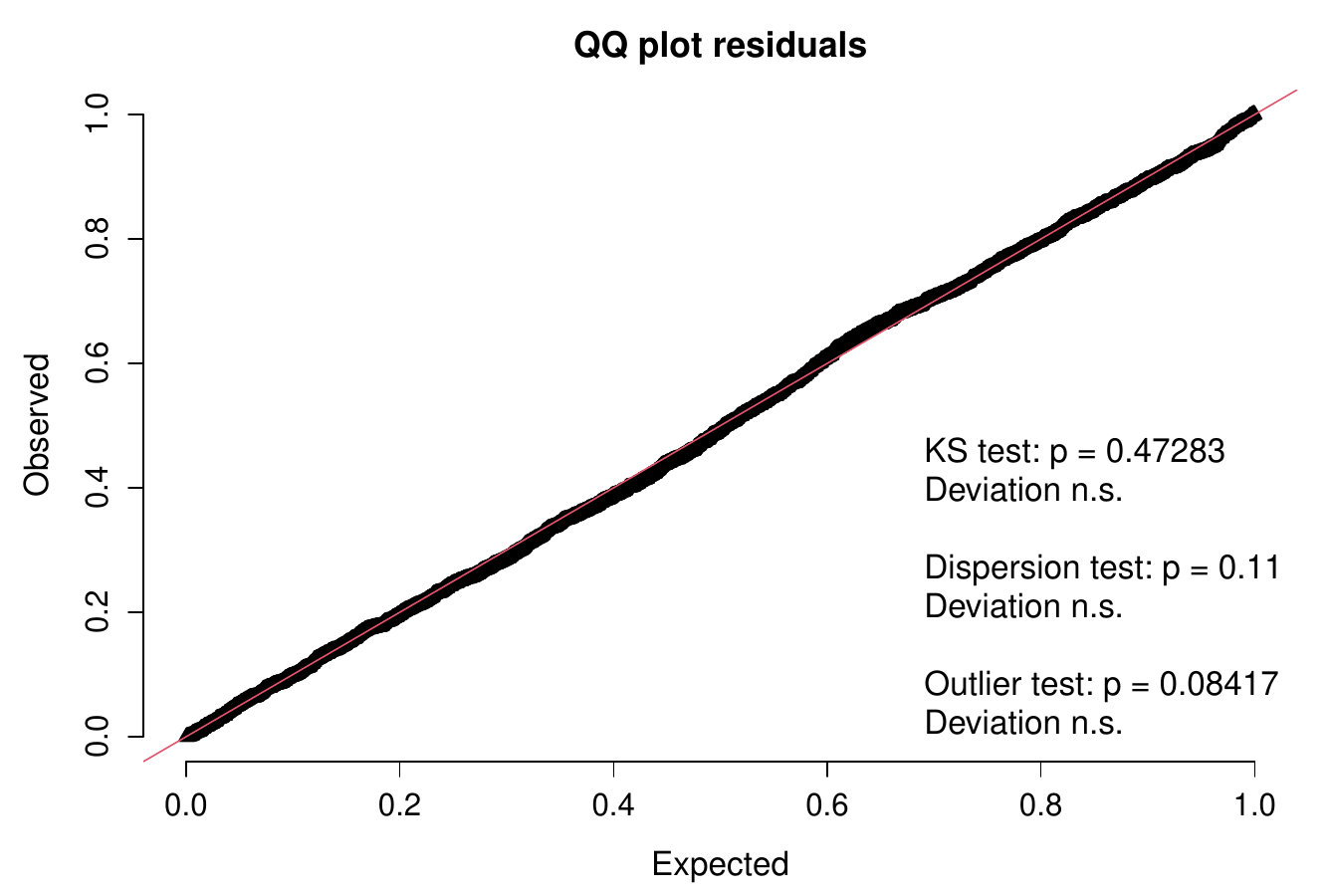}
    \caption{DHARMa residual diagnostics for the fitted Beta mixed-effects regression model. The QQ plot shows no significant deviation from the expected uniform residual distribution.}
    \label{fig:qq_plot}
\end{figure}

\subsection{Compute resources and runtime}\label{GPU}
All experiments were run on a single internal SLURM cluster. Each generation job uses one NVIDIA GPU (A100, H100, or H200), 8 CPU cores, and 100 GB host RAM, with a 6-hour wall-clock timeout per job. Overall, we estimate total project compute at approximately 30-40 GPU-hours, overwhelmingly dominated by rollout generation; metric computation and statistical analysis were negligible by comparison and can be completed on CPU.

\section{Additional robustness and sensitivity analyses}
\label{sec:robustness}

\subsection{Sensitivity to Top-$k$ Truncation}
\label{app:topk}
Our main experiments estimate next-token entropy using the renormalized distribution over the top $k=100$ tokens. Since truncating the probability distribution may underestimate entropy, we examine the sensitivity of our findings to this approximation by repeating the same experiments after increasing $k$ to $5000$.

\paragraph{Global uncertainty.} \autoref{tab:global_uncertainty_topk5000} reports the results for GenPPL, BF, and CE$^\star$ under $k=5000$. Consistent with our main findings, instruction tuning reduces GenPPL and BF across all 12 model--task comparisons, and the reductions are generally larger than under top-$k=100$. CE$^\star$ decreases in 11 of 12 comparisons, with the only exception being Gemma-3-12B on story generation, which shows
a small increase of $4.43\%\pm6.8\%$.

\begin{table*}[h]
\centering
\small
\setlength{\tabcolsep}{2.2pt}
\renewcommand{\arraystretch}{1.14}
\caption{Generation-space uncertainty under top-$k=5000$. Standard errors are estimated via prompt-cluster bootstrapping as in \autoref{app:bootstrap}. Blue cells indicate reductions in uncertainty after instruction tuning, while red cells indicate increases.}
\label{tab:global_uncertainty_topk5000}
\resizebox{\textwidth}{!}{%
\begin{tabular}{llccccccccc}
\toprule
\multirow{2}{*}{Task} & \multirow{2}{*}{Type}
& \multicolumn{3}{c}{\textbf{GenPPL}}
& \multicolumn{3}{c}{\textbf{BF}}
& \multicolumn{3}{c}{$\mathbf{CE}^{\star}$} \\
\cmidrule(lr){3-5}\cmidrule(lr){6-8}\cmidrule(lr){9-11}
&
& Qwen3-8B & Llama-3.1-8B & Gemma-3-12B
& Qwen3-8B & Llama-3.1-8B & Gemma-3-12B
& Qwen3-8B & Llama-3.1-8B & Gemma-3-12B \\
\midrule

\multirow{3}{*}{\textbf{Math}}
& Base
& 1.26 $\pm$ 0.013 & 7.76 $\pm$ 0.15 & 6.32 $\pm$ 0.087
& 1.24 $\pm$ 0.0067 & 9.29 $\pm$ 0.23 & 6.34 $\pm$ 0.077
& 65.1 $\pm$ 3.7 & 830 $\pm$ 31 & 593 $\pm$ 14 \\

& Instruct
& 1.17 $\pm$ 0.0073 & 4.95 $\pm$ 1.5 & 1.16 $\pm$ 0.0091
& 1.16 $\pm$ 0.0057 & 1.63 $\pm$ 0.052 & 1.15 $\pm$ 0.0052
& 41.0 $\pm$ 2.7 & 346 $\pm$ 84 & 56.7 $\pm$ 4.7 \\

& \% Change
& \cellcolor{colddark!7!coldlight} -7.26 $\pm$ 0.94\%
& \cellcolor{colddark!36!coldlight} -36.2 $\pm$ 20\%
& \cellcolor{colddark!82!coldlight} -81.6 $\pm$ 0.28\%
& \cellcolor{colddark!7!coldlight} -6.52 $\pm$ 0.40\%
& \cellcolor{colddark!82!coldlight} -82.4 $\pm$ 0.65\%
& \cellcolor{colddark!82!coldlight} -81.9 $\pm$ 0.20\%
& \cellcolor{colddark!37!coldlight} -37.0 $\pm$ 4.1\%
& \cellcolor{colddark!58!coldlight} -58.3 $\pm$ 10\%
& \cellcolor{colddark!90!coldlight} -90.4 $\pm$ 0.88\% \\

\midrule

\multirow{3}{*}{\textbf{Coding}}
& Base
& 1.93 $\pm$ 0.022 & 5.47 $\pm$ 0.13 & 3.39 $\pm$ 0.045
& 2.35 $\pm$ 0.043 & 7.04 $\pm$ 0.084 & 3.88 $\pm$ 0.049
& 288 $\pm$ 8.9 & 551 $\pm$ 21 & 675 $\pm$ 33 \\

& Instruct
& 1.23 $\pm$ 0.0054 & 1.30 $\pm$ 0.013 & 1.04 $\pm$ 0.0023
& 1.21 $\pm$ 0.0061 & 1.28 $\pm$ 0.0085 & 1.04 $\pm$ 0.0018
& 81.9 $\pm$ 4.0 & 138 $\pm$ 6.1 & 16.2 $\pm$ 1.0 \\

& \% Change
& \cellcolor{colddark!37!coldlight} -36.5 $\pm$ 0.76\%
& \cellcolor{colddark!76!coldlight} -76.2 $\pm$ 0.57\%
& \cellcolor{colddark!69!coldlight} -69.2 $\pm$ 0.40\%
& \cellcolor{colddark!49!coldlight} -48.5 $\pm$ 0.99\%
& \cellcolor{colddark!82!coldlight} -81.8 $\pm$ 0.24\%
& \cellcolor{colddark!73!coldlight} -73.1 $\pm$ 0.33\%
& \cellcolor{colddark!72!coldlight} -71.6 $\pm$ 1.2\%
& \cellcolor{colddark!75!coldlight} -74.9 $\pm$ 1.3\%
& \cellcolor{colddark!98!coldlight} -97.6 $\pm$ 0.20\% \\

\midrule

\multirow{3}{*}{\textbf{Sentence Completion}}
& Base
& 14.8 $\pm$ 0.78 & 19.9 $\pm$ 0.52 & 13.8 $\pm$ 0.15
& 7.42 $\pm$ 0.31 & 18.0 $\pm$ 0.47 & 14.4 $\pm$ 0.17
& 1926 $\pm$ 110 & 1631 $\pm$ 60 & 1359 $\pm$ 44 \\

& Instruct
& 1.69 $\pm$ 0.027 & 2.74 $\pm$ 0.13 & 1.79 $\pm$ 0.020
& 1.56 $\pm$ 0.018 & 2.89 $\pm$ 0.070 & 1.72 $\pm$ 0.024
& 114 $\pm$ 11 & 219 $\pm$ 14 & 517 $\pm$ 30 \\

& \% Change
& \cellcolor{colddark!89!coldlight} -88.6 $\pm$ 0.57\%
& \cellcolor{colddark!86!coldlight} -86.2 $\pm$ 0.73\%
& \cellcolor{colddark!87!coldlight} -87.0 $\pm$ 0.17\%
& \cellcolor{colddark!79!coldlight} -79.0 $\pm$ 0.81\%
& \cellcolor{colddark!84!coldlight} -84.0 $\pm$ 0.51\%
& \cellcolor{colddark!88!coldlight} -88.0 $\pm$ 0.20\%
& \cellcolor{colddark!94!coldlight} -94.1 $\pm$ 0.48\%
& \cellcolor{colddark!87!coldlight} -86.6 $\pm$ 1.0\%
& \cellcolor{colddark!62!coldlight} -61.9 $\pm$ 2.3\% \\

\midrule

\multirow{3}{*}{\textbf{Story Generation}}
& Base
& 5.31 $\pm$ 0.29 & 20.1 $\pm$ 0.31 & 11.7 $\pm$ 0.11
& 5.57 $\pm$ 0.20 & 17.9 $\pm$ 0.26 & 12.0 $\pm$ 0.10
& 626 $\pm$ 34 & 1996 $\pm$ 53 & 878 $\pm$ 29 \\

& Instruct
& 2.35 $\pm$ 0.043 & 8.57 $\pm$ 0.87 & 2.26 $\pm$ 0.027
& 2.16 $\pm$ 0.043 & 6.04 $\pm$ 0.41 & 2.20 $\pm$ 0.030
& 469 $\pm$ 29 & 973 $\pm$ 60 & 917 $\pm$ 40 \\

& \% Change
& \cellcolor{colddark!56!coldlight} -55.7 $\pm$ 1.8\%
& \cellcolor{colddark!57!coldlight} -57.4 $\pm$ 4.3\%
& \cellcolor{colddark!81!coldlight} -80.7 $\pm$ 0.23\%
& \cellcolor{colddark!61!coldlight} -61.3 $\pm$ 1.4\%
& \cellcolor{colddark!66!coldlight} -66.3 $\pm$ 2.4\%
& \cellcolor{colddark!82!coldlight} -81.6 $\pm$ 0.23\%
& \cellcolor{colddark!25!coldlight} -25.0 $\pm$ 2.8\%
& \cellcolor{colddark!51!coldlight} -51.2 $\pm$ 2.8\%
& \cellcolor{warmdark!4!warmlight} 4.43 $\pm$ 6.8\% \\

\bottomrule
\end{tabular}%
}
\end{table*}

\paragraph{Length--entropy coupling.} \autoref{tab:er_len_r_change_topk5000} reports the corresponding prompt-controlled length--entropy correlations. $\rho(N,r_N)$ increases after instruction tuning in the same 11 of the 12 model--task comparisons as in the original $k=100$ setting. This consistency under a substantially larger $k$ suggests that the predominantly positive shift in length--entropy coupling is not an artifact of the original top-$100$ trucation.

\begin{table*}[h]
\centering
\small
\setlength{\tabcolsep}{1.9pt}
\renewcommand{\arraystretch}{1.10}
\caption{Base--instruct comparison of the length--entropy
Pearson correlation $\rho(N,r_N)$ under top-$k=5000$. Each entry reports the estimate $\pm$ standard error, with the absolute change computed as the instruct value minus the corresponding base value.}
\label{tab:er_len_r_change_topk5000}
\resizebox{\textwidth}{!}{%
\begin{tabular}{llcccccccccccc}
\toprule
\multirow{2}{*}{Metric} &
\multirow{2}{*}{Model}
& \multicolumn{3}{c}{\textbf{MATH}}
& \multicolumn{3}{c}{\textbf{CODING}}
& \multicolumn{3}{c}{\textbf{SENTENCE COMPLETION}}
& \multicolumn{3}{c}{\textbf{STORY GENERATION}} \\
\cmidrule(lr){3-5}
\cmidrule(lr){6-8}
\cmidrule(lr){9-11}
\cmidrule(lr){12-14}
&
& Base & Instruct & Abs. Change
& Base & Instruct & Abs. Change
& Base & Instruct & Abs. Change
& Base & Instruct & Abs. Change \\
\midrule

\multirow{3}{*}{\textbf{$\rho(N, r_N)$}}

& Qwen
& \cellcolor{warmdark!2!warmlight} 0.021 $\pm$ 0.022
& \cellcolor{warmdark!10!warmlight} 0.096 $\pm$ 0.039
& \cellcolor{warmdark!8!warmlight} 0.075 $\pm$ 0.043

& \cellcolor{colddark!28!coldlight} -0.28 $\pm$ 0.017
& \cellcolor{warmdark!62!warmlight} 0.62 $\pm$ 0.033
& \cellcolor{warmdark!90!warmlight} 0.90 $\pm$ 0.036

& \cellcolor{warmdark!46!warmlight} 0.46 $\pm$ 0.017
& \cellcolor{warmdark!52!warmlight} 0.52 $\pm$ 0.024
& \cellcolor{warmdark!7!warmlight} 0.066 $\pm$ 0.029

& \cellcolor{colddark!6!coldlight} -0.063 $\pm$ 0.020
& \cellcolor{warmdark!19!warmlight} 0.19 $\pm$ 0.021
& \cellcolor{warmdark!25!warmlight} 0.25 $\pm$ 0.027 \\

& Llama
& \cellcolor{colddark!10!coldlight} -0.10 $\pm$ 0.017
& \cellcolor{warmdark!69!warmlight} 0.69 $\pm$ 0.046
& \cellcolor{warmdark!80!warmlight} 0.80 $\pm$ 0.052

& \cellcolor{colddark!20!coldlight} -0.20 $\pm$ 0.014
& \cellcolor{warmdark!33!warmlight} 0.33 $\pm$ 0.026
& \cellcolor{warmdark!53!warmlight} 0.53 $\pm$ 0.027

& \cellcolor{warmdark!12!warmlight} 0.12 $\pm$ 0.014
& \cellcolor{colddark!18!coldlight} -0.18 $\pm$ 0.025
& \cellcolor{colddark!29!coldlight} -0.29 $\pm$ 0.028

& \cellcolor{warmdark!20!warmlight} 0.20 $\pm$ 0.015
& \cellcolor{warmdark!41!warmlight} 0.41 $\pm$ 0.025
& \cellcolor{warmdark!21!warmlight} 0.21 $\pm$ 0.027 \\

& Gemma
& \cellcolor{warmdark!2!warmlight} 0.015 $\pm$ 0.016
& \cellcolor{warmdark!21!warmlight} 0.21 $\pm$ 0.038
& \cellcolor{warmdark!19!warmlight} 0.19 $\pm$ 0.041

& \cellcolor{colddark!15!coldlight} -0.15 $\pm$ 0.0091
& \cellcolor{warmdark!39!warmlight} 0.39 $\pm$ 0.035
& \cellcolor{warmdark!54!warmlight} 0.54 $\pm$ 0.034

& \cellcolor{colddark!12!coldlight} -0.12 $\pm$ 0.014
& \cellcolor{warmdark!16!warmlight} 0.16 $\pm$ 0.039
& \cellcolor{warmdark!27!warmlight} 0.27 $\pm$ 0.042

& \cellcolor{colddark!3!coldlight} -0.034 $\pm$ 0.013
& \cellcolor{colddark!2!coldlight} -0.022 $\pm$ 0.027
& \cellcolor{warmdark!1!warmlight} 0.013 $\pm$ 0.029 \\

\bottomrule
\end{tabular}%
}
\end{table*}

\paragraph{Semantic diversity.} We refit the main Beta mixed-effects regression under $k=5000$, keeping the same specification. As shown in \autoref{tab:beta_topk5000}, the entropy rate coefficient remains positive ($\beta_1 = 0.562$, $p < 2\times 10^{-16}$), with a strongly positive entropy-rate $\times$ fine-tuning interaction ($\beta_3 = 1.497$, $p < 2\times 10^{-16}$). The entropy rate slope thus increases from $0.562$ for base models to $2.059$ for fine-tuned models ($3.66\times$), consistent with our main finding that fine-tuned models translate token uncertainty into semantic diversity more effectively. The main effect of fine-tuning remains negative, although it is only marginally significant ($\beta_2 = -0.181$, $p = 0.074$). Residual diagnostics show no significant deviations in the KS ($p = 0.16$) or outlier ($p = 0.084$) tests, although the dispersion test detects residual overdispersion ($p < 2.2\times 10^{-16}$).

\begin{table}[h]
\centering
\caption{Full table of conditional mean model estimates for the Beta mixed-effects regression with top-$k=5000$.}
\label{tab:beta_topk5000}
\begin{tabular}{lrrr}
\toprule
Term & Estimate & Std. Error & $p$-value \\
\midrule
Intercept & -0.823 & 0.096 & $< 2\times 10^{-16}$ \\
\textbf{\boldmath $R$ ($\beta_1$)} & \textbf{0.562} & 0.037 & $< 2\times 10^{-16}$ \\
\textbf{\boldmath Instruct ($\beta_2$)} & \textbf{-0.181} & 0.102 & 0.074 \\
$\mathrm{ns}(\mathrm{invN},4)_1$ & 0.404 & 0.038 & $< 2\times 10^{-16}$ \\
$\mathrm{ns}(\mathrm{invN},4)_2$ & 1.080 & 0.274 & $7.91 \times 10^{-5}$ \\
$\mathrm{ns}(\mathrm{invN},4)_3$ & -0.621 & 1.450 & 0.668 \\
$\mathrm{ns}(\mathrm{invN},4)_4$ & -3.039 & 3.076 & 0.323 \\
Coding & -0.135 & 0.040 & $7.19 \times 10^{-4}$ \\
Math & 0.073 & 0.032 & 0.021 \\
Stories & -0.208 & 0.022 & $< 2\times 10^{-16}$ \\
\textbf{\boldmath $R \times$ Instruct ($\beta_3$)} & \textbf{1.497} & 0.055 & $< 2\times 10^{-16}$ \\
Instruct $\times \mathrm{ns}(\mathrm{invN},4)_1$ & -0.479 & 0.079 & $1.51 \times 10^{-9}$ \\
Instruct $\times \mathrm{ns}(\mathrm{invN},4)_2$ & 1.358 & 0.346 & $8.80 \times 10^{-5}$ \\
Instruct $\times \mathrm{ns}(\mathrm{invN},4)_3$ & 2.232 & 1.463 & 0.127 \\
Instruct $\times \mathrm{ns}(\mathrm{invN},4)_4$ & 3.834 & 3.076 & 0.213 \\
\bottomrule
\end{tabular}
\end{table}

\subsection{Robustness to few-shot prompting}
\label{app:fewshot}
We extend our analysis to five-shot prompting, where base models exhibit substantially different generation and stopping behaviors compared to the zero-shot setting. Rather than falling into ``completion traps'' (\autoref{app:completion_traps}), base models frequently continue generating additional question--answer pairs, resulting in longer generation lengths and a greater proportion of long trajectories.

\paragraph{Global uncertainty.} \autoref{tab:global_uncertainty_fewshot} reports the results for GenPPL, BF, and CE$^\star$ under five-shot prompting. Consistent with our main findings, instruction tuning reduces GenPPL and BF across all 12 model--task comparisons. CE$^\star$ also decreases in all 12 comparisons, including story generation, where the short few-shot example answers lead fine-tuned models to produce substantially shorter outputs, reducing the contribution of trajectory length to total uncertainty.

\begin{table*}[h]
\centering
\small
\setlength{\tabcolsep}{2.2pt}
\renewcommand{\arraystretch}{1.14}
\caption{Generation-space uncertainty under few-shot prompting. Standard errors are estimated via prompt-cluster bootstrapping as in \autoref{app:bootstrap}. Blue cells indicate reductions in uncertainty after instruction tuning, while red cells indicate increases.}
\label{tab:global_uncertainty_fewshot}
\resizebox{\textwidth}{!}{%
\begin{tabular}{llccccccccc}
\toprule
\multirow{2}{*}{Task} & \multirow{2}{*}{Type}
& \multicolumn{3}{c}{\textbf{GenPPL}}
& \multicolumn{3}{c}{\textbf{BF}}
& \multicolumn{3}{c}{$\mathbf{CE}^{\star}$} \\
\cmidrule(lr){3-5}\cmidrule(lr){6-8}\cmidrule(lr){9-11}
&
& Qwen3-8B & Llama-3.1-8B & Gemma-3-12B
& Qwen3-8B & Llama-3.1-8B & Gemma-3-12B
& Qwen3-8B & Llama-3.1-8B & Gemma-3-12B \\
\midrule

\multirow{3}{*}{\textbf{Math}}
& Base
& 1.37 $\pm$ 0.0023 & 3.68 $\pm$ 0.014 & 3.86 $\pm$ 0.012
& 1.38 $\pm$ 0.0028 & 3.61 $\pm$ 0.014 & 3.53 $\pm$ 0.013
& 380 $\pm$ 3.3 & 2413 $\pm$ 19 & 2560 $\pm$ 22 \\

& Instruct
& 1.10 $\pm$ 0.0041 & 1.74 $\pm$ 0.065 & 1.15 $\pm$ 0.024
& 1.09 $\pm$ 0.0034 & 1.48 $\pm$ 0.021 & 1.10 $\pm$ 0.0070
& 12.7 $\pm$ 0.88 & 93.6 $\pm$ 9.6 & 17.4 $\pm$ 3.4 \\

& \% Change
& \cellcolor{colddark!20!coldlight} -19.8 $\pm$ 0.30\%
& \cellcolor{colddark!53!coldlight} -52.7 $\pm$ 1.8\%
& \cellcolor{colddark!70!coldlight} -70.3 $\pm$ 0.61\%
& \cellcolor{colddark!21!coldlight} -20.7 $\pm$ 0.25\%
& \cellcolor{colddark!59!coldlight} -58.9 $\pm$ 0.54\%
& \cellcolor{colddark!69!coldlight} -68.7 $\pm$ 0.19\%
& \cellcolor{colddark!97!coldlight} -96.7 $\pm$ 0.22\%
& \cellcolor{colddark!96!coldlight} -96.1 $\pm$ 0.39\%
& \cellcolor{colddark!99!coldlight} -99.3 $\pm$ 0.13\% \\

\midrule

\multirow{3}{*}{\textbf{Coding}}
& Base
& 1.95 $\pm$ 0.0082 & 2.26 $\pm$ 0.0098 & 2.25 $\pm$ 0.0070
& 2.01 $\pm$ 0.012 & 2.10 $\pm$ 0.010 & 2.13 $\pm$ 0.0078
& 564 $\pm$ 7.6 & 828 $\pm$ 11 & 2222 $\pm$ 15 \\

& Instruct
& 1.06 $\pm$ 0.0021 & 1.42 $\pm$ 0.010 & 1.05 $\pm$ 0.0020
& 1.06 $\pm$ 0.0019 & 1.41 $\pm$ 0.0092 & 1.05 $\pm$ 0.0018
& 10.6 $\pm$ 0.61 & 97.0 $\pm$ 3.3 & 11.7 $\pm$ 0.69 \\

& \% Change
& \cellcolor{colddark!46!coldlight} -45.8 $\pm$ 0.22\%
& \cellcolor{colddark!37!coldlight} -37.0 $\pm$ 0.45\%
& \cellcolor{colddark!53!coldlight} -53.3 $\pm$ 0.15\%
& \cellcolor{colddark!48!coldlight} -47.5 $\pm$ 0.31\%
& \cellcolor{colddark!33!coldlight} -33.1 $\pm$ 0.38\%
& \cellcolor{colddark!51!coldlight} -50.7 $\pm$ 0.19\%
& \cellcolor{colddark!98!coldlight} -98.1 $\pm$ 0.11\%
& \cellcolor{colddark!88!coldlight} -88.3 $\pm$ 0.43\%
& \cellcolor{colddark!99!coldlight} -99.5 $\pm$ 0.031\% \\

\midrule

\multirow{3}{*}{\textbf{Sentence Completion}}
& Base
& 6.76 $\pm$ 0.032 & 6.30 $\pm$ 0.043 & 9.26 $\pm$ 0.027
& 6.42 $\pm$ 0.036 & 7.44 $\pm$ 0.040 & 9.45 $\pm$ 0.038
& 3006 $\pm$ 36 & 2157 $\pm$ 25 & 1122 $\pm$ 15 \\

& Instruct
& 1.66 $\pm$ 0.0045 & 2.83 $\pm$ 0.022 & 1.90 $\pm$ 0.0098
& 1.66 $\pm$ 0.0047 & 2.86 $\pm$ 0.024 & 1.86 $\pm$ 0.012
& 367 $\pm$ 5.4 & 298 $\pm$ 2.1 & 448 $\pm$ 11 \\

& \% Change
& \cellcolor{colddark!75!coldlight} -75.4 $\pm$ 0.12\%
& \cellcolor{colddark!55!coldlight} -55.1 $\pm$ 0.46\%
& \cellcolor{colddark!79!coldlight} -79.5 $\pm$ 0.11\%
& \cellcolor{colddark!74!coldlight} -74.2 $\pm$ 0.16\%
& \cellcolor{colddark!62!coldlight} -61.5 $\pm$ 0.37\%
& \cellcolor{colddark!80!coldlight} -80.3 $\pm$ 0.13\%
& \cellcolor{colddark!88!coldlight} -87.8 $\pm$ 0.20\%
& \cellcolor{colddark!86!coldlight} -86.2 $\pm$ 0.20\%
& \cellcolor{colddark!60!coldlight} -60.0 $\pm$ 1.1\% \\

\midrule

\multirow{3}{*}{\textbf{Story Generation}}
& Base
& 4.70 $\pm$ 0.053 & 5.13 $\pm$ 0.055 & 6.98 $\pm$ 0.044
& 5.09 $\pm$ 0.067 & 6.16 $\pm$ 0.068 & 7.02 $\pm$ 0.056
& 323 $\pm$ 7.7 & 1836 $\pm$ 30 & 995 $\pm$ 19 \\

& Instruct
& 1.68 $\pm$ 0.015 & 2.86 $\pm$ 0.054 & 2.30 $\pm$ 0.020
& 1.63 $\pm$ 0.014 & 3.03 $\pm$ 0.049 & 2.19 $\pm$ 0.027
& 64.4 $\pm$ 5.2 & 217 $\pm$ 9.0 & 537 $\pm$ 30 \\

& \% Change
& \cellcolor{colddark!64!coldlight} -64.2 $\pm$ 0.36\%
& \cellcolor{colddark!44!coldlight} -44.3 $\pm$ 1.3\%
& \cellcolor{colddark!67!coldlight} -67.1 $\pm$ 0.40\%
& \cellcolor{colddark!68!coldlight} -67.9 $\pm$ 0.29\%
& \cellcolor{colddark!51!coldlight} -50.8 $\pm$ 0.95\%
& \cellcolor{colddark!69!coldlight} -68.8 $\pm$ 0.37\%
& \cellcolor{colddark!80!coldlight} -80.1 $\pm$ 1.3\%
& \cellcolor{colddark!88!coldlight} -88.2 $\pm$ 0.44\%
& \cellcolor{colddark!46!coldlight} -46.0 $\pm$ 2.5\% \\

\bottomrule
\end{tabular}%
}
\end{table*}

\textbf{Length--entropy coupling.} As shown in \autoref{tab:er_len_r_change_fewshot}, $\rho(N,r_N)$ continues to increase after instruction tuning in 9 of the 12 model--task comparisons, despite base models frequently generating additional question--answer pairs that substantially extend their outputs. This suggests that the predominantly positive shift in length--entropy coupling is robust to changes in prompting strategy and generation behavior.

\begin{table*}[h]
\centering
\small
\setlength{\tabcolsep}{1.9pt}
\renewcommand{\arraystretch}{1.10}
\caption{Base--instruct comparison of the length--entropy
Pearson correlation $\rho(N,r_N)$ under few-shot prompting. Each entry reports the estimate $\pm$ standard error, with the absolute change computed as the instruct value minus the corresponding base value.}
\label{tab:er_len_r_change_fewshot}
\resizebox{\textwidth}{!}{%
\begin{tabular}{llcccccccccccc}
\toprule
\multirow{2}{*}{Metric} &
\multirow{2}{*}{Model}
& \multicolumn{3}{c}{\textbf{MATH}}
& \multicolumn{3}{c}{\textbf{CODING}}
& \multicolumn{3}{c}{\textbf{SENTENCE COMPLETION}}
& \multicolumn{3}{c}{\textbf{STORY GENERATION}} \\
\cmidrule(lr){3-5}
\cmidrule(lr){6-8}
\cmidrule(lr){9-11}
\cmidrule(lr){12-14}
&
& Base & Instruct & Abs. Change
& Base & Instruct & Abs. Change
& Base & Instruct & Abs. Change
& Base & Instruct & Abs. Change \\
\midrule

\multirow{3}{*}{\textbf{$\rho(N, r_N)$}}

& Qwen
& \cellcolor{colddark!5!coldlight} -0.053 $\pm$ 0.014
& \cellcolor{warmdark!14!warmlight} 0.14 $\pm$ 0.044
& \cellcolor{warmdark!20!warmlight} 0.20 $\pm$ 0.045

& \cellcolor{colddark!4!coldlight} -0.045 $\pm$ 0.0088
& \cellcolor{warmdark!37!warmlight} 0.37 $\pm$ 0.042
& \cellcolor{warmdark!41!warmlight} 0.41 $\pm$ 0.042

& \cellcolor{warmdark!11!warmlight} 0.11 $\pm$ 0.0095
& \cellcolor{warmdark!16!warmlight} 0.16 $\pm$ 0.016
& \cellcolor{warmdark!4!warmlight} 0.044 $\pm$ 0.018

& \cellcolor{colddark!25!coldlight} -0.25 $\pm$ 0.016
& \cellcolor{warmdark!9!warmlight} 0.086 $\pm$ 0.021
& \cellcolor{warmdark!34!warmlight} 0.34 $\pm$ 0.030 \\

& Llama
& \cellcolor{warmdark!9!warmlight} 0.095 $\pm$ 0.011
& \cellcolor{warmdark!49!warmlight} 0.49 $\pm$ 0.034
& \cellcolor{warmdark!40!warmlight} 0.40 $\pm$ 0.036

& \cellcolor{warmdark!33!warmlight} 0.33 $\pm$ 0.014
& \cellcolor{warmdark!29!warmlight} 0.29 $\pm$ 0.021
& \cellcolor{colddark!4!coldlight} -0.040 $\pm$ 0.024

& \cellcolor{colddark!44!coldlight} -0.44 $\pm$ 0.0093
& \cellcolor{colddark!9!coldlight} -0.086 $\pm$ 0.014
& \cellcolor{warmdark!35!warmlight} 0.35 $\pm$ 0.016

& \cellcolor{colddark!44!coldlight} -0.44 $\pm$ 0.012
& \cellcolor{colddark!22!coldlight} -0.22 $\pm$ 0.017
& \cellcolor{warmdark!22!warmlight} 0.22 $\pm$ 0.016 \\

& Gemma
& \cellcolor{warmdark!51!warmlight} 0.51 $\pm$ 0.0090
& \cellcolor{warmdark!23!warmlight} 0.23 $\pm$ 0.040
& \cellcolor{colddark!28!coldlight} -0.28 $\pm$ 0.042

& \cellcolor{warmdark!58!warmlight} 0.58 $\pm$ 0.010
& \cellcolor{warmdark!49!warmlight} 0.49 $\pm$ 0.035
& \cellcolor{colddark!8!coldlight} -0.083 $\pm$ 0.035

& \cellcolor{colddark!12!coldlight} -0.12 $\pm$ 0.013
& \cellcolor{warmdark!31!warmlight} 0.31 $\pm$ 0.024
& \cellcolor{warmdark!43!warmlight} 0.43 $\pm$ 0.027

& \cellcolor{colddark!4!coldlight} -0.043 $\pm$ 0.012
& \cellcolor{warmdark!11!warmlight} 0.11 $\pm$ 0.023
& \cellcolor{warmdark!16!warmlight} 0.16 $\pm$ 0.023 \\

\bottomrule
\end{tabular}%
}
\end{table*}

\paragraph{Semantic diversity.}
The substantially different generation behaviors under five-shot prompting introduce additional heterogeneity in response lengths and semantic diversity across model--task combinations. We therefore adapt the model to better account for these differences:
\begin{align*}
\operatorname{logit}(\mu_{tpm}) ={}& \beta_1 R_{tpm} + \beta_3 R_{tpm}\mathbf{1}\{m=\mathrm{FT}\} \\
&+ f_0(\log N_{tpm}) + f_1(\log N_{tpm})\mathbf{1}\{m=\mathrm{FT}\} \\
&+ \alpha_{tm} + \psi s_{tpm} + u_p, \\[4pt]
\log(\phi_{tpm}) ={}& \delta_{tm} + h(\log N_{tpm}) + \gamma_2 R_{tpm} \\
&+ \gamma_3 R_{tpm}\mathbf{1}\{m=\mathrm{FT}\} + \kappa s_{tpm}.
\end{align*}
Here, $(t,p,m)$ indexes the task, prompt, and model, respectively. We introduce model--task-specific intercepts in both the conditional mean and dispersion models, $\alpha_{tm}$ and $\delta_{tm}$, respectively, to account for differences in baseline semantic diversity across model--task combinations and in its variability across prompts within each combination. We also replace inverse sequence length with log sequence length to better distinguish long generations, as output lengths vary substantially more under five-shot prompting. The additional term $s_{tpm}$, the standard deviation of log sequence lengths across rollouts for each task--prompt--model combination, captures within-prompt length variability arising from the mixture of shorter responses and extended question--answer continuations. As before, $R_{tpm}$, $\log N_{tpm}$, and $s_{tpm}$ are standardized, and $f_0$, $f_1$, and $h$ denote natural spline functions.

As shown in \autoref{tab:beta_fewshot}, the entropy rate coefficient remains positive ($\beta_1=0.115$, $p = 5.28\times10^{-6}$), with a strongly positive entropy rate $\times$ fine-tuning interaction ($\beta_3=0.483$, $p<2\times10^{-16}$). The entropy rate slope thus increases from $0.115$ for base models to $0.598$ for fine-tuned models ($5.2 \times$), consistent with our main finding that fine-tuned models exhibit a stronger relationship between entropy rate and semantic diversity. Residual diagnostics indicate an adequate fit (KS $p=0.13$, dispersion $p=0.81$, outlier $p=0.084$).

\begin{table}[h]
\centering
\caption{Conditional mean model estimates for the Beta mixed-effects regression under few-shot prompting, omitting the model--task intercepts $\alpha_{tm}$ (24 levels).}
\label{tab:beta_fewshot}
\begin{tabular}{lrrr}
\toprule
Term & Estimate & Std. Error & $p$-value \\
\midrule
\textbf{\boldmath $R$ ($\beta_1$)} & \textbf{0.115} & 0.025 & $5.28 \times 10^{-6}$ \\
\textbf{\boldmath $R \times$ Instruct ($\beta_3$)} & \textbf{0.483} & 0.059 & $< 2\times 10^{-16}$ \\
Length spread $s$ ($\psi$) & 0.282 & 0.009 & $< 2\times 10^{-16}$ \\
$\mathrm{ns}(\log N,4)_1$ & -0.277 & 0.250 & 0.269 \\
$\mathrm{ns}(\log N,4)_2$ & -0.471 & 0.175 & 0.007 \\
$\mathrm{ns}(\log N,4)_3$ & -0.469 & 0.501 & 0.349 \\
$\mathrm{ns}(\log N,4)_4$ & -1.175 & 0.112 & $< 2\times 10^{-16}$ \\
Instruct $\times \mathrm{ns}(\log N,4)_1$ & 0.192 & 0.282 & 0.497 \\
Instruct $\times \mathrm{ns}(\log N,4)_2$ & -1.970 & 0.261 & $4.66 \times 10^{-14}$ \\
Instruct $\times \mathrm{ns}(\log N,4)_3$ & 3.292 & 1.049 & 0.002 \\
Instruct $\times \mathrm{ns}(\log N,4)_4$ & 7.576 & 1.514 & $5.59 \times 10^{-7}$ \\
\bottomrule
\end{tabular}
\end{table}

\subsection{Robustness of the entropy rate--semantic diversity relationship}
\label{app:semantic_robustness}

\subsubsection{Alternative regression specifications}
We assess sensitivity to treating model family as a random effect by replacing random intercepts $v_m$ with fixed family effects, keeping all other terms unchanged. Here, $(t,p,m)$ indexes the task, prompt, and model, respectively. The conditional mean model is then specified as
$$
\begin{aligned}
\operatorname{logit}(\mu_{tpm}) ={}&
\alpha + \beta_1 R_{tpm}
+ \beta_2 \mathbf{1}\{m=\mathrm{FT}\}
+ f_0(\mathrm{invN}_{tpm})\\
&+ \sum_k \theta_k \mathbf{1}\{t=k\}
+ \beta_3 (R_{tpm}\cdot\mathbf{1}\{m=\mathrm{FT}\})\\
&+ f_1(\mathrm{invN}_{tpm})\mathbf{1}\{m=\mathrm{FT}\}
+ u_p + \sum_j \omega_j \mathbf{1}\{\mathrm{fam}(m)=j\},
\end{aligned}
$$
where $\omega_j$ are the fixed effect coefficients for model family $\mathrm{fam}(m) = j \in \{$ Llama-3.1-8B, Qwen3-8B $\}$, with Gemma-3-12B as the reference. The dispersion model remains unchanged.

As shown in \autoref{tab:beta_fixed_family}, the entropy rate coefficient remains positive ($\beta_1 = 0.686$, $p < 2\times 10^{-16}$), with a strongly positive entropy-rate $\times$ fine-tuning interaction ($\beta_3 = 1.321$, $p < 2\times 10^{-16}$). The entropy rate slope thus increases from $0.686$ for base models to $2.007$ for fine-tuned models ($2.93\times$), and the main effect of fine-tuning remains negative ($\beta_2 = -0.330$, $p = 1.59\times 10^{-4}$). The KS test indicates well-behaved residuals ($p = 0.35$), while the dispersion ($p = 0.038$) and outlier ($p = 0.0015$) tests detect some deviations.

\begin{table}[h]
\centering
\caption{Full table of conditional mean model estimates for the Beta mixed-effects regression with fixed model-family effects.}
\label{tab:beta_fixed_family}
\begin{tabular}{lrrr}
\toprule
Term & Estimate & Std. Error & $p$-value \\
\midrule
Intercept & -0.779 & 0.052 & $< 2\times 10^{-16}$ \\
\textbf{\boldmath $R$ ($\beta_1$)} & \textbf{0.686} & 0.034 & $< 2\times 10^{-16}$ \\
\textbf{\boldmath Instruct ($\beta_2$)} & \textbf{-0.330} & 0.087 & $1.59 \times 10^{-4}$ \\
$\mathrm{ns}(\mathrm{invN},4)_1$ & 0.171 & 0.040 & $1.59 \times 10^{-5}$ \\
$\mathrm{ns}(\mathrm{invN},4)_2$ & 0.456 & 0.236 & 0.053 \\
$\mathrm{ns}(\mathrm{invN},4)_3$ & 0.367 & 1.497 & 0.807 \\
$\mathrm{ns}(\mathrm{invN},4)_4$ & -0.009 & 3.151 & 0.998 \\
Coding & -0.037 & 0.037 & 0.318 \\
Math & 0.179 & 0.034 & $2.06 \times 10^{-7}$ \\
Stories & -0.252 & 0.028 & $< 2\times 10^{-16}$ \\
Llama-3.1-8B & -0.177 & 0.018 & $< 2\times 10^{-16}$ \\
Qwen3-8B & -0.192 & 0.031 & $3.76 \times 10^{-10}$ \\
\textbf{\boldmath $R \times$ Instruct ($\beta_3$)} & \textbf{1.321} & 0.054 & $< 2\times 10^{-16}$ \\
Instruct $\times \mathrm{ns}(\mathrm{invN},4)_1$ & -0.097 & 0.077 & 0.211 \\
Instruct $\times \mathrm{ns}(\mathrm{invN},4)_2$ & 1.310 & 0.333 & $8.17 \times 10^{-5}$ \\
Instruct $\times \mathrm{ns}(\mathrm{invN},4)_3$ & 1.246 & 1.512 & 0.409 \\
Instruct $\times \mathrm{ns}(\mathrm{invN},4)_4$ & 1.072 & 3.154 & 0.734 \\
\bottomrule
\end{tabular}
\end{table}

\subsubsection{Alternative embedding choices}

\paragraph{ModernBERT with clustering prefix.} To assess sensitivity to the embedding prefix, we recompute semantic diversity $D_{tpm}$ using the same model modernbert-embed-large by replaceing the original \texttt{"search\_document:"} prefix with \texttt{"clustering:"}, which is designed for grouping semantically similar texts, while keeping the regression specification unchanged.

As shown in \autoref{tab:beta_modernbert_clustering}, the entropy-rate coefficient remains positive ($\beta_1 = 0.515$, $p < 2\times 10^{-16}$), with a strongly positive entropy-rate $\times$ fine-tuning interaction ($\beta_3 = 1.106$, $p < 2\times 10^{-16}$). The entropy-rate slope thus increases from $0.515$ for base models to $1.621$ for fine-tuned models ($3.15\times$), while the main effect of fine-tuning remains negative ($\beta_2 = -0.149$, $p = 0.027$). Residual diagnostics show no significant departures in the KS ($p = 0.085$) or dispersion ($p = 0.68$) tests, while the outlier test detects a small deviation ($p = 0.013$).

\begin{table}[h]
\centering
\caption{Full table of conditional mean model estimates for the Beta mixed-effects regression with semantic diversity computed using the ModernBERT clustering prefix.}
\label{tab:beta_modernbert_clustering}
\begin{tabular}{lrrr}
\toprule
Term & Estimate & Std. Error & $p$-value \\
\midrule
Intercept & -1.472 & 0.045 & $< 2\times 10^{-16}$ \\
\textbf{\boldmath $R$ ($\beta_1$)} & \textbf{0.515} & 0.025 & $< 2\times 10^{-16}$ \\
\textbf{\boldmath Instruct ($\beta_2$)} & \textbf{-0.149} & 0.068 & 0.027 \\
$\mathrm{ns}(\mathrm{invN},4)_1$ & -0.032 & 0.030 & 0.284 \\
$\mathrm{ns}(\mathrm{invN},4)_2$ & 0.137 & 0.171 & 0.423 \\
$\mathrm{ns}(\mathrm{invN},4)_3$ & -0.331 & 1.119 & 0.767 \\
$\mathrm{ns}(\mathrm{invN},4)_4$ & -0.566 & 2.357 & 0.810 \\
Coding & 0.258 & 0.029 & $< 2\times 10^{-16}$ \\
Math & 0.171 & 0.026 & $7.93 \times 10^{-11}$ \\
Stories & -0.160 & 0.023 & $1.29 \times 10^{-12}$ \\
\textbf{\boldmath $R \times$ Instruct ($\beta_3$)} & \textbf{1.106} & 0.047 & $< 2\times 10^{-16}$ \\
Instruct $\times \mathrm{ns}(\mathrm{invN},4)_1$ & -0.346 & 0.065 & $1.11 \times 10^{-7}$ \\
Instruct $\times \mathrm{ns}(\mathrm{invN},4)_2$ & 0.701 & 0.346 & 0.043 \\
Instruct $\times \mathrm{ns}(\mathrm{invN},4)_3$ & 2.132 & 1.471 & 0.147 \\
Instruct $\times \mathrm{ns}(\mathrm{invN},4)_4$ & 4.421 & 3.165 & 0.162 \\
\bottomrule
\end{tabular}
\end{table}

\paragraph{E5-Mistral.} To further examine whether our findings depend on the embedding model, we recompute semantic diversity $D_{tpm}$ using e5-mistral-7b-instruct \citep{DBLP:conf/acl/WangYHYMW24} with the instruction \texttt{"Instruct: Retrieve semantically similar text.\textbackslash nQuery:"}, as used in the original E5-Mistral semantic similarity evaluation, and keep the regression specification unchanged.

As shown in \autoref{tab:beta_e5_mistral}, the entropy rate coefficient remains positive ($\beta_1 = 0.295$, $p < 2\times 10^{-16}$), with a strongly positive entropy-rate $\times$ fine-tuning interaction ($\beta_3 = 1.061$, $p < 2\times 10^{-16}$). The entropy rate slope thus increases from $0.295$ for base models to $1.356$ for fine-tuned models ($4.60\times$), while the main effect of fine-tuning remains negative ($\beta_2 = -0.486$, $p = 2.43\times 10^{-12}$). Residual diagnostics show no significant deviations in the KS ($p = 0.079$), dispersion ($p = 0.81$), or outlier ($p = 0.38$) tests.

\begin{table}[h]
\centering
\caption{Full table of conditional mean model estimates for the Beta mixed-effects regression with semantic diversity computed using e5-mistral-7b-instruct embeddings.}
\label{tab:beta_e5_mistral}
\begin{tabular}{lrrr}
\toprule
Term & Estimate & Std. Error & $p$-value \\
\midrule
Intercept & -1.131 & 0.044 & $< 2\times 10^{-16}$ \\
\textbf{\boldmath $R$ ($\beta_1$)} & \textbf{0.295} & 0.025 & $< 2\times 10^{-16}$ \\
\textbf{\boldmath Instruct ($\beta_2$)} & \textbf{-0.486} & 0.069 & $2.43 \times 10^{-12}$ \\
$\mathrm{ns}(\mathrm{invN},4)_1$ & 0.121 & 0.026 & $4.28 \times 10^{-6}$ \\
$\mathrm{ns}(\mathrm{invN},4)_2$ & 0.141 & 0.149 & 0.342 \\
$\mathrm{ns}(\mathrm{invN},4)_3$ & 1.326 & 0.765 & 0.083 \\
$\mathrm{ns}(\mathrm{invN},4)_4$ & 2.387 & 1.630 & 0.143 \\
Coding & -0.196 & 0.026 & $4.59 \times 10^{-14}$ \\
Math & -0.181 & 0.023 & $1.18 \times 10^{-14}$ \\
Stories & -0.129 & 0.017 & $1.09 \times 10^{-14}$ \\
\textbf{\boldmath $R \times$ Instruct ($\beta_3$)} & \textbf{1.061} & 0.039 & $< 2\times 10^{-16}$ \\
Instruct $\times \mathrm{ns}(\mathrm{invN},4)_1$ & -0.171 & 0.060 & 0.004 \\
Instruct $\times \mathrm{ns}(\mathrm{invN},4)_2$ & 1.304 & 0.258 & $4.16 \times 10^{-7}$ \\
Instruct $\times \mathrm{ns}(\mathrm{invN},4)_3$ & -0.140 & 0.979 & 0.886 \\
Instruct $\times \mathrm{ns}(\mathrm{invN},4)_4$ & -0.617 & 2.110 & 0.770 \\
\bottomrule
\end{tabular}
\end{table}

\subsubsection{Alternative semantic diversity metrics}

\paragraph{Normalized Vendi Score.} As an alternative to average pairwise cosine distance, we evaluate semantic diversity using the Vendi Score \citep{friedman2022vendi}, which measures the effective number of distinct outputs based on the eigenvalue distribution of their similarity matrix. For each task--prompt--model tuple $(t,p,m)$, let $K \in \mathbb{R}^{M \times M}$ denote the cosine-similarity matrix of the $M$ rollout embeddings, and let $\xi_1,\ldots,\xi_M$ be the eigenvalues of $K/M$. The Vendi Score is then defined as
\[
\mathrm{VS}_{tpm} = \exp\left(-\sum_{i=1}^{M} \xi_i \log \xi_i\right),
\]
ranging from $1$ for identical outputs to $M$ for mutually orthogonal embeddings. It captures the overall similarity structure among $M$ rollouts rather than individual pairwise relationships. We then normalize the score via $(\mathrm{VS}_{tpm}-1)/(M-1)$ to $[0,1]$ for the Beta regression with specifications unchanged.

As shown in \autoref{tab:beta_vendi}, the entropy rate coefficient remains positive ($\beta_1 = 0.576$, $p < 2\times 10^{-16}$), with a strongly positive entropy-rate $\times$ fine-tuning interaction ($\beta_3 = 1.040$, $p < 2\times 10^{-16}$). The entropy rate slope thus increases from $0.576$ for base models to $1.615$ for fine-tuned models ($2.81\times$), while the main effect of fine-tuning remains negative ($\beta_2 = -1.157$, $p < 2\times 10^{-16}$). Residual diagnostics show no significant deviation in the KS test ($p = 0.070$), while the dispersion ($p = 4\times 10^{-4}$) and outlier ($p = 0.0015$) tests detect some deviations.

\begin{table}[h]
\centering
\caption{Full table of conditional mean model estimates for the Beta mixed-effects regression with the normalized Vendi Score as the semantic diversity measure.}
\label{tab:beta_vendi}
\begin{tabular}{lrrr}
\toprule
Term & Estimate & Std. Error & $p$-value \\
\midrule
Intercept & -2.517 & 0.062 & $< 2\times 10^{-16}$ \\
\textbf{\boldmath $R$ ($\beta_1$)} & \textbf{0.576} & 0.032 & $< 2\times 10^{-16}$ \\
\textbf{\boldmath Instruct ($\beta_2$)} & \textbf{-1.157} & 0.089 & $< 2\times 10^{-16}$ \\
$\mathrm{ns}(\mathrm{invN},4)_1$ & 0.150 & 0.049 & 0.002 \\
$\mathrm{ns}(\mathrm{invN},4)_2$ & -0.433 & 0.348 & 0.214 \\
$\mathrm{ns}(\mathrm{invN},4)_3$ & 2.202 & 1.881 & 0.242 \\
$\mathrm{ns}(\mathrm{invN},4)_4$ & 4.904 & 4.010 & 0.221 \\
Coding & -0.727 & 0.041 & $< 2\times 10^{-16}$ \\
Math & -0.619 & 0.042 & $< 2\times 10^{-16}$ \\
Stories & -0.304 & 0.036 & $< 2\times 10^{-16}$ \\
\textbf{\boldmath $R \times$ Instruct ($\beta_3$)} & \textbf{1.040} & 0.044 & $< 2\times 10^{-16}$ \\
Instruct $\times \mathrm{ns}(\mathrm{invN},4)_1$ & 0.169 & 0.078 & 0.030 \\
Instruct $\times \mathrm{ns}(\mathrm{invN},4)_2$ & 0.795 & 0.402 & 0.048 \\
Instruct $\times \mathrm{ns}(\mathrm{invN},4)_3$ & 2.112 & 1.978 & 0.286 \\
Instruct $\times \mathrm{ns}(\mathrm{invN},4)_4$ & 2.215 & 4.224 & 0.600 \\
\bottomrule
\end{tabular}
\end{table}

\paragraph{Relaxed Semantic Entropy.}  As an embedding-free alternative, we further evaluate semantic diversity using Relaxed Semantic Entropy \citep{DBLP:journals/corr/abs-2506-09659}, which groups semantically equivalent outputs into clusters and measures the entropy over the resulting clusters. For each task--prompt--model tuple $(t,p,m)$, we sample $\tilde M = 24$ of the $M = 100$ rollouts and use a few-shot prompted LLM judge GPT-5.6-Luna to assess pairwise semantic similarity. Two responses are connected only if the judge considers them similar in both presentation orders, and semantic clusters are identified as the connected components of the resulting graph. We then compute Relaxed Semantic Entropy based on the empirical proportions of rollouts belonging to each cluster:
\[
\mathrm{RSE}_{tpm} = -\sum_{c} \frac{|c|}{\tilde M} \log \frac{|c|}{\tilde M},
\]
where $|c|$ denotes the number of rollouts in semantic cluster $c$. The score ranges from $0$ when all rollouts belong to one semantic cluster to $\log \tilde M$ when every rollout forms a distinct cluster. We then normalize the score via $\mathrm{RSE}_{tpm}/\log \tilde M$ to $[0,1]$ for regression analysis. \footnote{RSE results use an earlier rollout set with minor prompt tokenization differences.}

We model normalized RSE using ordered-beta mixed-effects regression \citep{kubinec2023ordered}, which accommodates both continuous values and exact boundary observations at $0$ and $1$. Moreover, since the LLM judge evaluates semantic equivalence in the context of each prompt, the resulting RSE can be influenced by prompt-specific judging criteria. To account for this, we decompose entropy rate and inverse sequence length into within-prompt and between-prompt components, distinguishing variation among \textit{different models on the same prompt} from variation across \textit{different prompts}.

Specifically, for each task--prompt pair $(t,p)$, let $R^{\mathrm{b}}_{tp}$ and $\mathrm{invN}^{\mathrm{b}}_{tp}$ denote the averages of the standardized $R_{tpm}$ and $\mathrm{invN}_{tpm}$ across the six models, where each model's values are computed over all $M$ rollouts, and define $R^{\mathrm{w}}_{tpm} = R_{tpm} - R^{\mathrm{b}}_{tp}$ and $\mathrm{invN}^{\mathrm{w}}_{tpm} = \mathrm{invN}_{tpm} - \mathrm{invN}^{\mathrm{b}}_{tp}$ as the within-prompt deviations. The conditional mean model is then specified as
\begin{align*}
\operatorname{logit}(\mu_{tpm}) ={}&
\alpha + \beta_1 R^{\mathrm{w}}_{tpm}
+ \beta_2 \mathbf{1}\{m=\mathrm{FT}\}
+ \beta_3 (R^{\mathrm{w}}_{tpm} \cdot \mathbf{1}\{m=\mathrm{FT}\}) \notag\\
&+ \beta_4 R^{\mathrm{b}}_{tp}
+ \beta_5 (R^{\mathrm{b}}_{tp} \cdot \mathbf{1}\{m=\mathrm{FT}\})
+ f_0(\mathrm{invN}^{\mathrm{w}}_{tpm}) \notag\\
&+ f_1(\mathrm{invN}^{\mathrm{w}}_{tpm}) \mathbf{1}\{m=\mathrm{FT}\}
+ \beta_6 \mathrm{invN}^{\mathrm{b}}_{tp} \notag\\
&+ \sum_k \theta_k \mathbf{1}\{t=k\}
+ \sum_j \omega_j \mathbf{1}\{\mathrm{fam}(m)=j\}
+ u_p,
\end{align*}
where $f_0$ and $f_1$ are natural splines with three degrees of freedom, $\omega_j$ are model-family fixed-effect coefficients with Gemma-3-12B as the reference, and $u_p$ is a random prompt intercept. We also use a reduced dispersion model that excludes the length and entropy-rate terms:
\begin{align*}
\log(\phi_{tpm}) &= \gamma_0 + \gamma_1 \mathbf{1}\{m=\mathrm{FT}\} + \sum_k \delta_k \mathbf{1}\{t=k\} \\ 
&+ \sum_j \eta_j \mathbf{1}\{\mathrm{fam}(m)=j\} + \sum_k \zeta_k \mathbf{1}\{m=\mathrm{FT}\}\mathbf{1}\{t=k\}.
\end{align*}

As shown in \autoref{tab:rse_ordbeta}, the within-prompt entropy rate coefficient is positive ($\beta_1 = 0.659$, $p < 2\times 10^{-16}$), with a positive entropy-rate $\times$ fine-tuning interaction ($\beta_3 = 0.446$, $p = 5.20\times 10^{-4}$). The within-prompt entropy rate slope thus increases from $0.659$ for base models to $1.104$ for fine-tuned models ($1.68\times$), while the main effect of fine-tuning is negative ($\beta_2 = -2.910$, $p = 0.037$), consistent with our main findings. Residual diagnostics show no significant deviation in the KS test ($p = 0.11$) and no outliers ($p = 1$), while the dispersion test indicates some residual underdispersion ($p < 2.2\times 10^{-16}$).

\begin{table}[h]
\centering
\caption{Full table of conditional mean model estimates for the ordered Beta mixed-effects regression with relaxed semantic entropy as the semantic diversity measure.}
\label{tab:rse_ordbeta}
\begin{tabular}{lrrr}
\toprule
Term & Estimate & Std. Error & $p$-value \\
\midrule
Intercept & -0.293 & 0.510 & 0.565 \\
\textbf{\boldmath $R^{\mathrm{w}}$ ($\beta_1$)} & \textbf{0.659} & 0.073 & $< 2\times 10^{-16}$ \\
\textbf{\boldmath Instruct ($\beta_2$)} & \textbf{-2.910} & 1.398 & 0.037 \\
$R^{\mathrm{b}}$ ($\beta_4$) & -0.287 & 0.290 & 0.322 \\
$\mathrm{ns}(\mathrm{invN}^{\mathrm{w}},3)_1$ & -8.125 & 0.661 & $< 2\times 10^{-16}$ \\
$\mathrm{ns}(\mathrm{invN}^{\mathrm{w}},3)_2$ & 37.842 & 7.630 & $7.06 \times 10^{-7}$ \\
$\mathrm{ns}(\mathrm{invN}^{\mathrm{w}},3)_3$ & 76.677 & 15.205 & $4.59 \times 10^{-7}$ \\
$\mathrm{invN}^{\mathrm{b}}$ ($\beta_6$) & 0.150 & 0.074 & 0.044 \\
Coding & 0.599 & 0.274 & 0.029 \\
Math & -0.641 & 0.253 & 0.011 \\
Stories & 0.049 & 0.105 & 0.640 \\
Llama-3.1-8B & -0.062 & 0.064 & 0.330 \\
Qwen3-8B & 0.082 & 0.072 & 0.255 \\
\textbf{\boldmath $R^{\mathrm{w}} \times$ Instruct ($\beta_3$)} & \textbf{0.446} & 0.128 & $5.20 \times 10^{-4}$ \\
$R^{\mathrm{b}} \times$ Instruct ($\beta_5$) & -0.546 & 0.097 & $1.78 \times 10^{-8}$ \\
Instruct $\times \mathrm{ns}(\mathrm{invN}^{\mathrm{w}},3)_1$ & 8.867 & 0.990 & $< 2\times 10^{-16}$ \\
Instruct $\times \mathrm{ns}(\mathrm{invN}^{\mathrm{w}},3)_2$ & -31.941 & 8.309 & $1.21 \times 10^{-4}$ \\
Instruct $\times \mathrm{ns}(\mathrm{invN}^{\mathrm{w}},3)_3$ & -73.760 & 15.356 & $1.56 \times 10^{-6}$ \\
\bottomrule
\end{tabular}
\end{table}

\newpage
\end{document}